\documentclass{article}

\usepackage[preprint]{neurips_2026}

\usepackage[utf8]{inputenc} 
\usepackage[T1]{fontenc}    
\usepackage{hyperref}       
\usepackage{url}            
\usepackage{booktabs}       
\usepackage{amsfonts}       
\usepackage{nicefrac}       
\usepackage{microtype}      
\usepackage{xcolor}         
\usepackage{amsmath}
\usepackage{amsthm}
\usepackage{subcaption}
\usepackage{graphicx}
\usepackage{booktabs}
\usepackage{array}
\usepackage{multirow}
\usepackage{tikz}
\usepackage{graphicx}
\usepackage{array}
\usetikzlibrary{arrows.meta}
\usepackage{booktabs}
\usepackage{siunitx}
\usepackage{placeins}
\usepackage{mathabx}

\newtheorem{theorem}{Theorem}
\newtheorem{lemma}[theorem]{Lemma}

\newtheorem{definition}[theorem]{Definition}
\newtheorem{remark}[theorem]{Remark}

\title{Vision Transformer-Conditioned UNet for Domain-Adaptive Semantic Segmentation}

\author{
  Joel Valdivia Ortega \(^{1,2,3}\)
  \And
  Tingying Peng \(^{1}\)
  \And
  Marion Jasnin \(^{2,4}\)
  \And
  \\
   \(^1\)Helmholtz AI, Helmoltz Munich, Neuherberg, Germany\\
   \(^2\)Helmholtz Pioneer Campus, Helmholtz Munich, Neuherberg, Germany\\
   \(^3\)School of Computation, Information and Technology, TUM, Garching, Germany\\
   \(^4\)Department of Chemistry, TUM, Garching, Germany.
   \\
   \\
   All authors can be contacted for correspondance.\\
     \{\texttt{joel.valdiviaortega}, \texttt{tingying.peng}, \texttt{marion.jasnin}\}\texttt{@helmholtz-munich.de}
}

\begin{document}

\maketitle

\begin{abstract}

Semantic segmentation is essential for analysing anatomical features in biomedical research, yet a performance gap remains for Vision Transformers (ViTs) in the field, particularly for sparse, fine-structured, and low signal-to-noise targets. We attribute this challenge in part to the lightweight pixel decoders commonly used in promptable ViT models, who may lack the local inductive bias needed for high-precision biomedical masks. We bridge this gap by introducing ViTC-UNet, which conditions a UNet on frozen pre-trained ViT representations through learnable tokens and a two-way attention decoder. This combines ViT global visual priors with the local inductive bias and high-resolution decoding capacity of UNets, while avoiding end-to-end ViT fine-tuning even in cross-domain settings. ViTC-UNet outperforms baseline results in semantic segmentation tasks across MRI and CT modalities, demonstrating that structure-conditioned UNet decoding can efficiently adapt large-scale visual priors to high-complexity biomedical segmentation. Code can be accessed at [url will be released after publication].

\end{abstract}

\section{Introduction}
\label{sec:introduction}

Semantic segmentation remains a central challenge in computer vision across a wide range of imaging modalities. In natural images, the field has shifted from convolutional neural networks (CNNs) \cite{segnet, convolutional networks} towards ViTs \cite{vit}, which now serve as the backbone of many state-of-the-art (SOTA) segmentation models. These models pair strong pre-trained representations with task-specific decoders, ranging from simple linear heads \cite{dinov2} to lightweight upsampling mask decoders \cite{mask2former, dpt}, and further enable prompt-based interaction through points, boxes, masks \cite{sam, sam2}, and text inputs \cite{sam3, cosine model}.

Adapting this paradigm to biomedical imaging remains challenging. Biomedical datasets are often annotation-limited \cite{data scarcity 2, data scarcity 1}, and dominated by specialized structures with irregular, sparse, thin, or low-contrast morphology \cite{context aware segmentation, totalseg}. While fine-tuning large ViT-based models has improved performance in several biomedical applications \cite{dinobloom, transfer learning 1, transfer learning 2}, such adaptation can be computationally expensive and data-intensive, and lightweight ViT decoders may lack the local inductive bias needed for precise biomedical masks. These challenges have preserved the relevance of UNet-based architectures \cite{unet_biomedical}; in particular, nnU-Net \cite{nnunet} and its derivatives remain highly competitive across diverse medical and biological imaging modalities \cite{membrain, nnunet baseline fundus}. This suggests a methodological synergy where ViTs provide pre-trained global representations and flexible conditioning, whereas UNets provide sample-efficient, high-resolution decoding for intricate biomedical structures.

To exploit this complementarity, we propose ViTC-UNet (Figure \ref{fig:general_description}), a framework that leverages the robust, general-purpose visual priors of a frozen ViT pre-trained on the natural modality to guide a UNet as pixel decoder optimized for the intricate structures of biomedical imaging. By fusing convolutional inductive biases with these rich latent representations, our approach retains the sample efficiency of UNets while bypassing the need for computationally intensive Transformer fine-tuning. This architecture effectively transforms the UNet into a promptable, structure-conditioned and high-fidelity pixel decoder (Figure \ref{fig:proposed_arch}) that achieves SOTA performance in semantic segmentation across diverse modalities, including sparse and fine-structured targets where fixed-channel decoders are defied (Table \ref{tab:biomedical results}). Furthermore, the underlying conditioning mechanism decouples semantic identity from the output layer, allowing new structure tokens to be incorporated without modifying the network architecture (Table \ref{tab:halved results}), demonstrating a robust capacity for incremental learning in heterogeneous domains.
\begin{figure}[!t]
        \centering
        \includegraphics[width=0.82\linewidth]{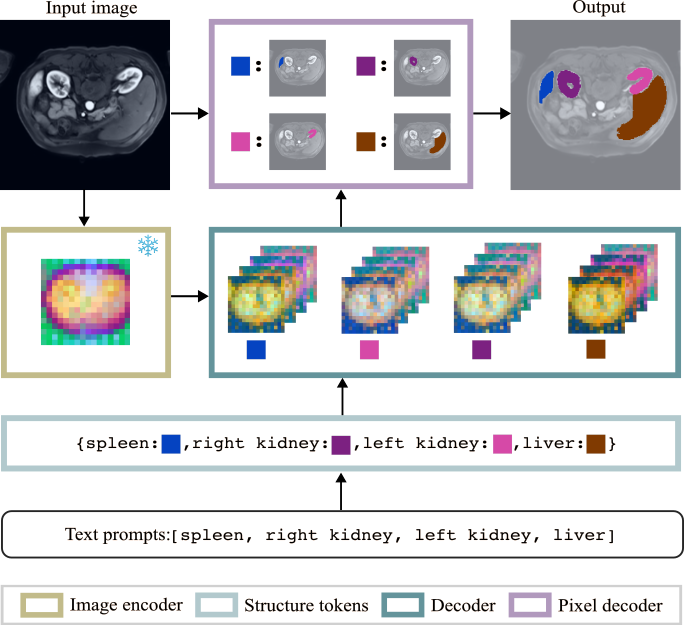}
        \caption{Information flow of ViTC-UNet: a multi-modal framework for semantic segmentation. The input image is processed through a frozen ViT to extract high-level visual priors. These priors are modulated by structure tokens in a two-way attention decoder via latent conditioning. The resulting feature maps, along with the original image, are fed into a UNet acting as a pixel-level decoder. By fusing the global context from the conditioned ViT priors with the local inductive biases of CNNs, the UNet generates precise, structure-specific masks. Notably, the architecture produces individual masks for each input structure through a single-channel output, efficiently leveraging the learned latent conditioning via prompting.}
        \label{fig:general_description}
\end{figure}

\begin{figure}[!t]
        \centering
        \includegraphics[width=0.99\linewidth]{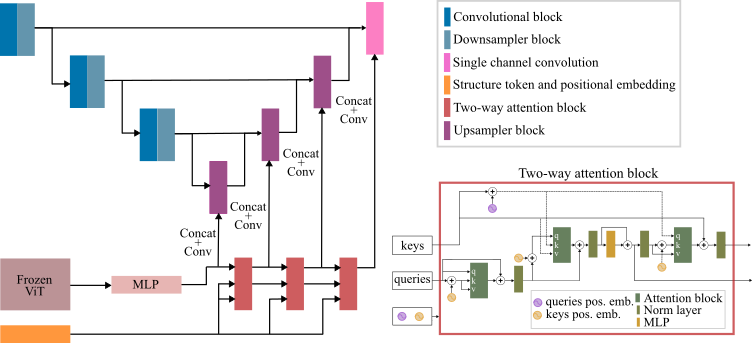}
        \caption{Overview of ViTC-UNet. A ViT generates image embeddings which are integrated with structure tokens through a two-way attention-based decoder modified from SAM \cite{sam}. These conditioned features then serve to prompt the nnU-Net architecture as pixel decoder.}
        \label{fig:proposed_arch}
\end{figure}

\section{Related Work}

\subsection{Hybridizing Convolutional and Attention-Based Architectures}

Despite the dominance of UNet in biomedical segmentation \cite{nnunet}, recent trends favor integrating ViTs to capture global context \cite{microsam,medsam}. Early hybrids typically replace the convolutional bottleneck or augment skip connections with attention blocks \cite{3d-vit-unet, evit}, but require computationally intensive end-to-end fine-tuning and often fail to consistently outperform a well-tuned UNet under rigorous validation \cite{nnunet revisited}. 

To mitigate this, work has been done to employ frozen ViT backbones. \citet{frozen mae} fuse features from a static Masked Autoencoder into a convolutional decoder, though it remains dependent on domain-specific pre-training. Although \citet{frozen dinov2} utilize a frozen DINOv2 for the segmentation of surgical instruments, the domain shift remains minimal compared to modalities such as CT or MRI. Closely related to our approach, \citet{rmlp} concatenates ViT embeddings into a UNet bottleneck. However, their framework requires domain-adapted backbones and remains structurally rigid. Our method distinguishes itself by integrating natural image priors without domain-specific fine-tuning or task-specific architectural restrictions such as imposing structural constraints on the output dimensionality relative to the labels set.

\subsection{Promptable Segmentation and Pixel Decoder Conditioning}
ViT-based models such as SAM \cite{sam3, sam, sam2} have pushed interactive segmentation through multi-modal prompting. However, their reliance on relatively simple upsampling decoders often necessitates massive fine-tuning of domain-specific labels to handle biomedical heterogeneity \cite{microsam, medsam}. While BiomedParse \cite{biomedparse} introduces text-conditioned lightweight decoders, its reliance on large-scale image-mask-text supervision and joint vision-language fine-tuning makes adaptation costly, and its performance may be limited in specialized biomedical domains that are underrepresented in the training corpus. Furthermore, the continued competitiveness of the more sample-efficient nnU-Net \cite{nnunet} against large-scale biomedical models suggests that convolutional decoders remain effective alternatives for resolving intricate anatomical structures.

In contrast, most UNet-based models \cite{nnunet abdtumor_adrenal,nnunet baseline brightfield 1} use fixed output channels, with one channel per target structure. While effective for standard semantic segmentation, this formulation ties semantic identity to output dimensionality and lacks target-specific latent conditioning. We depart from this by introducing a conditioned UNet decoder (Figure \ref{fig:proposed_arch}), merging the flexibility of promptable models with the robust spatial precision of the UNet framework.

\section{Model Architecture} \label{sec:architecture}

The ViTC-UNet architecture (Figure \ref{fig:proposed_arch}) is based on a modular setting that merges feature extraction with high-resolution mask prediction, consisting of a frozen backbone encoder, a learnable conditioning decoder, and a UNet-based pixel decoder seeking to unify the global representational power of ViTs with the localized precision of CNNs. Aiming to avoid costly fine-tuning while leveraging the robust embeddings of ViTs pre-trained on natural images, our focus is not on modifying the backbone, but on how frozen ViT representations can be decoded. ViTC-UNet injects target-specific ViT guidance into the inherently multi-scale reconstruction path of a UNet, enabling high-capacity, structure-conditioned decoding for biomedical targets.

\subsection{Backbone Encoder}

The encoder consists of a frozen ViT backbone pre-trained on large-scale natural images to provide robust visual priors. To contrast the adaptability across distinct pretraining objectives, we evaluate our framework using both DINOv2 \cite{dinov2}, which follows a self-supervised contrastive regime, and SAM 2 \cite{sam2}, optimized for supervised promptable segmentation. By maintaining fixed weights, the model bypasses the intensive data requirements of Transformer fine-tuning and preserves high-quality, general-purpose features. This strategy reduces the number of trainable parameters and mitigates the data demands associated with end-to-end ViT fine-tuning in annotation-limited biomedical settings \cite{data scarcity 2, data scarcity 1}. 

\subsection{Conditioning Decoder}

The transition from high-level natural image embeddings to biomedical spatial representations is mediated by a conditioning decoder. This module employs an MLP followed by a series of modified two-way attention blocks \cite{sam} to facilitate bidirectional communication between image features and learnable prompts. Unlike fixed-label architectures, we introduce a structure token for each target class, paired with shared spatial positional embeddings. The MLP functions as a continuous mapping, projecting backbone representations into a specialized segmentation space. Within the attention blocks, image embeddings and structure tokens mutually attend to one another, allowing the visual representation to be modulated by the requested target while the structure token aggregates image-specific context. In the initial fusion stage, the structure token is spatially replicated to match the image-embedding grid before bidirectional cross-attention. This configuration enables a dynamic, fully learnable segmentation process optimized for task-specific alignment.

\subsection{Multi-Stage and Multi-Scale Latent Conditioning}\label{representing curve}

Empirical evidence suggests that dense prediction benefits from using more than a single latent representation during decoding compared to relying solely on the final block \cite{dinov2, dpt}. In ViTC-UNet, we do not extract intermediate layers' outputs from the frozen ViT backbone, but start from the final ViT representation and generate a sequence of structure-conditioned decoder states through successive two-way attention blocks. \citet{rmlp} established theoretically that ViTs act as homeomorphisms i.e bijective continuous mappings with continuous inverses, provided that they operate on compact domains and satisfy an injectivity constraint. This property is enforced in models like DINOv2 during training through the KoLeo regularizer \cite{koleo}, defined as \[\mathcal{L}_{\text{KoLeo}}\left(\{x_1,\dots,x_n\}\right) = -\frac{1}{n} \sum_{i=1}^n \log\left(\min_{j \neq i} \|x_i - x_j\|_2\right),\] which diverges unless injectivity is maintained. Using Heine-Borel theorem (Theorem \ref{heine-borel}), we extend this reasoning to SAM 2 since, given fixed-resolution natural image domains and binary mask codomains, an optimal segmentation mapping is inherently injective, as identical masks across all potential queries imply identical input images. 

Given the ViT backbone functions as an homeomorphism, the decoder must preserve this topological trait to maintain segmentation fidelity. Per Remark \ref{concatenation of homeomorphisms}, a finite composition of mappings is homeomorphic if and only if each constituent transformation is itself an homeomorphism. Consequently, we can understand individual attention blocks within our decoder as homeomorphic transformations. This perspective treats the sequence of intermediate outputs not as discrete points, but as a continuous evolution of the image representation tracing a curve in the latent space and thus becoming able to represent images despite a change in domain taking place. This topological property ensures a well-behaved representation space, facilitating more robust downstream decoding and enhancing segmentation precision.

Motivated by this, the two-way attention decoder progressively transforms the frozen ViT embedding into a sequence of target-conditioned latent states. Rather than relying on a single bottleneck representation, this trajectory provides repeated opportunities to align image features with the target structure during decoding. Architecturally, these states provide multi-stage conditioning, since image-token interactions are refined across successive attention blocks, and multi-scale conditioning, since each state is fused with a corresponding level of the UNet decoding path. Deeper decoding stages receive target-specific global context, while shallower stages support the recovery of fine boundaries, narrow structures, and sparse targets. 

\subsection{Structure-conditioned UNet Pixel Decoder}

To resolve fine morphological details, we employ nnU-Net architecture as the pixel decoder, integrated in such a way that the output of the n\(^{\text{th}}\) attention block from the decoder is concatenated with the skip connection of the n\(^{\text{th}}\) shallowest convolutional block (Figure \ref{fig:proposed_arch}). This fusion ensures the final mask benefits from both convolutional inductive biases and the target-specific global context provided by the frozen backbone and conditioning decoder. Notably, the pixel decoder utilizes a single output channel regardless of the number of target structures. Since the mask is conditioned by the input structure token, the model decouples semantic identity from output dimensionality, enabling new target structures to be incorporated through additional learnable tokens without further pixel decoder modifications (Table \ref{tab:halved results}). At inference time, each requested structure token is paired with the frozen ViT image embedding to produce a single-channel binary mask. Multi-class segmentations are obtained by repeating this process over the target token set.

\section{Implementation and Training Details} \label{sec:implementation}

\paragraph{Encoder Configurations.} To evaluate the robustness and flexibility of the proposed architecture, we conducted experiments using two distinct ViT backbones, namely DINOv2 and SAM 2. To demonstrate that the framework's efficacy is not solely dependent on massive model scales, we utilized the small variants of both encoders. These models were selected to test the system's adaptability to disparate latent space geometries and training paradigms, specifically the self-supervised contrastive learning of DINOv2 versus the supervised, interactive objective of SAM 2. In all experiments, the encoder remained frozen to preserve pre-trained visual priors and minimize computational overhead.

\paragraph{Decoder Specifications.} The two-way attention blocks maintained a consistent structural configuration across all experimental trials, with the latent dimensionality adjusted to match the respective backbone. We represented object classes using a dictionary of learnable structure tokens, with the number of entries corresponding to the total classes in each dataset, excluding the background. Each token was initialized as a unidimensional tensor congruent with the ViT’s latent dimension. Furthermore, the decoder utilized a learnable positional encoding scheme designed to emulate the spatial embeddings of DINOv2.

\paragraph{Pixel Decoder Baselines.} We benchmarked our proposed ViTC-Unet architecture against two pixel decoder settings. The first baseline employed a simple linear head for pixel-wise classification, while the second utilized the ViT-UNet hybrid configuration used by \citet{rmlp}. For both the hybrid baseline and our proposed model, the convolutional backbone followed a fix structure from the DynUNet architecture used in nnU-Net with a 2D configuration following that of the encoder. Notably, while the linear and hybrid baselines treated the background as a dedicated output channel, our proposed architecture omitted a background one.

\paragraph{Training and Optimization.} Experimental training and evaluation was performed using a single NVIDIA A100 GPU with 64GB of memory for less than 12 hours per job, highlighting this way the low resource training setting. All models were optimized using AdamW \cite{adamw} with a batch size of 32 for a maximum of 100 epochs, where each epoch comprised 300 iterations. To facilitate robust convergence and prevent over-fitting, we employed an early stopping criterion where training was terminated if the validation metric failed to yield a relative improvement of 1\% over a 10-epoch plateau. The training objective was defined as a linear combination of Focal loss \cite{focal loss} and Dice loss \cite{dice loss}, with the Focal loss hyperparameter \(\gamma\) set to 2 to address potential class imbalance. To ensure the reproducibility and statistical significance of our findings, each configuration of backbone and pixel decoder was evaluated across five independent trials using distinct stochastic weight initializations. The resulting performance metrics, aggregated across these runs, are reported in Table \ref{tab:biomedical results}.

\paragraph{Datasets and Evaluation Framework.} To ensure a rigorous benchmark against the nnU-Net baseline, we utilized the data compilation curated by \citet{biomedparse}. We trained and evaluated our architecture on individual datasets within this collection to simulate the data-constrained regimes typical of specialized biomedical modalities. For each dataset, we employed an 80-20 random split of the training and validation partitions. To evaluate generalization, we utilized the independent holdout set from the "Interactive 3D Biomedical Image Segmentation" benchmark as our final test ensemble. To ensure statistical robustness and mitigate selection bias, the assignment of training and validation folds was randomized independently across each of our five experimental runs.

\paragraph{Class Sampling Strategy.} Empirical evaluation revealed that the proposed conditioned pixel decoder achieves optimal performance when all semantic classes within a dataset are presented simultaneously within each optimization step. Consequently, our training protocol involves loading a single image and computing the full ensemble of segmentation masks for all defined classes, irrespective of their presence in the specific input volume. This exhaustive sampling strategy mirrors the training objective of a standard UNet, ensuring that the model is exposed to a balanced distribution of both positive and negative examples across all anatomical structures. By maintaining this parity, we ensure that the conditioning mechanism, driven by the structure tokens, learns to maintain the localized precision characteristic of traditional convolutional decoders.

\paragraph{Label Space Expansion.}\label{sec:label_expansion} To assess the model's capacity for incremental learning, we conducted a two-stage training experiment. We first trained the model on half of the available labels, followed by a second stage comprising the full label space. Results in Table \ref{tab:halved results} demonstrate that the architecture integrates novel classes without degrading performance on previously learned structures. This confirms that our conditioning mechanism enables label-space expansion while preserving representational fidelity.

\section{Results and Discussion}\label{sec:results}

We evaluate segmentation performance using mean Intersection over Union (mIoU), excluding the background class to focus on discriminative accuracy. Volumetric scans are processed slice-wise, with 2D predictions reassembled into 3D volumes for final metric computation.

\paragraph{Comparative Performance.}
As shown in Table \ref{tab:biomedical results}, DINOv2-based ViTC-UNet achieves the strongest overall performance among the evaluated models, with an average foreground mIoU of 0.90 across the tested datasets under equal dataset weighting. This exceeds nnU-Net, which achieves an average mIoU of 0.79. ViTC-UNet also matches or exceeds nnU-Net on 14 of 15 benchmarks, falling slightly below it only on Abdomen1K and matching it on TotalSegmentator. Notably, this average performance also slightly exceeds the aggregated 0.88 mIoU reported by BiomedParse on its dataset, while ViTC-UNet uses frozen ViT features and does not require large-scale image-mask-text training. Together, these results indicate that our structure-conditioned UNet decoding provides an effective and computationally efficient route for adapting large-scale visual priors to biomedical segmentation.

\setlength{\tabcolsep}{3pt} 
\setlength{\parskip}{1em}
\begin{table}[]
  \caption{Comparative quantitative analysis of semantic segmentation performance across CT and MRI modalities. We report the mean mIoU \(\pm\) standard deviation across five independent experimental runs for each encoder-decoder-pixel decoder configuration. Bold and underlined indicates the models that performed the best and second best in mean per dataset respectively. Datasets are taken from the compilation made by \citet{biomedparse}.}
  \vspace{1em}
  \label{tab:biomedical results}
  \centering
  \begin{subtable}[t]{0.95\textwidth}
    \centering
    \scriptsize
     \begin{tabular}{lccc @{\hspace{3.0em}} ccc @{\hspace{3.0em}} c}
  \toprule
   & \multicolumn{3}{l}{\hspace{6em}DINOv2} & \multicolumn{3}{l}{\hspace{6em}SAM 2} & \\
  \cmidrule(r{2.7em}){2-4}\cmidrule(r{2.7em}){5-7}
     &&ViT-UNet&&&ViT-UNet&&\\
    Dataset   &   Linear   & hybrid             &ViTC-UNet  &Linear   & hybrid             &ViTC-UNet&nn-UNet\\

    \midrule
    &  \multicolumn{7}{c}{CT}\\
    \cmidrule(r){2-8}
    Abdomen1K \cite{abdomen1k}&0.04\(\pm\)0.03&0.07\(\pm\)0.05&\underline{0.81\(\pm\)0.07}&0.03\(\pm\)0.03&0.07\(\pm\)0.05&0.14\(\pm\)0.01&\textbf{0.85\(\pm\)0.05} \cite{abdomen1k}\\
    Aorta \cite{aorta}&0.13\(\pm\)0.08&0.33\(\pm\)0.21&\textbf{0.98\(\pm\)0.01}&0.04\(\pm\)0.03&0.35\(\pm\)0.21&0.75\(\pm\)0.04&\underline{0.96} \cite{nnunet aorta}\\
    AirwayTree \cite{airwaytree}&0.01\(\pm\)0.01&0.01\(\pm\)0.01&\textbf{0.93\(\pm\)0.04}&0.01\(\pm\)0.01&0.03\(\pm\)0.01&0.68\(\pm\)0.04&\underline{0.82\(\pm\)0.03} \cite{airwaytree}\\
    HaN-Seg \cite{han}&0.01\(\pm\)0.01&0.06\(\pm\)0.01&\textbf{0.97\(\pm\)0.01}&0.01\(\pm\)0.01&0.16\(\pm\)0.01&0.76\(\pm\)0.01&\underline{0.77\(\pm\)0.07} \cite{han}\\
    Lungs \cite{lungs}&0.63\(\pm\)0.21&0.77\(\pm\)0.26&\textbf{0.90\(\pm\)0.03}&0.63\(\pm\)0.21&0.78\(\pm\)0.26&\underline{0.85\(\pm\)0.06}&0.67\(\pm\)0.07 \cite{lungs}\\
    TotalSeg \cite{totalseg}&0.12\(\pm\)0.05&0.36\(\pm\)0.13&\textbf{0.87\(\pm\)0.06}&0.03\(\pm\)0.02&0.33\(\pm\)0.16&\underline{0.55\(\pm\)0.04}&\textbf{0.87} \cite{totalsegmentator}\vspace{1em}\\
    

    &  \multicolumn{7}{c}{MRI}\\
    \cmidrule(r){2-8}
    AMOS \cite{amos}&0.30\(\pm\)0.12&0.56\(\pm\)0.25&\textbf{0.90\(\pm\)0.03}&0.19\(\pm\)0.07&0.52\(\pm\)0.25&0.48\(\pm\)0.03&\underline{0.76} \cite{amos unet}\\
    BraTS \cite{brats 1, brats 3, brats 2}&0.08\(\pm\)0.13&0.19\(\pm\)0.21&\textbf{0.93\(\pm\)0.03}&0.07\(\pm\)0.10&0.13\(\pm\)0.16&0.43\(\pm\)0.04&\underline{0.75} \cite{nnunet brats}\\
    CHAOS-T1 \cite{chaos}&0.57\(\pm\)0.01&0.71\(\pm\)0.03&\textbf{0.91\(\pm\)0.01}&0.56\(\pm\)0.01&0.75\(\pm\)0.02&0.26\(\pm\)0.07&\underline{0.90\(\pm\)0.01} \cite{nnunet chaos}\\
    Heart\_ACDC \cite{nnunet heart acdc}&0.18\(\pm\)0.17&0.22\(\pm\)0.22&\textbf{0.92\(\pm\)0.02}&0.20\(\pm\)0.18&0.22\(\pm\)0.23&0.46\(\pm\)0.01&\underline{0.62\(\pm\)0.22} \cite{nnunet heart acdc}\\
    ISLES ADC/DWI \cite{isles}&0.01\(\pm\)0.01&0.01\(\pm\)0.01&\textbf{0.99\(\pm\)0.01}&0.01\(\pm\)0.01&0.01\(\pm\)0.01&\underline{0.72\(\pm\)0.01}&0.69\(\pm\)0.08 \cite{nnunet isles}\\
    LeftAtrium \cite{medical decathlon}&0.28\(\pm\)0.01&0.63\(\pm\)0.03&\textbf{0.98\(\pm\)0.01}&0.36\(\pm\)0.08&0.61\(\pm\)0.09&0.50\(\pm\)0.09&\underline{0.87}\cite{nnunet}\\
    ProstateADC/T2 \cite{medical decathlon}&0.59\(\pm\)0.07&0.73\(\pm\)0.06&\underline{0.73\(\pm\)0.02}&0.62\(\pm\)0.09&0.71\(\pm\)0.06&\textbf{0.92\(\pm\)0.05}&0.71 \cite{nnunet}\\
    QIN-Lesion \cite{QIN-PROSTATE}&0.00\(\pm\)0.00&0.01\(\pm\)0.01&\underline{0.80\(\pm\)0.06}&0.01\(\pm\)0.01&0.01\(\pm\)0.01&\textbf{0.99\(\pm\)0.01}&0.79 \cite{prostate lesion}\\
    Spider \cite{spider} & 0.12\(\pm\)0.05&0.56\(\pm\)0.10&\textbf{0.84\(\pm\)0.05}&0.09\(\pm\)0.05&0.44\(\pm\)0.08&0.61\(\pm\)0.08&\underline{0.82\(\pm\)0.02} \cite{nnunet spider}\\ 

    \bottomrule
  \end{tabular}
  \end{subtable}

  \end{table}

{\setlength{\tabcolsep}{3pt} 
\setlength{\parskip}{1em}
\begin{table}[]
  \caption{Impact of incremental class acquisition on ViTC-UNet performance. We compare a standard joint training baseline (all classes from initialization, top) against a two-phase training where the model is first trained on half subset of classes before incorporating the remainder (bottom). Results are reported as mean mIoU \(\pm\) standard deviation over five independent runs. The lack of significant performance degradation in the two-phase setting demonstrates ViTC-UNet's capacity for incremental expansion of the label space without the need of architectural modifications nor compromising existing representations or final segmentation fidelity.}
  \vspace{1em}
  \label{tab:halved results}
  \centering
  \begin{subtable}[t]{0.95\textwidth}
    \centering
    \scriptsize
     \begin{tabular}{c@{\hspace{1.0em}}cc @{\hspace{2.0em}} cccccc}
  \toprule
   & \multicolumn{2}{c}{\hspace{0em}CT} & \multicolumn{5}{c}{\hspace{0em}MRI} & \\
  \cmidrule(r{2.0em}){2-3}\cmidrule(l{0.25em}){4-8}
     Two-stage&&&&&&&\\
     Training&HaN-Seg \cite{han}&TotalSeg \cite{totalseg}&AMOS \cite{amos}& CHAOS-T1 \cite{chaos}&ISLES ADC/DWI \cite{isles}&QIN-Lesion \cite{QIN-PROSTATE}&Spider \cite{spider}&\\
     \midrule
     No&0.97\(\pm\)0.01&0.87\(\pm\)0.06&0.90\(\pm\)0.03&0.91\(\pm\)0.01&0.99\(\pm\)0.01&0.80\(\pm\)0.06&0.84\(\pm\)0.05\\
    Yes&0.94\(\pm\)0.01&0.93\(\pm\)0.05&0.87\(\pm\)0.02&0.86\(\pm\)0.03&0.99\(\pm\)0.01&0.79\(\pm\)0.03&0.84\(\pm\)0.05\\

    \bottomrule
  \end{tabular}
  \end{subtable}

  \end{table}}

\begin{figure}[h!]
    \centering
    \newcommand{\imgcell}[1]{%
        \begin{minipage}[c][1.9cm][c]{\linewidth}
            \centering
            \includegraphics[width=\linewidth]{#1}
        \end{minipage}%
    }
    \newcommand{\tikzimg}[2]{%
        \begin{minipage}[c][1.9cm][c]{\linewidth}
            \centering
            \begin{tikzpicture}
                \node[inner sep=0pt] (image) {\includegraphics[width=\linewidth]{#1}};
                #2
            \end{tikzpicture}
        \end{minipage}%
    }
    \setlength{\tabcolsep}{7pt}
    \renewcommand{\arraystretch}{5.5} 

    \begin{subfigure}{\textwidth}
        \centering
        \begin{tabular}{>{\centering\arraybackslash}p{0.16\textwidth} >{\centering\arraybackslash}p{0.16\textwidth} >{\centering\arraybackslash}p{0.16\textwidth} >{\centering\arraybackslash}p{0.16\textwidth} >{\centering\arraybackslash}p{0.16\textwidth}}
            \textbf{CT Image} & \textbf{Ground Truth} & \textbf{Linear} & \textbf{ViT-UNet hybrid} & \textbf{ViTC-UNet} \\[-2.0em]
            
            \imgcell{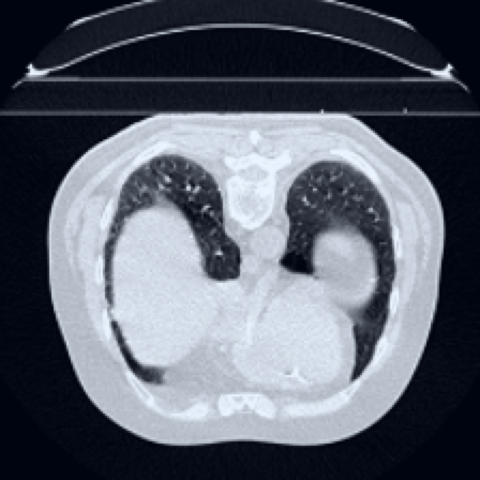} &
            \tikzimg{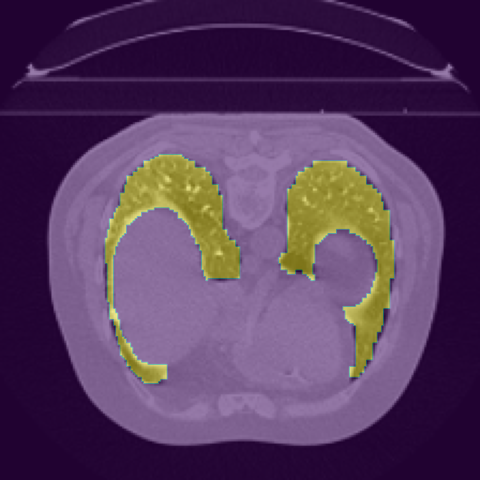}{\draw[white, line width=1.5pt, ->, >=Latex] (0.05,-0.1) -- ++(0.4,-0.4);} &
            \tikzimg{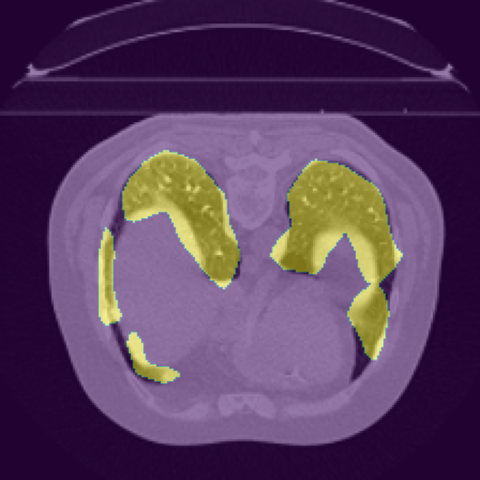}{\draw[white, line width=1.5pt, ->, >=Latex] (0.05,-0.1) -- ++(0.4,-0.4);} &
            \tikzimg{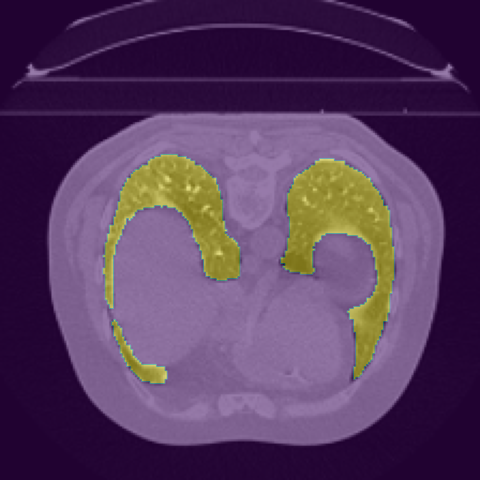}{\draw[white, line width=1.5pt, ->, >=Latex] (0.05,-0.1) -- ++(0.4,-0.4);} &
            \tikzimg{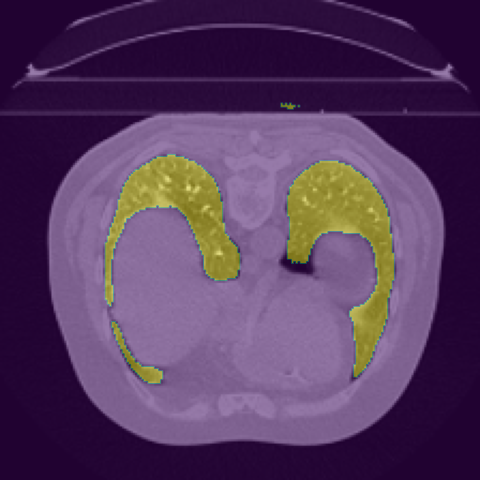}{\draw[white, line width=1.5pt, ->, >=Latex] (0.05,-0.1) -- ++(0.4,-0.4);} \\
            
            \imgcell{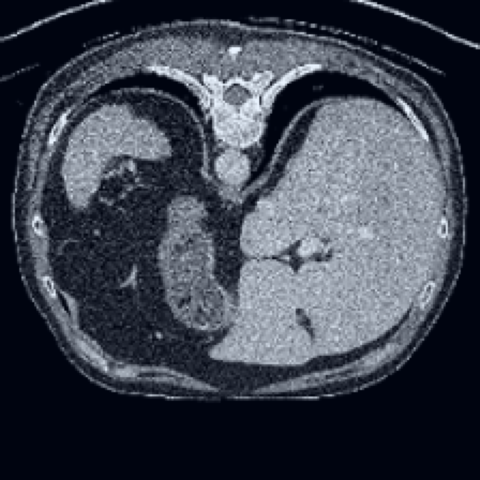} &
            \tikzimg{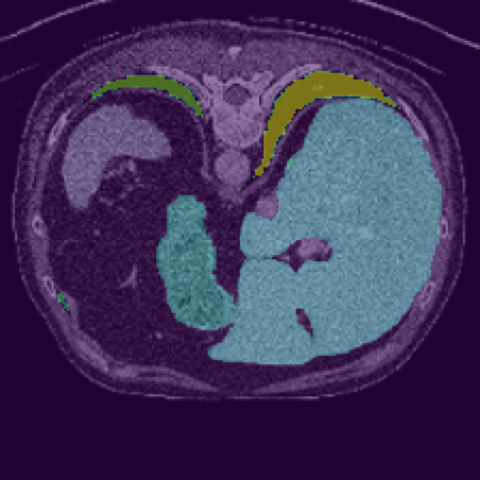}{\draw[white, line width=1.5pt, ->, >=Latex] (-0.2,0.2) -- ++(0.4,0.4);} &
            \tikzimg{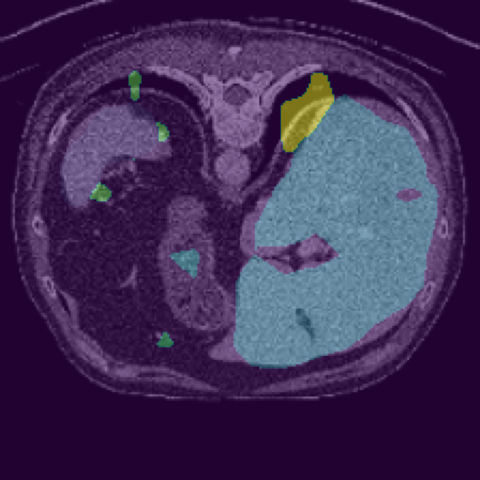}{\draw[white, line width=1.5pt, ->, >=Latex] (-0.2,0.2) -- ++(0.4,0.4);} &
            \tikzimg{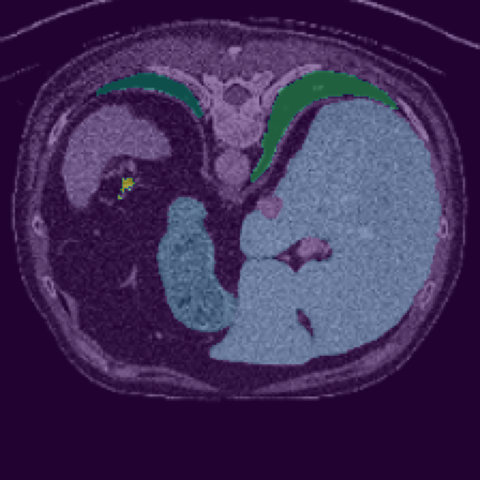}{\draw[white, line width=1.5pt, ->, >=Latex] (-0.2,0.2) -- ++(0.4,0.4);} &
            \tikzimg{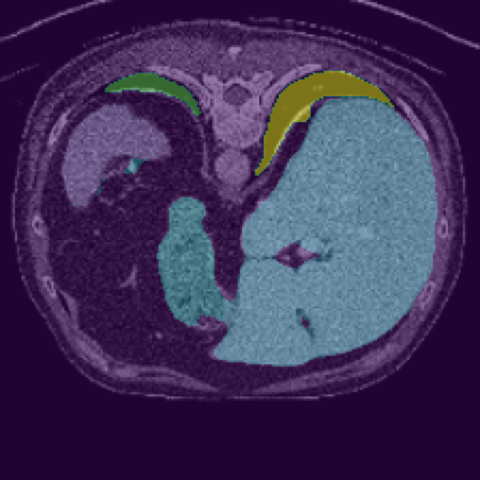}{\draw[white, line width=1.5pt, ->, >=Latex] (-0.2,0.2) -- ++(0.4,0.4);} \\
        \end{tabular}
        \vspace{1em}
        \caption{CT modality results across datasets. Rows illustrate samples from Lungs \cite{lungs} and Totalseg \cite{totalseg}. Arrows highlight the mask segmentation for lungs (row 1) and right lung (row 2).}
        \label{fig:qualitative_ct}
    \end{subfigure}

    \vspace{-2em} 

    \begin{subfigure}{\textwidth}
        \centering
        \begin{tabular}{>{\centering\arraybackslash}p{0.16\textwidth} >{\centering\arraybackslash}p{0.16\textwidth} >{\centering\arraybackslash}p{0.16\textwidth} >{\centering\arraybackslash}p{0.16\textwidth} >{\centering\arraybackslash}p{0.16\textwidth}}
            \textbf{MRI Image} & \textbf{Ground Truth} & \textbf{Linear} & \textbf{ViT-UNet hybrid} & \textbf{ViTC-UNet} \\[-2.0em]
            
            \imgcell{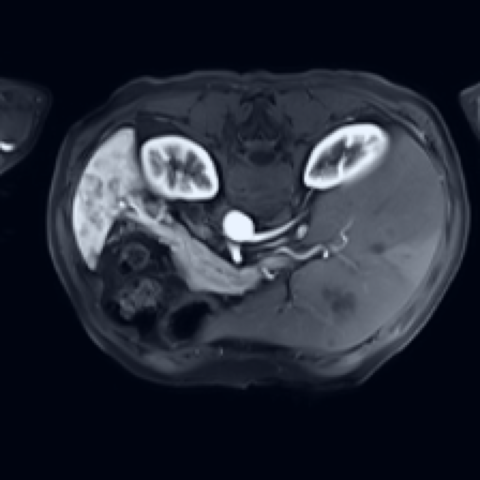} &
        \tikzimg{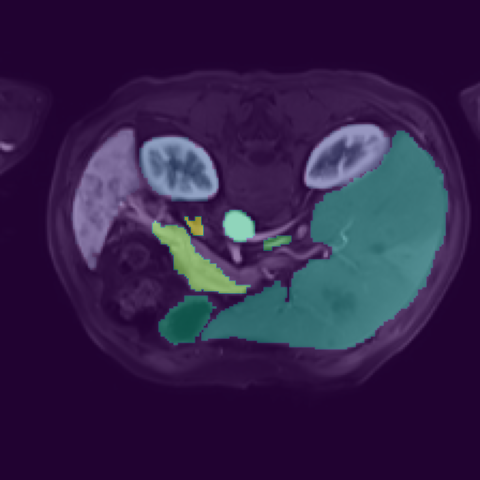}{\draw[white, line width=1.5pt, ->, >=Latex] (-0.65,0.3) -- ++(0.4,-0.4);} &
        \tikzimg{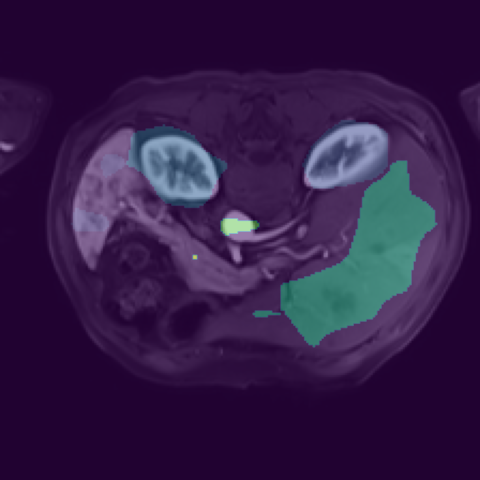}{\draw[white, line width=1.5pt, ->, >=Latex] (-0.65,0.3) -- ++(0.4,-0.4);} &
        \tikzimg{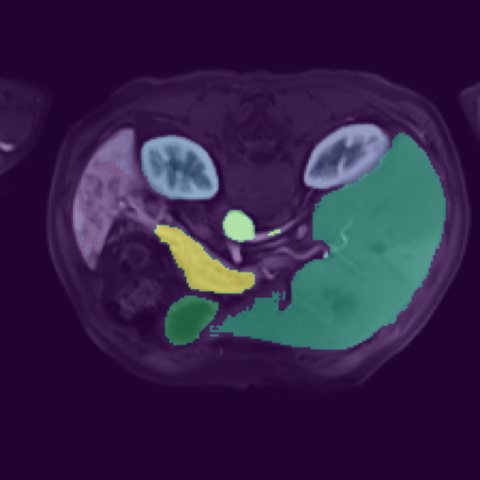}{\draw[white, line width=1.5pt, ->, >=Latex] (-0.65,0.3) -- ++(0.4,-0.4);} &
        \tikzimg{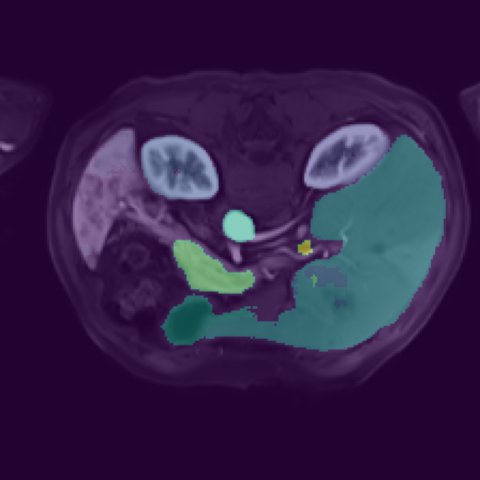}{\draw[white, line width=1.5pt, ->, >=Latex] (-0.65,0.3) -- ++(0.4,-0.4);}  \\

        \imgcell{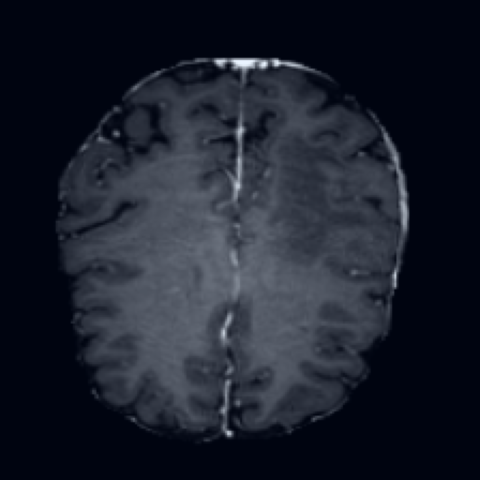} &
        \tikzimg{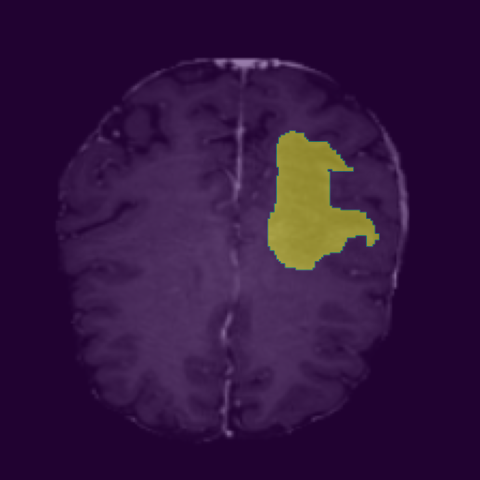}{\draw[white, line width=1.5pt, ->, >=Latex] (-0.2,0.6) -- ++(0.4,-0.4);} &
        \tikzimg{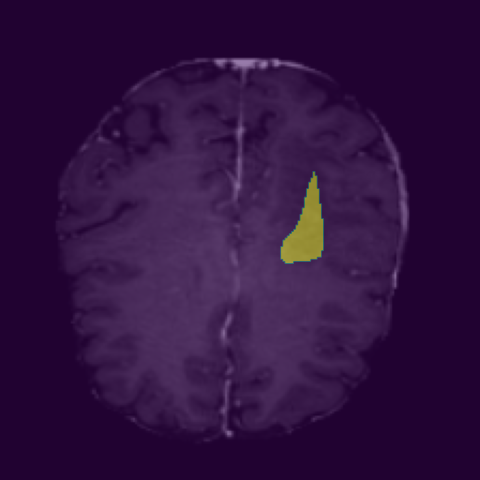}{\draw[white, line width=1.5pt, ->, >=Latex] (-0.2,0.6) -- ++(0.4,-0.4);} &
        \tikzimg{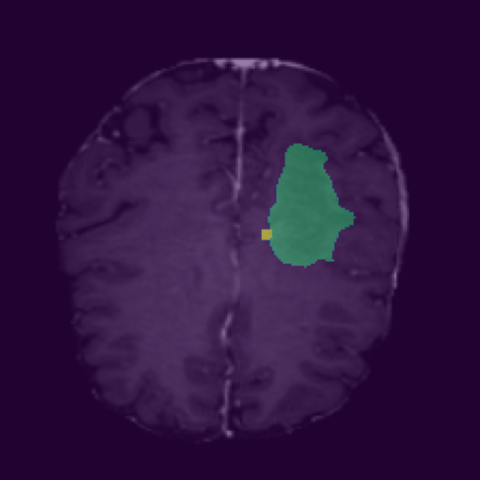}{\draw[white, line width=1.5pt, ->, >=Latex] (-0.2,0.6) -- ++(0.4,-0.4);} &
        \tikzimg{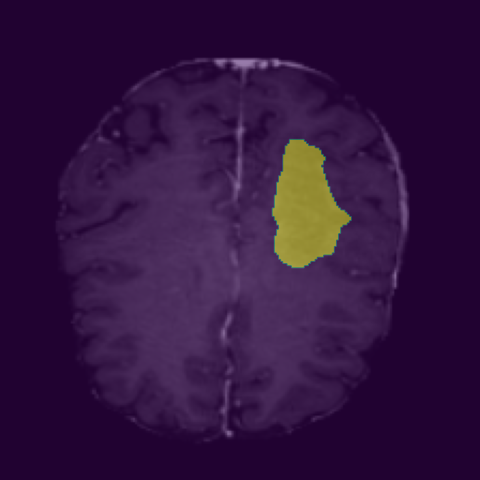}{\draw[white, line width=1.5pt, ->, >=Latex] (-0.2,0.6) -- ++(0.4,-0.4);}  \\

        \imgcell{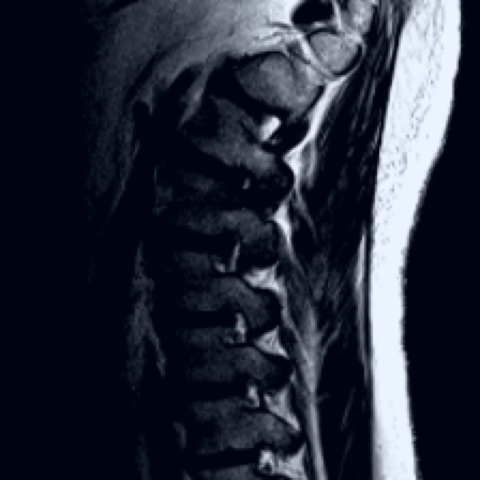} &
        \tikzimg{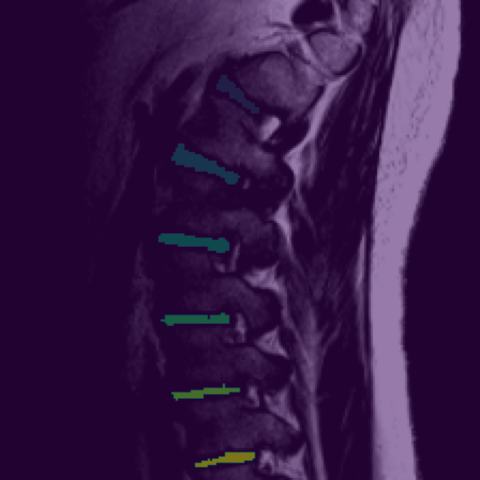}{\draw[white, line width=1.5pt, ->, >=Latex] (-0.78,0.43) -- ++(0.4,-0.4);} &
        \tikzimg{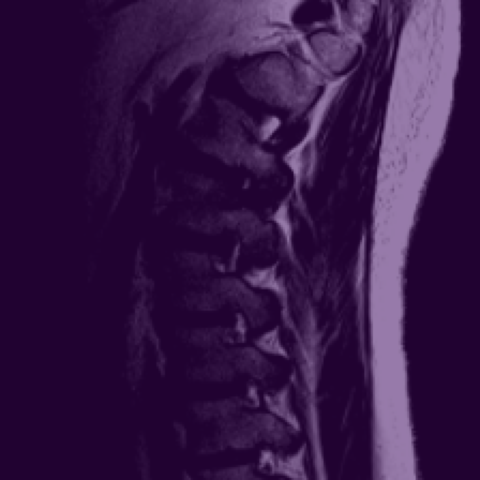}{\draw[white, line width=1.5pt, ->, >=Latex] (-0.78,0.43) -- ++(0.4,-0.4);} &
        \tikzimg{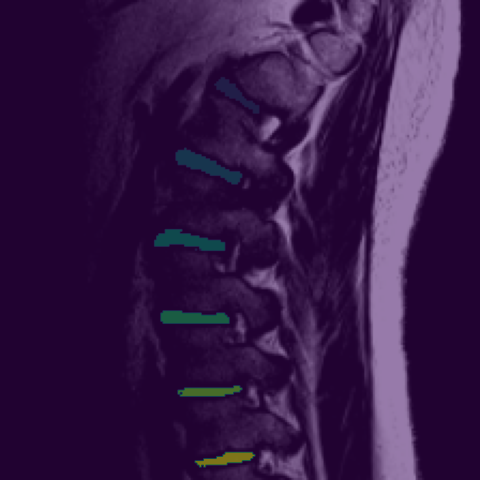}{\draw[white, line width=1.5pt, ->, >=Latex] (-0.78,0.43) -- ++(0.4,-0.4);} &
        \tikzimg{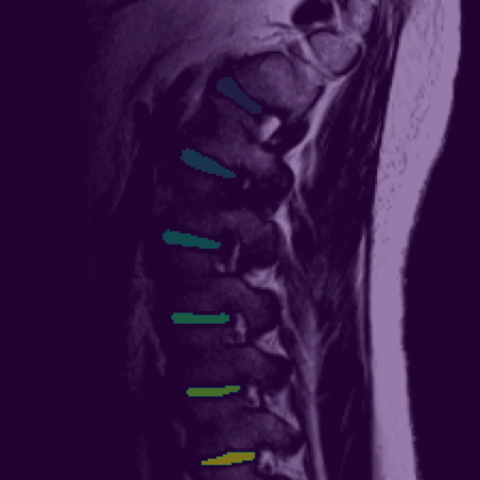}{\draw[white, line width=1.5pt, ->, >=Latex] (-0.78,0.43) -- ++(0.4,-0.4);}  \\

        \imgcell{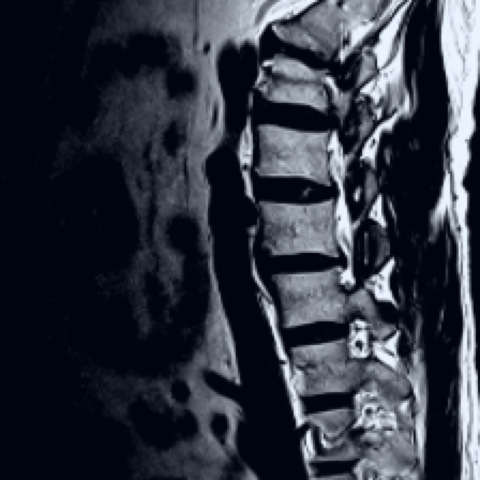} &
        \tikzimg{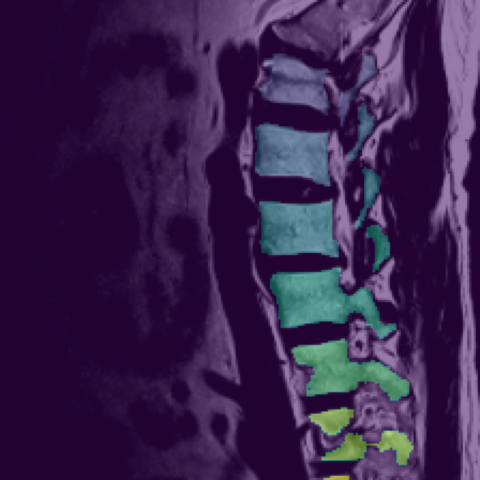}{\draw[white, line width=1.5pt, ->, >=Latex] (0.05,-0.4) -- ++(0.4,-0.4);} &
        \tikzimg{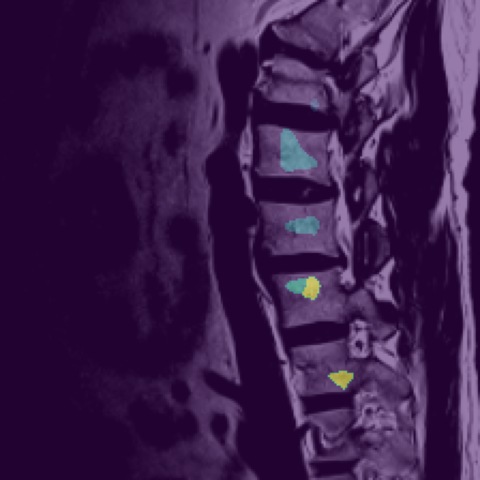}{\draw[white, line width=1.5pt, ->, >=Latex] (0.05,-0.4) -- ++(0.4,-0.4);} &
        \tikzimg{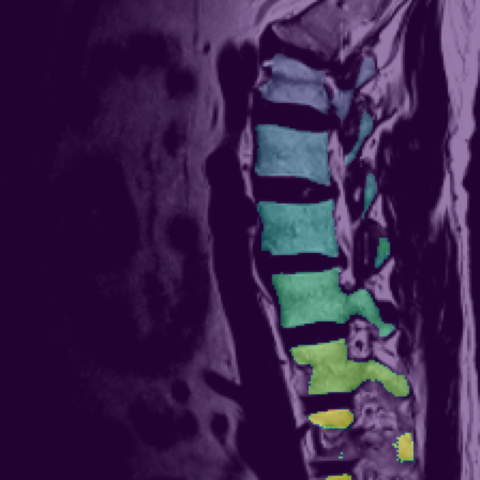}{\draw[white, line width=1.5pt, ->, >=Latex] (0.05,-0.4) -- ++(0.4,-0.4);} &
        \tikzimg{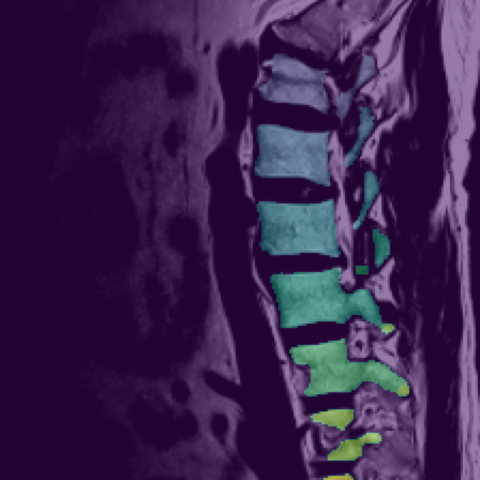}{\draw[white, line width=1.5pt, ->, >=Latex] (0.05,-0.4) -- ++(0.4,-0.4);}  \\
        \end{tabular}
        \vspace{1em}
        \caption{MRI modality results across datasets. Datasets for first and second rows are AMOS \cite{amos}, BraTS \cite{brats 1, brats 2, brats 3} respectively while rows third and fourth are from the Spider \cite{spider} dataset. The anatomical structures pointed by arrows from top to bottom are pancreas, tumor annotation, inter-vertebral disk and vertebrae. More qualitative results can be found in the Appendix Section \ref{sec:sup figs}.}
        \label{fig:qualitative_mri}
    \end{subfigure}

    \caption{Qualitative comparison of segmentation performance across (a) CT and (b) MRI modalities. ViTC-UNet using DINOv2 as encoder demonstrates superior boundary resolution and structure classification compared to hybrid and linear baselines.}
    \label{fig:qualitative_results}
\end{figure}

\paragraph{The Role of Multi-Stage and Multi-Scale Decoder.} The performance gain of the ViT-UNet hybrid over the linear baseline confirms that dense biomedical segmentation benefits significantly from convolutional inductive biases. Moreover, Table \ref{tab:biomedical results} shows that bottleneck-level integration of frozen ViT representations into a fixed-channel UNet-based pixel decoder is not always sufficient. Although the ViT-UNet hybrid is able to delineate the anatomical targets via its skip connections and convolutional blocks, the lack of multi-stage and multi-scale information fusing added to the cross-domain nature of the test might lead to structure misclassifications. In contrast, by making use of the complementary abilities from ViTs and CNNs, our ViTC-UNet using DINOv2 as encoder is able to better detect and correctly recognize the anatomical morphologies. These behaviors are qualitatively illustrated in Figure \ref{fig:qualitative_results}. This suggests that reliable frozen-ViT decoding requires not only convolutional reconstruction, but also target-specific conditioning injected across the multi-stage and multi-scale UNet decoding path. This aligns with prior literature citing the loss of inductive bias in Transformers as a primary hurdle in biomedical tasks, specially in data scarse settings \cite{hybrid 3, hybrid 2, hybrid 1}. Furthermore, our model’s outperformance over nnU-Net indicates that general-purpose ViT representations can effectively complement strong biomedical segmentation pipelines when processed by a structure-conditioned, multi-stage and multi-scale UNet decoder capable of expressing target-specific information from conditioned latent states (Subsection \ref{representing curve}).

\paragraph{Latent Space Geometry.} We observe that our model's performance significantly depends on the backbone's latent space geometry. While DINOv2-based models consistently perform best, we note a performance decrease when employing SAM 2. We hypothesize that DINOv2’s self-supervised regime, which encodes information into orthogonal linear subspaces \cite{orthogonal subspaces}, provides a more separable trajectory for our decoder. In contrast, SAM 2’s prompt-centric supervised objective may produce a latent space less conducive to the homeomorphic transformations modeled in our two-way attention blocks (Section \ref{representing curve}). By tracing a continuous evolution through conditioned decoder states rather than a discrete point, our pixel decoder effectively resolves the structured DINOv2 space into high-fidelity masks.

\paragraph{Label Space Flexibility.} Table \ref{tab:halved results} demonstrates that ViTC-UNet incorporates novel classes without degrading performance on previously learned structures with a reduction of 36\% of training time during fine-tuning. This capacity for label-space expansion follows from the token-conditioned single-channel formulation, which avoids architectural modification when new labels are introduced.

\paragraph{UNet as promptable pixel decoder.} These results validate the feasibility of conditioning a UNet architecture, highlighting a path towards more sophisticated promptable pixel decoders. While natural-domain SOTA models often rely on simplified upsampling for segmentation, our work suggests that integrating high-capacity, conditioned convolutional decoders with frozen Transformer backbones provides the localized precision necessary for rigorous biomedical accuracy.

\section{Conclusions}

In this work, we have demonstrated that SOTA performance in biomedical semantic segmentation can be achieved by conditioning the UNet architecture on frozen ViT latent representations when processed by a multi-stage and multi-scale decoder in cross-domain settings. Our findings represent an advancement in dense prediction by proposing an architecture that combines the flexible conditioning mechanisms of contemporary Transformer-based segmentation models with the high-resolution decoding capacity of UNets for biomedical modalities with intricate morphological complexity, low signal-to-noise ratios and the label scarcity inherent in the field.

Beyond empirical performance, this study enhances our theoretical understanding of the UNet's architectural flexibility as a conditionable pixel decoder. We provide evidence that convolutional layers and skip connections can effectively integrate global, prompt-based inductive biases. By decoupling semantic identity from the output layer, our ViTC-Unet architecture facilitates multi-class segmentation with UNet-level precision without requiring dedicated output channels for each object class. This design supports scalability, where new anatomical structures can be incorporated through the addition of learnable structure tokens rather than architectural modifications.

Furthermore, our results show that ViTs pre-trained on natural images possess useful transferable priors for disparate biomedical modalities when decoded through structure-conditioned UNet reconstruction. By achieving these results without fine-tuning the backbone, we offer a computationally efficient and small-data alternative to the standard paradigm of end-to-end adaptation, which is often prohibitive in data-scarce biomedical domains. Ultimately, this work demonstrates that general-purpose visual representations can complement strong biomedical segmentation pipelines, bridging the gap between large-scale vision models and specialized clinical and biological applications.

\section{Limitations} \label{sec:limitations}

While our framework demonstrates robust performance across diverse biomedical datasets, several limitations remain that offer avenues for future investigation. Firstly, although biomedical volumes are inherently 3D, our current pipeline processes data in a slice-wise 2D manner due to the ViT encoders being pre-trained on 2D natural images. Consequently, the model does not fully exploit the spatial continuity and volumetric context available in 3D scans. Extending the architecture to incorporate 3D native encoders or specialized volumetric attention mechanisms could further leverage volumetric inductive biases during decoding.

Secondly, while our conditioning mechanism effectively utilizes a dictionary of learnable structure tokens, they remain discrete and finite. This poses a challenge for open-vocabulary generalizability. To enhance the flexibility of the architecture across a broader spectrum of anatomical classes, future work could explore generating these tokens through a continuous prompt space, potentially utilizing the latent outputs of Large Language Models.

Furthermore, our model performance required orthogonality on the latent for the encoder. More work is needed to understand how to make use of latent representations in codomains with different geometries.

Finally, while this work focuses on semantic structure tokens, extending the conditioning capability of the UNet to include interactive visual prompts such as points or bounding boxes would be beneficial. Such an extension would further our understanding of the UNet as a high-fidelity pixel decoder compared to the lightweight mask decoders commonly used in Transformer-based segmentation models.

\begin{ack}
J.V.O. received support from the Helmholtz Association under the joint research school "Munich School for Data Science - MUDS".
The authors thank Rushin H. Gindra for his valuable and constructive feedback during writing this paper. J.V.O. wants to further thank Rushin H. Gindra because, although he put colors out of their figures, he puts them into their life.
\end{ack}


\bibliographystyle{alpha}

\begin{thebibliography}{99}

\bibitem[3d-vit-unet (2026)]{3d-vit-unet}
Afridi, S. et al. (2026) ‘3D-VIT-unet: 3D Vision Transformer based unet-like model for volumetric brain tumor segmentation’, PLOS Digital Health, 5(3). doi:10.1371/journal.pdig.0001323. 

\bibitem[Antonelli (2022)]{medical decathlon}
Antonelli, M., Reinke, A., Bakas, S. et al. The Medical Segmentation Decathlon. Nat Commun 13, 4128 (2022). https://doi.org/10.1038/s41467-022-30695-9

\bibitem[microsam (2025)]{microsam}
Archit, A. et al. (2025) ‘Segment anything for Microscopy’, Nature Methods, 22(3), pp. 579–591. doi:10.1038/s41592-024-02580-4. 

\bibitem[SegNet (2017)]{segnet}
Badrinarayanan,V., et al. (2017) SegNet: A Deep Convolutional Encoder-Decoder Architecture for Image Segmentation. IEEE Transactions on Pattern Analysis and Machine Intelligence, 39(12), 2481-2495

\bibitem[brats 1 (2021)]{brats 1}
U.Baid, et al., "The RSNA-ASNR-MICCAI BraTS 2021 Benchmark on Brain Tumor Segmentation and Radiogenomic Classification", arXiv:2107.02314, 2021(opens in a new window).

\bibitem[brats 3 (2017)]{brats 3}
S. Bakas, H. Akbari, A. Sotiras, M. Bilello, M. Rozycki, J.S. Kirby, et al., "Advancing The Cancer Genome Atlas glioma MRI collections with expert segmentation labels and radiomic features", Nature Scientific Data, 4:170117 (2017) DOI: 10.1038/sdata.2017.117


\bibitem[SAM 3 (2025)]{sam3}
Carion, N. et al. (2025) SAM 3: Segment Anything with Concepts. arXiv. https://arxiv.org/abs/2511.16719 

\bibitem[mask2former(2021)]{mask2former}
Cheng, B., et al. (2021) Masked-attention Mask Transformer for Universal Image Segmentation. arXiv. arXiv:2112.01527 

\bibitem[Chi (2020)]{data scarcity 2}
Chi, W. et al. (2020) ‘Deep learning-based medical image segmentation with limited labels’, Physics in Medicine \&amp; Biology, 65(23), p. 235001. doi:10.1088/1361-6560/abc363. 

\bibitem[orthogonal subspaces (2026)]{orthogonal subspaces}
Doshi, F.R, et al. (2026) Bi-Orthogonal Factor Decomposition for Vision Transformers. arXiv. arXiv:2601.05328

\bibitem[Dosovitskiy et al.(2020)]{vit}
Dosovitskiy, A., Beyer, L., Kolesnikov, A., Weissenborn, D., Zhai, X., Unterthiner, T., Dehghani, M., Minderer, M., Heigold, G., Gelly, S., et al.: An image is worth 16x16 words: Transformers for image recognition at scale. arXiv preprint arXiv:2010.11929 (2020)

\bibitem[QIN-PROSTATE (2018)]{QIN-PROSTATE}
Fedorov, A; Schwier, M; Clunie, D; Herz, C; Pieper, S; Kikinis, R; Tempany, C; Fennessy, F. (2018). Data From QIN-PROSTATE-Repeatability. The Cancer Imaging Archive. DOI: 10.7937/K9/TCIA.2018.MR1CKGND

\bibitem[hybrid 3 (2024)]{hybrid 3}
He J, Ma Y, Yang M, Yang W, Wu C, Chen S. TAC-UNet: transformer-assisted convolutional neural network for medical image segmentation. Quant Imaging Med Surg. 2024 Dec 5;14(12):8824-8839. doi: 10.21037/qims-24-1229. Epub 2024 Nov 5. PMID: 39698603; PMCID: PMC11651933.


\bibitem[isles (2022)]{isles}
Hernandez Petzsche, M.R., de la Rosa, E., Hanning, U. et al. ISLES 2022: A multi-center magnetic resonance imaging stroke lesion segmentation dataset. Sci Data 9, 762 (2022). https://doi.org/10.1038/s41597-022-01875-5

\bibitem[Isensee (2024)]{nnunet revisited}
Isensee, F. et al. (2024) ‘NNU-Net Revisited: A Call for rigorous validation in 3D medical image segmentation’, Lecture Notes in Computer Science, pp. 488–498. doi:10.1007/978-3-031-72114-4\_47. 

\bibitem[nnU-Net(2021)]{nnunet}
Isensee, F., Jaeger, P. F., Kohl, S. A., Petersen, J., \& Maier-Hein, K. H. (2021).
nnU-Net: a self-configuring method for deep learning-based biomedical image segmentation.
Nature Methods, 18(2), 203-211.

\bibitem[nnunet brats (2020)]{nnunet brats}
Isensee, F., et al. (2020) nnU-Net for Brain Tumor Segmentation. arXiv. arXiv:2011.00848

\bibitem[amos unet (2022)]{amos unet}
Isensee, F., et al. (2022) Extending nnU-Net is all you need. arXiv. arXiv:2208.10791

\bibitem[chaos (2019)]{chaos}
Kavur, A., et al. (2019) ‘CHAOS - Combined (CT-MR) Healthy Abdominal Organ Segmentation Challenge Data’. The IEEE International Symposium on Biomedical Imaging (ISBI), Zenodo. doi:10.5281/zenodo.3431873.

\bibitem[nnunet chaos (2021)]{nnunet chaos}
Kavur, A.E. et al. (2021) ‘Chaos challenge - combined (CT-MR) healthy abdominal organ segmentation’, Medical Image Analysis, 69, p. 101950. doi:10.1016/j.media.2020.101950. 

\bibitem[SAM (2023)]{sam}
Kirillov, A., Mintun, E., Ravi N., Mao, H., Rolland, C., Gustafson, L., Xiao, T., Whitehead, S., Berg, A., Lo, W., Doll{\'a}r, P., Girshick, R.: Segment Anything. arXiv (2023) arXiv:2304.02643

\bibitem[Koch (2024)]{dinobloom}
Koch, V., et al (2024) DinoBloom: A Foundation Model for Generalizable Cell Embeddings in Hematology. arXiv. arXiv:2404.05022

\bibitem[membrain (2024)]{membrain}
Lamm, L. et al. (2024) MemBrain V2: An end-to-end tool for the analysis of membranes in cryo-electron tomography [Preprint]. doi:10.1101/2024.01.05.574336. 

\bibitem[context aware segmentation (2025)]{context aware segmentation}
Last, M.G., Voortman, L.M. and Sharp, T.H. (2025) Scaling data analyses in cellular cryoET using comprehensive segmentation [Preprint]. doi:10.1101/2025.01.16.633326. 

\bibitem[Li (2022)]{nnunet aorta}
Li, F. et al. (2022) ‘Segmentation of human aorta using 3D NNU-net-oriented deep learning’, Review of Scientific Instruments, 93(11). doi:10.1063/5.0084433. 

\bibitem[heart acdc (2023)]{nnunet heart acdc}
Li, L. et al. (2023) ‘MyoPS: A benchmark of myocardial pathology segmentation combining three-sequence cardiac magnetic resonance images’, Medical Image Analysis, 87, p. 102808. doi:10.1016/j.media.2023.102808. 

\bibitem[Li, et al.(2025)]{frozen mae}
Li, M., et al. (2025) 'Few-Shot Deployment of Pretrained MRI Transformers in Brain Imaging Tasks'. arXiv. arXiv:2508.05783

\bibitem[evit(2026)]{evit}
Li, X., et al. (2025) ‘Evit-UNET: U-net like efficient vision transformer for medical image segmentation on mobile and Edge Devices’, 2025 IEEE 22nd International Symposium on Biomedical Imaging (ISBI), pp. 1–5. doi:10.1109/isbi60581.2025.10981108. 

\bibitem[focal loss (2017)]{focal loss}
Lin, T., et al. (2017) Focal loss for dense object detection. ICCV

\bibitem[cosine model (2025)]{cosine model}
Liu, Y., et al. (2025) Unified Open-World Segmentation with Multi-Modal Prompts. ICCV. arXiv:2510.10524 

\bibitem[convolutional networks (2024)]{convolutional networks}
Long, J. (2015) Fully Convolutional Networks for Semantic Segmentation. Proceedings of the IEEE Conference on Computer Vision and Pattern Recognition (CVPR), 3431-3440.

\bibitem[Loschilov and Hutter(2019)]{adamw}
Loshchilov, I., Hutter, F.: Decoupled weight decay regularization. ICLR (2019)

\bibitem[abdomen1k (2022)]{abdomen1k}
Ma, J., et al. (2022) AbdomenCT-1K: Is Abdominal Organ Segmentation a Solved Problem?. IEEE Transactions on Pattern Analysis and Machine Intelligence. 10.1109/TPAMI.2021.3100536

\bibitem[lungs (2021)]{lungs}
Ma, J. et al. (2021) ‘Toward data‐efficient learning: A benchmark for Covid‐19 CT Lung and infection segmentation’, Medical Physics, 48(3), pp. 1197–1210. doi:10.1002/mp.14676. 

\bibitem[medsam (2025)]{medsam}
Ma, J. et al. (2024) ‘Segment anything in Medical Images’, Nature Communications, 15(1). doi:10.1038/s41467-024-44824-z. 

\bibitem[Matsoukas (2022)]{transfer learning 1}
Matsoukas, C., et al (2022) What Makes Transfer Learning Work For Medical Images: Feature Reuse \& Other Factors. arXiv. 	arXiv:2203.01825

\bibitem[Matsoukas (2023)]{transfer learning 2}
Matsoukas, C., et al. (2023) Pretrained ViTs Yield Versatile Representations For Medical Images. arXiv. arXiv:2303.07034 

\bibitem[nnunet baseline fundus (2024)]{nnunet baseline fundus}
Mehmood, M. (2024) LVS-Net: A Lightweight Vessels Segmentation Network for Retinal Image Analysis. arXiv. arXiv:2412.05968v1

\bibitem[brats 2 (2015)]{brats 2}
Menze, B.H. et al. (2015) ‘The Multimodal Brain Tumor Image Segmentation Benchmark (brats)’, IEEE Transactions on Medical Imaging, 34(10), pp. 1993–2024. doi:10.1109/tmi.2014.2377694.  

\bibitem[dice loss (2016)]{dice loss}
Milletari, F., et al. (2016) V-Net:Fully convolutional neural networks for volumetric medical image segmentation. 3DV

\bibitem[DINOv2 (2023)]{dinov2}
Oquab, M., Darcet, T., Moutakanni, T., Vo, H., Szafraniec, M., Khalidov, V., Fernandez, P., Haziza, D., Massa, F., El-Nouby, A., et al.: Dinov2: Learning robust visual features without supervision. arXiv preprint arXiv:2304.07193 (2023)

\bibitem[Pal (2025)]{nnunet abdtumor_adrenal}
Pal, D. et al. (2025) ‘Pannet: A feature-based attention aggregation model for segmenting pancreatic ductal adenocarcinoma on contrast-enhanced CT images of the abdomen’, Medical Imaging 2025: Computer-Aided Diagnosis, p. 63. doi:10.1117/12.3048971. 

\bibitem[Payer (2023)]{data scarcity 1}
Payer, T. et al. (2023) ‘Medical volume segmentation by overfitting sparsely annotated data’, Journal of Medical Imaging, 10(04). doi:10.1117/1.jmi.10.4.044007. 

\bibitem[han (2023)]{han}
Podobnik, G. et al. (2023) ‘Han‐Seg: The head and neck organ‐at‐risk CT and mr segmentation dataset’, Medical Physics, 50(3), pp. 1917–1927. doi:10.1002/mp.16197. 

\bibitem[aorta (2022)]{aorta}
Radl, Lukas; Jin, Yuan; Pepe, Antonio; Li, Jianning; Gsaxner, Christina; Zhao, Fen-hua; et al. (2022). Aortic Vessel Tree (AVT) CTA Datasets and Segmentations. figshare. Dataset. https://doi.org/10.6084/m9.figshare.14806362.v1

\bibitem[SAM 2 (2024)]{sam2}
Ravi, N., et al. (2024) SAM 2: Segment Anything in Images and Videos. arXiv. https://arxiv.org/abs/2408.00714

\bibitem[dpt (2021)]{dpt}
Ranftl, R., Bochkovskiy, A. and Koltun, V. (2021) Vision Transformers for dense prediction, 2021 IEEE/CVF International Conference on Computer Vision (ICCV), pp. 12159–12168. doi:10.1109/iccv48922.2021.01196. 

\bibitem[Ronenber et al (2015)]{unet_biomedical}
Ronneberger, O. (2015) U-Net: Convolutional Networks for Biomedical Image Segmentation. arXiv. arXiv:1505.04597 

\bibitem[nnunet isles (2025)]{nnunet isles}
de la Rosa, E., Reyes, M., Liew, SL. et al. DeepISLES: a clinically validated ischemic stroke segmentation model from the ISLES'22 challenge. Nat Commun 16, 7357 (2025). https://doi.org/10.1038/s41467-025-62373-x


\bibitem[Sablayrolles, et al.(2019)]{koleo}
Sablayrolles A., Douze M., Schmid C., and Jégou H.: Spreading vectors for similarity search. ICLR (2019)

\bibitem[hybrid 2 (2025)]{hybrid 2}
Sang, Y. et al. (2025) Benchmark of Segmentation Techniques for Pelvic Fracture in CT and X-ray: Summary of the PENGWIN 2024 Challenge. arXiv. arXiv:2504.02382

\bibitem[hybrid 1 (2023)]{hybrid 1}
Soh WK and Rajapakse JC (2023) Hybrid UNet transformer architecture for ischemic stoke segmentation with MRI and CT datasets. Front. Neurosci. 17:1298514. doi: 10.3389/fnins.2023.1298514

\bibitem[airwaytree (2024)]{airwaytree}
Støverud, K.-H. et al. (2024) ‘AeroPath: An airway segmentation benchmark dataset with challenging pathology and Baseline Method’, PLOS ONE, 19(10). doi:10.1371/journal.pone.0311416. 

\bibitem[Valdivia Ortega, et al.(2025)]{rmlp}
Valdivia Ortega, J., et al. (2025) Randomized-MLP Regularization Improves Domain Adaptation and Interpretability in DINOv2. NeurIPS. arXiv:2511.05509

\bibitem[nnunet spider (2024)]{nnunet spider}
van der Graaf, J.W., van Hooff, M.L., Buckens, C.F.M. et al. Lumbar spine segmentation in MR images: a dataset and a public benchmark. Sci Data 11, 264 (2024). https://doi.org/10.1038/s41597-024-03090-w

\bibitem[spider (2023)]{spider}
van der Graaf, J., et al. (2023) ‘SPIDER - Lumbar spine segmentation in MR images: a dataset and a public benchmark’. Zenodo. doi:10.5281/zenodo.10159290.


\bibitem[totalseg (2023)]{totalseg}
Wasserthal, J. (2023) ‘Dataset with segmentations of 117 important anatomical structures in 1228 CT images’. Zenodo. doi:10.5281/zenodo.10047292.

\bibitem[totalsegmentator (2023)]{totalsegmentator}
Wasserthal, J. et al. (2023) ‘TotalSegmentator: Robust segmentation of 104 anatomic structures in CT images’, Radiology: Artificial Intelligence, 5(5). doi:10.1148/ryai.230024. 

\bibitem[Wei, et al.(2024)]{frozen dinov2}
Wei, M., et al. (2024) 'Enhancing surgical instrument segmentation: integrating vision transformer insights with adapter' Int J Comput Assist Radiol Surg. 10.1007/s11548-024-03140-z

\bibitem[nnunet baseline brightfield 1 (2026)]{nnunet baseline brightfield 1}
Xu, Q. et al. (2026) ‘Robust multi-domain digital pathology image segmentation via joint balancing representation learning’, Expert Systems with Applications, 320, p. 132093. doi:10.1016/j.eswa.2026.132093. 

\bibitem[amos (2022)]{amos}
Yuanfeng, J., et al. (2022) AMOS: A Large-Scale Abdominal Multi-Organ Benchmark for Versatile Medical Image Segmentation. arXiv. arXiv:2206.08023

\bibitem[BiomedParse (2024)]{biomedparse}
Zhao, T. et al. (2024) ‘A Foundation model for joint segmentation, detection and recognition of biomedical objects across nine modalities’, Nature Methods, 22(1), pp. 166–176. doi:10.1038/s41592-024-02499-w. 

\bibitem[prostate lesion (2020)]{prostate lesion}
Zhu, Q., Du, B. and Yan, P. (2020) ‘Boundary-weighted domain adaptive neural network for prostate mr image segmentation’, IEEE Transactions on Medical Imaging, 39(3), pp. 753–763. doi:10.1109/tmi.2019.2935018. 



\end{thebibliography}

\medskip

\newpage

\newpage
\appendix

\section{Technical Appendices and Supplementary Material}

\subsection{Licenses}\label{sec:licenses}
The datasets used in this paper where obtained as part of the compilation made by BiomedParse \cite{biomedparse} who release it under the Creative Commons Attribution Non Commercial Share Alike 4.0 license. This work adheres to this license.

This paper used the models DINOv2 \cite{dinov2} and SAM 2 \cite{sam2} who were released under a Apache License 2.0. Their use adhered to this license in this paper.

\subsection{Mathematical Definitions and Proofs} \label{sec:math}

We provide some useful topological definitions and mathematical demonstrations to back up our decision to use the output of intermediate attention blocks instead of only the output of the last one.

\begin{definition}\label{compact set}
    Given \(X\) a topological space and \(A\subseteq X\), we say that \(A\) is compact if for every open cover of \(A\), i.e. \(\{U_i\}_{i\in I}\) such \(A\subseteq \bigcup\limits_{i\in I} \) where \(U_i\) is open for every \(i\in I\), there is \(\{i_0,\cdots,i_n\}\subseteq I\) such that \(A\subseteq \bigcup\limits_{n=0}^{N} \).
\end{definition}

\begin{theorem}\label{heine-borel}
A subset in \(\mathbb{R}^n\) is compact if and only if it is closed and bounded.
\end{theorem}

\begin{definition}\label{injectivity}
    Let \(f:X\rightarrow Y\) be a function. We say that \(f\) is injective if for every \(\{x,y\}\subseteq X\) such that \(x\neq y\), we have that \(f(x)\neq f(y)\).
\end{definition}
\begin{definition}\label{continuous}
    Given we have a function \(f:X\rightarrow Y\) with \(X\) and \(Y\) being topological spaces, we say that \(f\) is \textit{continuous} if, for every \(U\subseteq Y\) open, its inverse image under \(f\), i.e. \(f^{-1}[U]\), is open.
\end{definition}

\begin{definition}\label{open function}
    Given we have a function \(f:X\rightarrow Y\) with \(X\) and \(Y\) being topological spaces, we say that \(f\) is \textit{open} if, for every \(U\subseteq Y\) open, its image under \(f\), i.e. \(f[U]\), is open.
\end{definition}

\begin{definition}\label{homeomorphism}
    Given we have a function \(f:X\rightarrow Y\) with \(X\) and \(Y\) being topological spaces, we say that \(f\) is an \textit{homeomorphism} if \(f\) is continuous and invertible and its inverse, \(f^{-1}:Y\rightarrow X\) is also continuous. 
\end{definition}

\begin{remark}
    If \(f:X\rightarrow Y\) is an homeomorphism if and only if it is a continuous, open and invertible function.
\end{remark}
\begin{proof}
    Since we already have invertibility and continuity by hypothesis, we just need to prove openness and continuity for the inverse. 
    
    This way, let us assume that \(f\) is an homeomorphism and prove that \(f\) needs to be open. Then, \(U\) being open, we have that \((f^{-1})^{-1}[U]\) is open due to \(f^{-1}\) being continuous. Furthermore, \((f^{-1})^{-1}[U]=f[U]\).

    \(\therefore f\) is open.

    Let us now prove that \(f\) being open in this scenario implies that \(f^{-1}\) is continuous. If \(f\) is an open function and \(U\subseteq X\) is an open set, we have that \(f[U]\) is open while at the same time \(f[U]=(f^{-1})^{-1}[U]\).

    \(\therefore f^{-1}\) is continuous.
\end{proof}

\begin{lemma}\label{extended proof}
    If \(f:X\rightarrow Y\) is an homeomorphism and \(f=f_n\circ\cdots\circ f_1\), then \(f_i\) is an homeomorphism for every \(i\).
\end{lemma}

\begin{proof}
    We will show this results by mathematical induction. The initial step, i.e. \(f=f_1\) is true by hypothesis. Thus, let us take our induction hypothesis to be that our statement is true for \(g=f_n\circ\cdots\circ f_1\) for \(n\in \mathbb{N}\). 

    Without loss of generality, let us assume \(f_i:X\rightarrow X\) for every \(i\in\{1,\cdots n+1\}\) and take \(f:=f_{n+1}\circ g\) to be a homeomorphism. This way, let us begin by proving that \(f_{n+1}\) is injective.

    Since \(f_i\) is an homeomorphism for every \(i\in\{1,\cdots n\}\), \(g\) is an homeomorphism as well. Thus, let us take \(\{x,y\} \subseteq X\) such that \(x\neq y\). Then \(\exists \{u,w\} \subseteq X\) such that \(g(u)=x\) and \(g(w)=y\) since \(g\) is surjective. In addition, \(u\neq w\) due to \(g\) being a function and so it follows that \(f(u)\neq f(w)\) because of \(f\) being an homeomorphism. This way we have that \[f_{n+1}(x)=f_{n+1}(g(u)=f(u)\neq f(w)=f_{n+1}(g(w))=f_{n+1}(y).\]

    \(\therefore f_{n+1}\) is injective.

    To prove surjectivity, let us take \(y\in X\). Then \(x\in X\) such that \(f(x)=y\), but then \(f_{n+1}(g(x))=y\) where \(g(x)\in X\).

    \(\therefore f_{n+1}\) is surjective.

    Moving on for continuity, let us tke \(U\subseteq X\) open \(\Rightarrow f^{-1}[U]\) is open \(\Rightarrow g[f^{-1}[U]]\) is open, where we have that \[g[f^{-1}[u]]=g[(f_{n+1}\circ g)^{-1}[U]]=g\circ g^{-1} \circ f_{n+1}^{-1}[U]=f_{n+1}^{-1}[U].\]

    \(\therefore f_{n+1}\) is continuous.

    Finally, to prove that \(f_{n+1}\) is open, let us take \(U\subseteq X\) open \(\Rightarrow g^{-1}[U]\) is open out of continuity from \(g\) \(\Rightarrow f[g^{-1}[U]]\) is open out of \(f\) being open. This way we arrive to \[f\circ g^{-1}[U]=f_{n+1}\circ g\circ g^{-1}[U]=f_{n+1}[U].\]

    \(\therefore f_{n+1}\) is open.

    \(\therefore f_{n+1}\) is an homeomorphism.
    
\end{proof}

\begin{remark}\label{concatenation of homeomorphisms}
    If a function \(f:X\rightarrow Y\) is such that \(f=f_1\circ f_2\circ\cdots \circ f_n\), then \(f\) is an homeomorphism if an only if \(f_i\) is an homeomorphism for every \(i\).
\end{remark}
\begin{proof}
    The proof that if \(f_i\) being an homeomorphism for every \(i\) implies that \(f\) is an homeomorphism too comes directly from the fact that injectivity, surjectivity and being continuous are properties preserved by the finite composition of functions.

    The proof that \(f\) being an homeomorphism forces \(f_i\) to be an homeomorphism for every finite \(i\) follows form Lemma \ref{extended proof}.
\end{proof}  

\subsection{Supplementary Figures}\label{sec:sup figs}

In this section, we provide more qualitative examples on the performance of our proposed ViTC-UNet when using DINOv2 as encoder for segmenting anatomical structures on CT and MRI modalities.

\begin{figure}[h!]
    \centering
    \newcommand{\imgcell}[1]{%
        \begin{minipage}[c][1.9cm][c]{\linewidth}
            \centering
            \includegraphics[width=\linewidth]{#1}
        \end{minipage}%
    }
    \newcommand{\tikzimg}[2]{%
        \begin{minipage}[c][1.9cm][c]{\linewidth}
            \centering
            \begin{tikzpicture}
                \node[inner sep=0pt] (image) {\includegraphics[width=\linewidth]{#1}};
                #2
            \end{tikzpicture}
        \end{minipage}%
    }
    \setlength{\tabcolsep}{7pt}
    \renewcommand{\arraystretch}{5.5} 

    \begin{subfigure}{\textwidth}
        \centering
        \begin{tabular}{>{\centering\arraybackslash}p{0.16\textwidth} >{\centering\arraybackslash}p{0.16\textwidth} >{\centering\arraybackslash}p{0.16\textwidth} >{\centering\arraybackslash}p{0.16\textwidth} >{\centering\arraybackslash}p{0.16\textwidth}}
            \textbf{CT Image} & \textbf{Ground Truth} & \textbf{Linear} & \textbf{ViT-UNet hybrid} & \textbf{ViTC-UNet} \\[-2.0em]
            
            \imgcell{examples/overlays/ct/vitc_unet/lungs/slice15_raw.png} &
            \imgcell{examples/overlays/ct/vitc_unet/lungs/slice15_target.png} &
            \imgcell{examples/overlays/ct/linear/lungs/slice15_pred.png} &
            \imgcell{examples/overlays/ct/hybrid/lungs/slice15_pred.png} &
            \imgcell{examples/overlays/ct/vitc_unet/lungs/slice15_pred.png} \\
            
            \imgcell{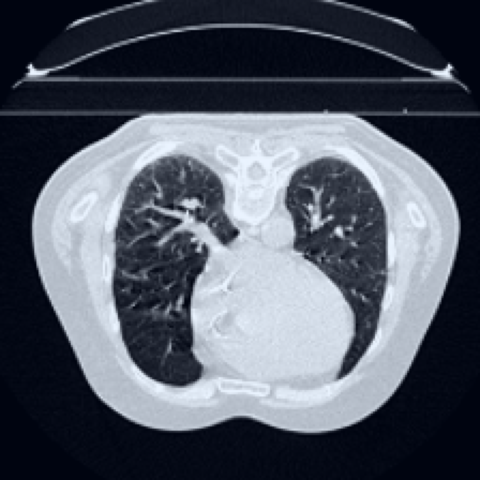} &
            \imgcell{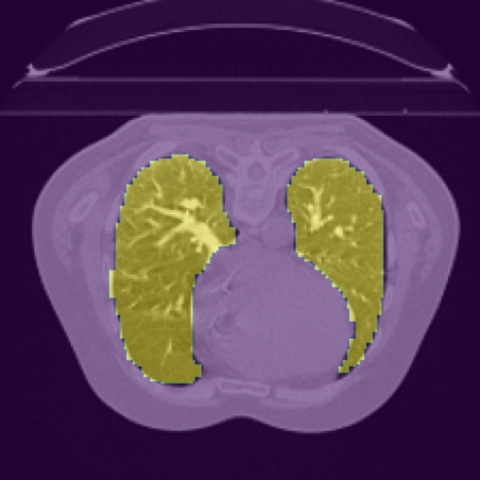} &
            \imgcell{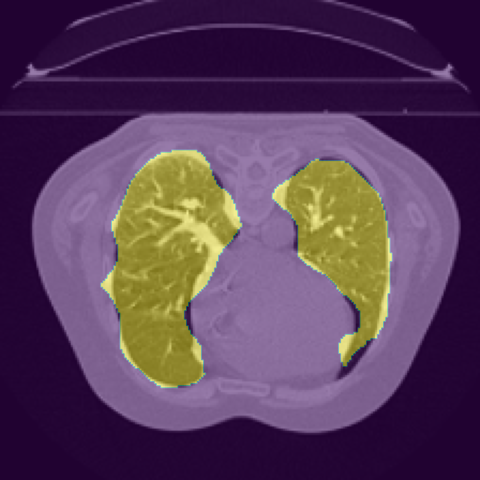} &
            \imgcell{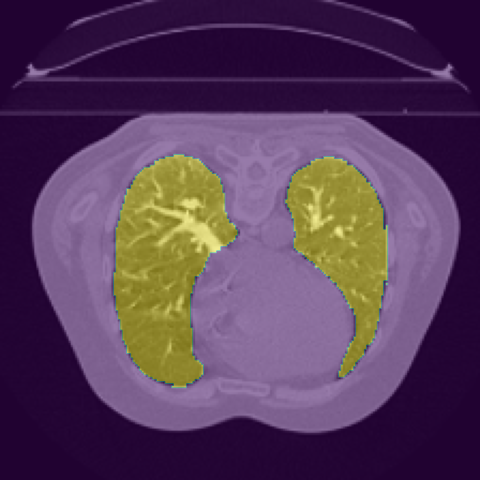} &
            \imgcell{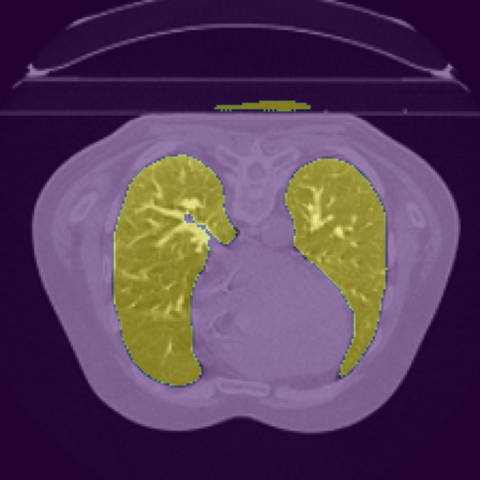}\\

            \imgcell{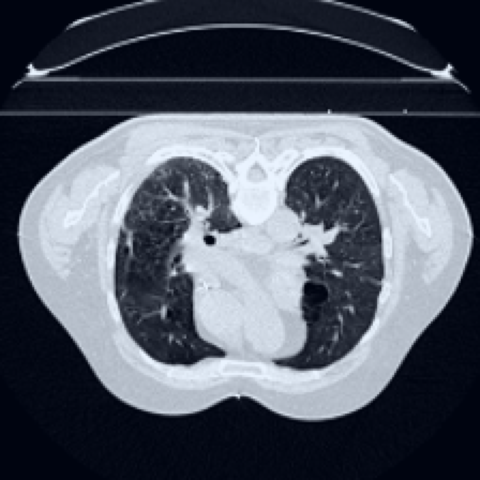} &
            \imgcell{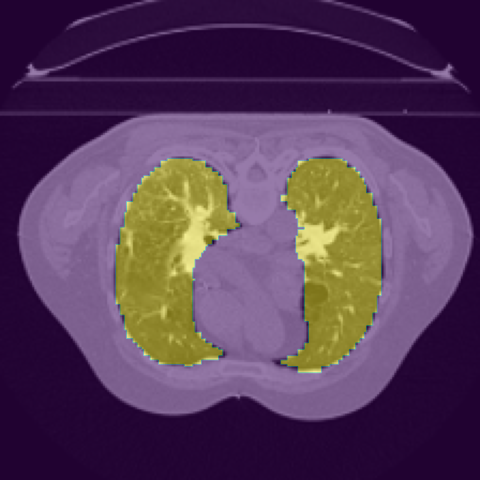} &
            \imgcell{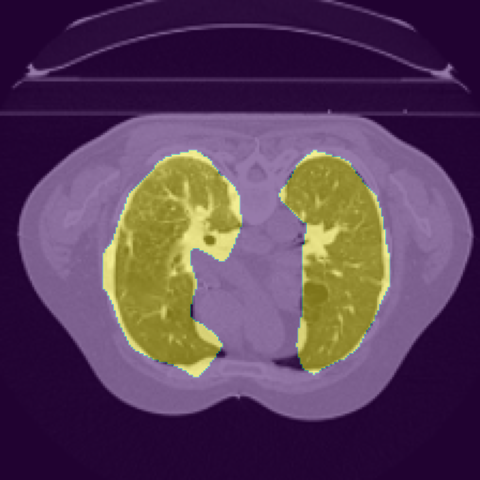} &
            \imgcell{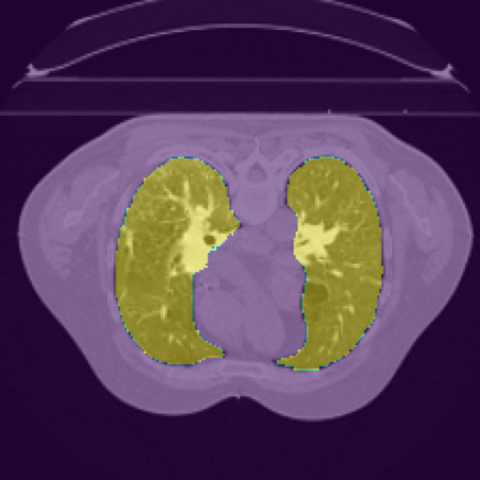} &
            \imgcell{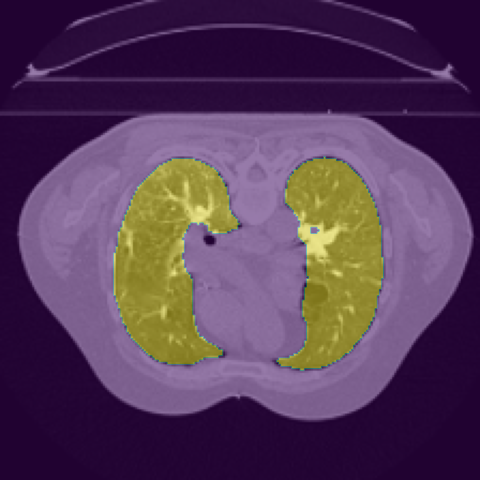}\\

            \imgcell{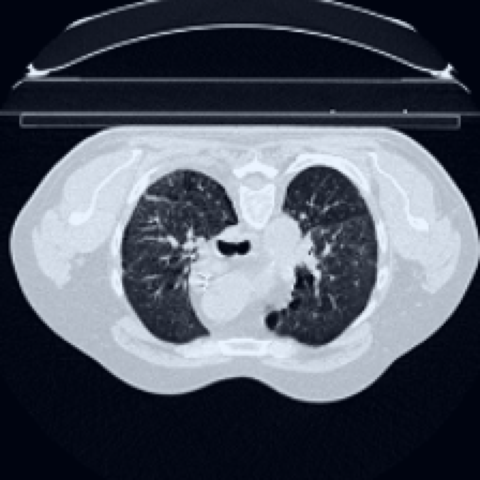} &
            \imgcell{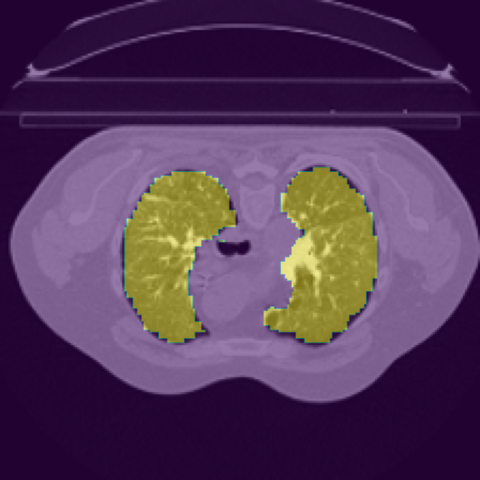} &
            \imgcell{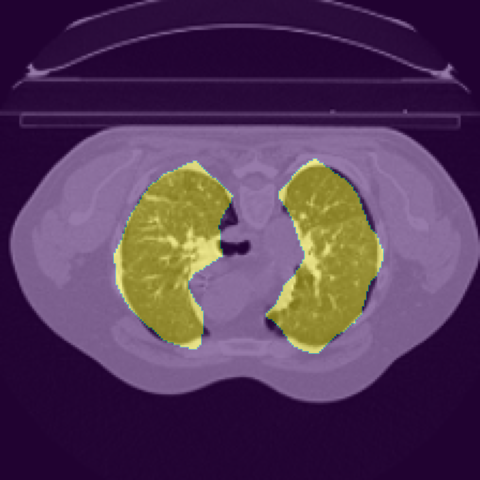} &
            \imgcell{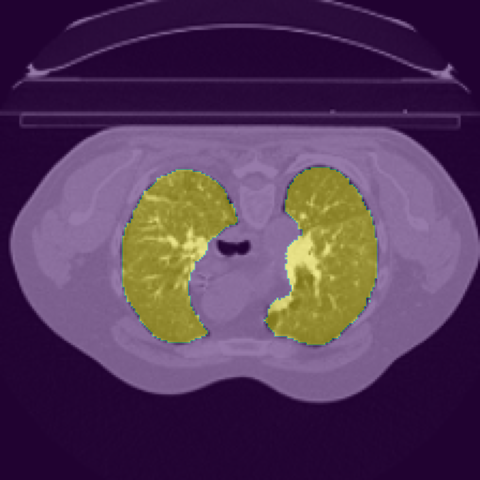} &
            \imgcell{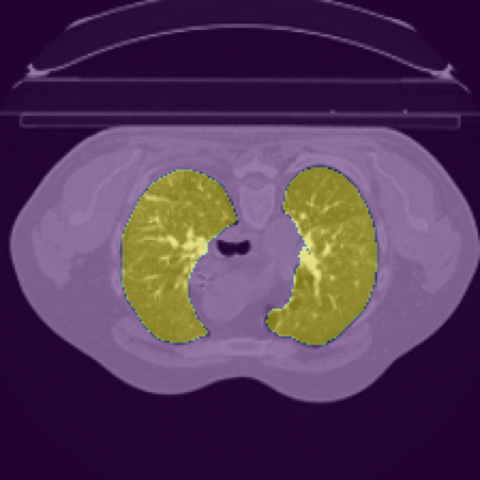}\\

            \imgcell{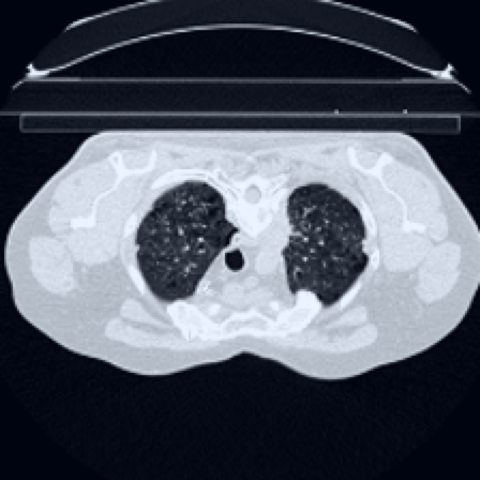} &
            \imgcell{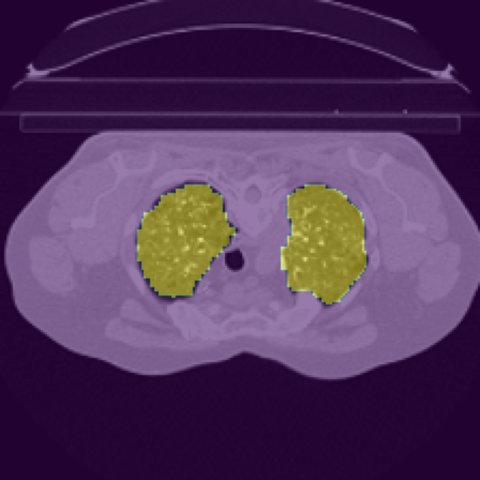} &
            \imgcell{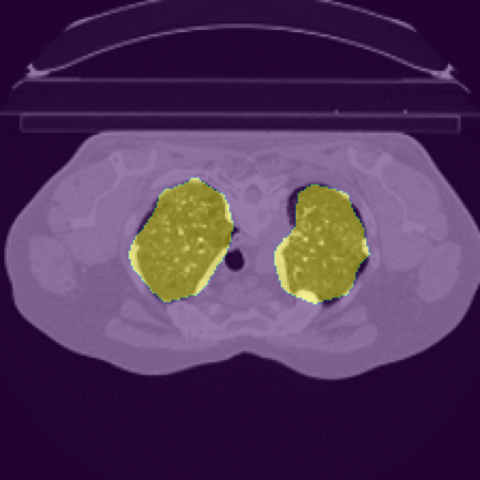} &
            \imgcell{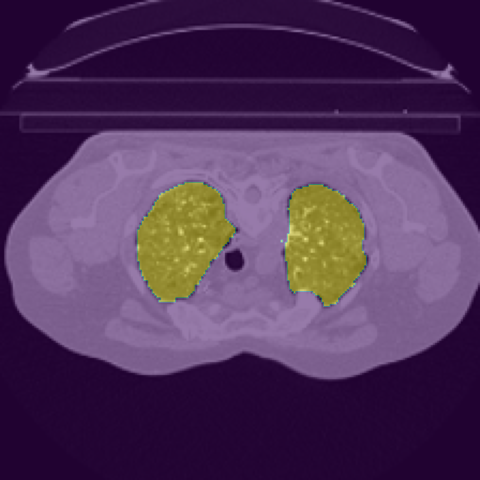} &
            \imgcell{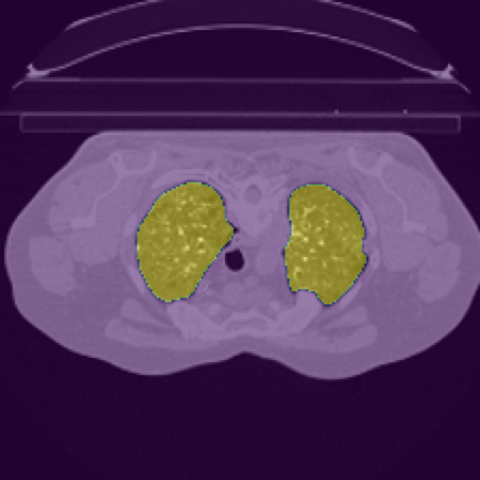}\\
            
        \end{tabular}
    \end{subfigure}

    \caption{Comparison of segmentation performance between a linear baseline, the ViT-UNet hybrid \cite{rmlp}, and our proposed ViTC-UNet utilizing DINOv2 \cite{dinov2} as encoder for semantic segmentation on the Lungs \cite{lungs} dataset.}
\end{figure}
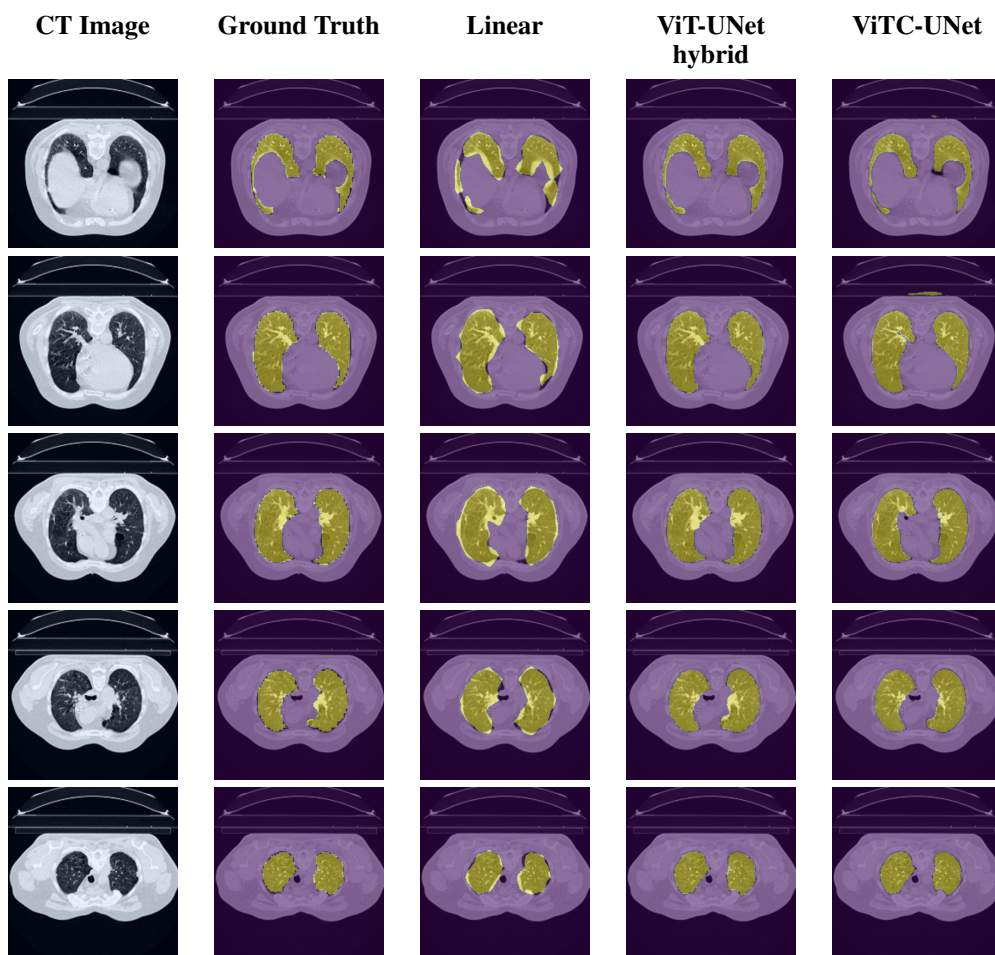

\begin{figure}[h!]
    \centering
    \newcommand{\imgcell}[1]{%
        \begin{minipage}[c][1.9cm][c]{\linewidth}
            \centering
            \includegraphics[width=\linewidth]{#1}
        \end{minipage}%
    }
    \newcommand{\tikzimg}[2]{%
        \begin{minipage}[c][1.9cm][c]{\linewidth}
            \centering
            \begin{tikzpicture}
                \node[inner sep=0pt] (image) {\includegraphics[width=\linewidth]{#1}};
                #2
            \end{tikzpicture}
        \end{minipage}%
    }
    \setlength{\tabcolsep}{7pt}
    \renewcommand{\arraystretch}{5.5} 

    \begin{subfigure}{\textwidth}
        \centering
        \begin{tabular}{>{\centering\arraybackslash}p{0.16\textwidth} >{\centering\arraybackslash}p{0.16\textwidth} >{\centering\arraybackslash}p{0.16\textwidth} >{\centering\arraybackslash}p{0.16\textwidth} >{\centering\arraybackslash}p{0.16\textwidth}}
            \textbf{CT Image} & \textbf{Ground Truth} & \textbf{Linear} & \textbf{ViT-UNet hybrid} & \textbf{ViTC-UNet} \\[-2.0em]
            
            \imgcell{examples/overlays/ct/vitc_unet/totalseg_organs/slice88_raw.png} &
            \imgcell{examples/overlays/ct/vitc_unet/totalseg_organs/slice88_target.png} &
            \imgcell{examples/overlays/ct/linear/totalseg_organs/slice88_pred.png} &
            \imgcell{examples/overlays/ct/hybrid/totalseg_organs/slice88_pred.png} &
            \imgcell{examples/overlays/ct/vitc_unet/totalseg_organs/slice88_pred.png} \\

            \imgcell{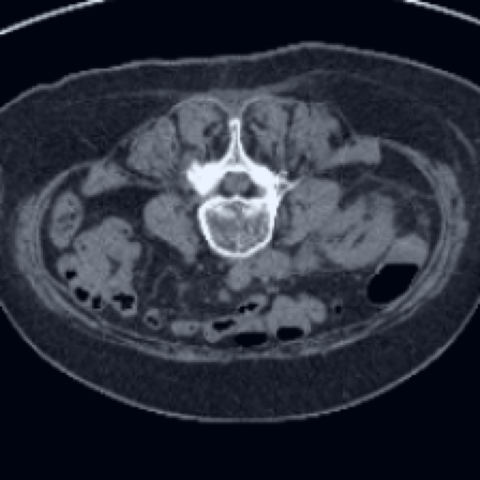} &
            \imgcell{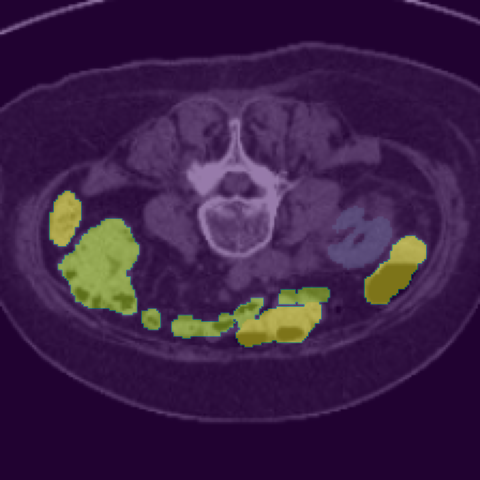} &
            \imgcell{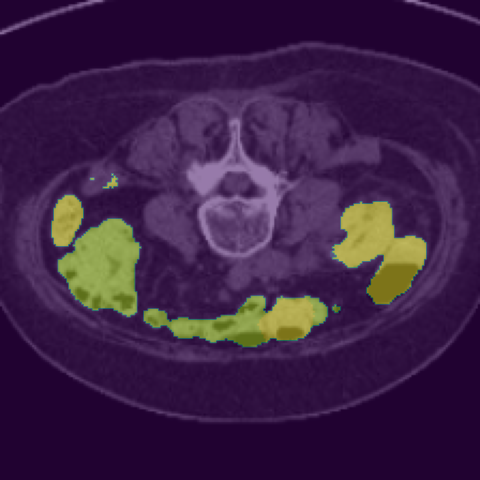} &
            \imgcell{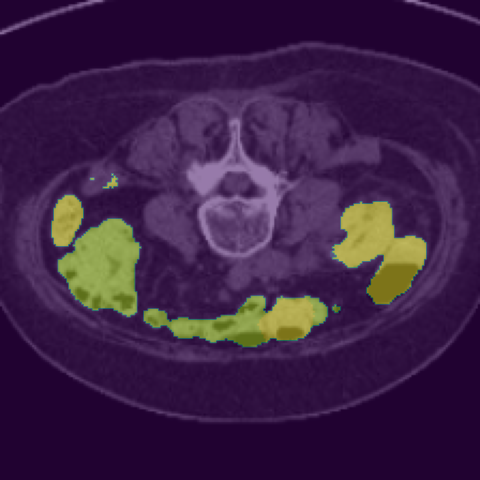} &
            \imgcell{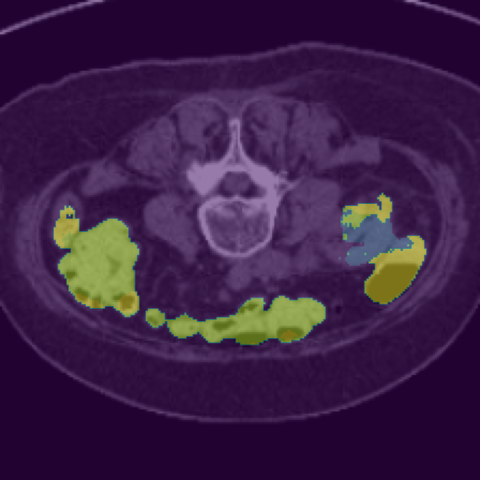}\\

            \imgcell{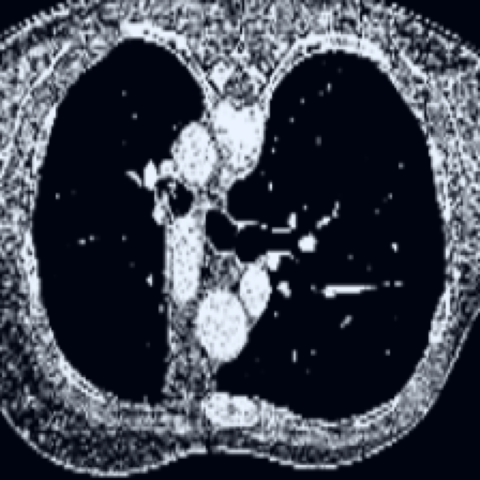} &
            \imgcell{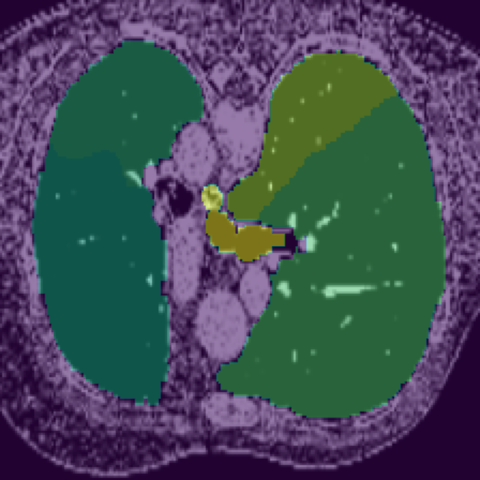} &
            \imgcell{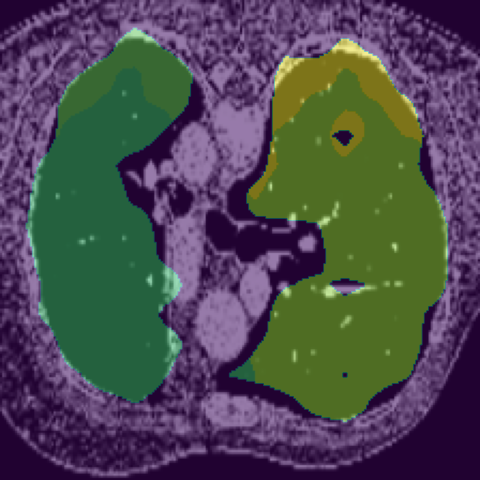} &
            \imgcell{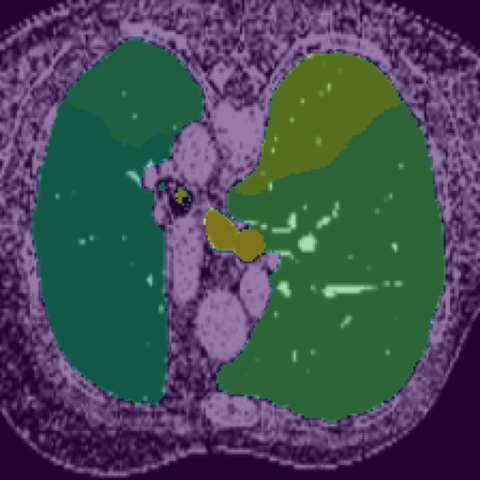} &
            \imgcell{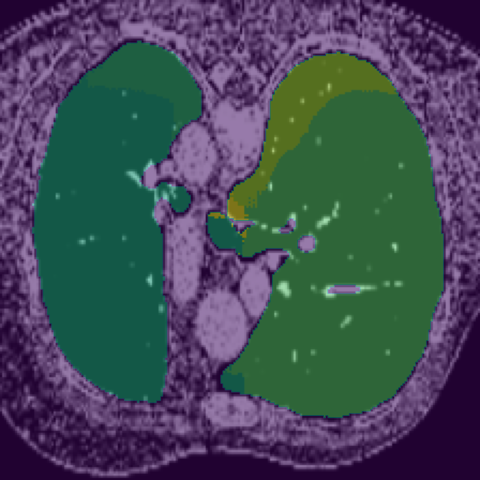}\\

            \imgcell{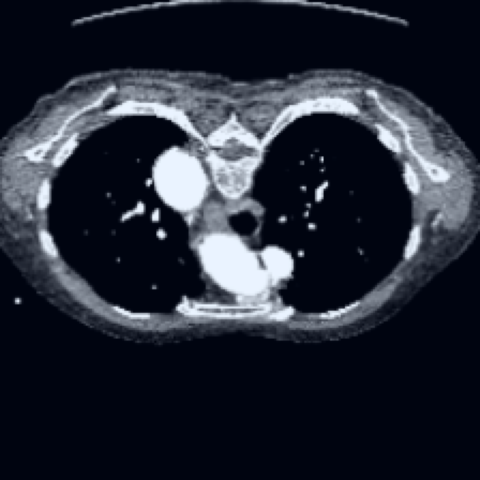} &
            \imgcell{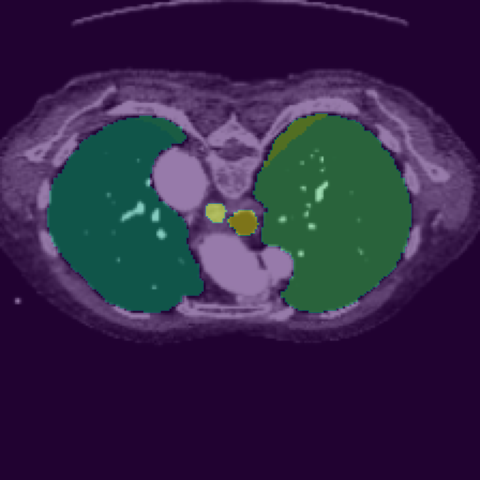} &
            \imgcell{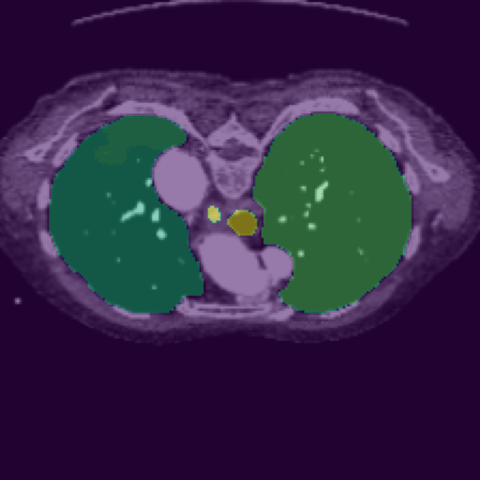} &
            \imgcell{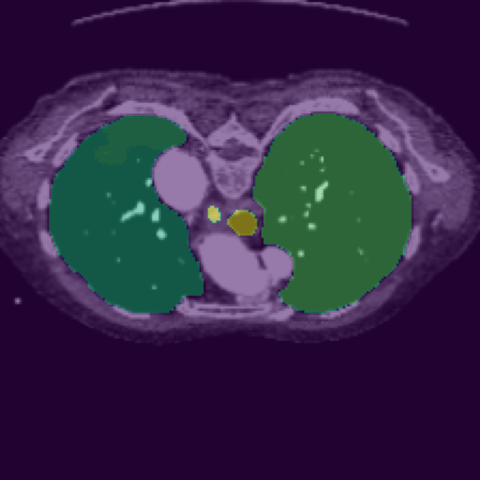} &
            \imgcell{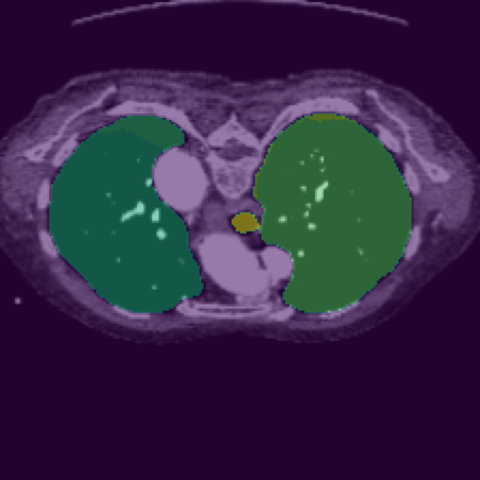}\\
            
        \end{tabular}
    \end{subfigure}

    \caption{Comparison of segmentation performance between a linear baseline, the ViT-UNet hybrid \cite{rmlp}, and our proposed ViTC-UNet utilizing DINOv2 \cite{dinov2} as encoder for semantic segmentation on the Totalseg \cite{totalseg} dataset.}
\end{figure}

\begin{figure}[h!]
    \centering
    \newcommand{\imgcell}[1]{%
        \begin{minipage}[c][1.9cm][c]{\linewidth}
            \centering
            \includegraphics[width=\linewidth]{#1}
        \end{minipage}%
    }
    \newcommand{\tikzimg}[2]{%
        \begin{minipage}[c][1.9cm][c]{\linewidth}
            \centering
            \begin{tikzpicture}
                \node[inner sep=0pt] (image) {\includegraphics[width=\linewidth]{#1}};
                #2
            \end{tikzpicture}
        \end{minipage}%
    }
    \setlength{\tabcolsep}{7pt}
    \renewcommand{\arraystretch}{5.5} 

    \begin{subfigure}{\textwidth}
        \centering
        \begin{tabular}{>{\centering\arraybackslash}p{0.16\textwidth} >{\centering\arraybackslash}p{0.16\textwidth} >{\centering\arraybackslash}p{0.16\textwidth} >{\centering\arraybackslash}p{0.16\textwidth} >{\centering\arraybackslash}p{0.16\textwidth}}
            \textbf{CT Image} & \textbf{Ground Truth} & \textbf{Linear} & \textbf{ViT-UNet hybrid} & \textbf{ViTC-UNet} \\[-2.0em]
            
            \imgcell{examples/overlays/mri/vitc_unet/amos/slice31_raw.png} &
            \imgcell{examples/overlays/mri/vitc_unet/amos/slice31_target.png} &
            \imgcell{examples/overlays/mri/linear/amos/slice31_pred.png} &
            \imgcell{examples/overlays/mri/hybrid/amos/slice31_pred.png} &
            \imgcell{examples/overlays/mri/vitc_unet/amos/slice31_pred.png} \\
            
            \imgcell{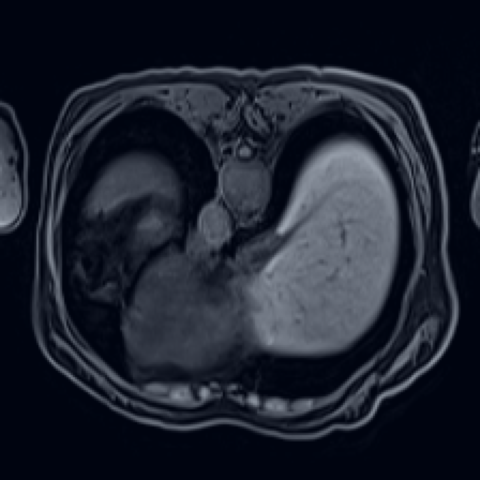} &
            \imgcell{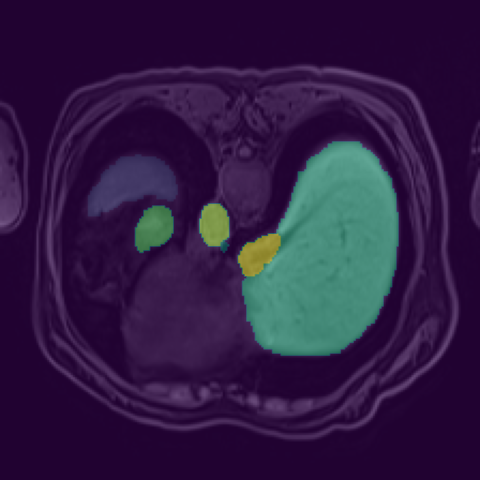} &
            \imgcell{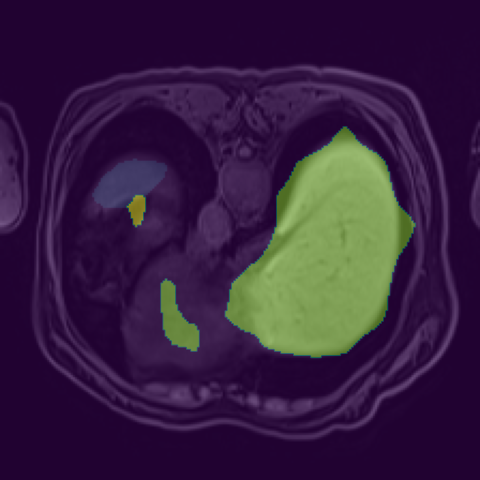} &
            \imgcell{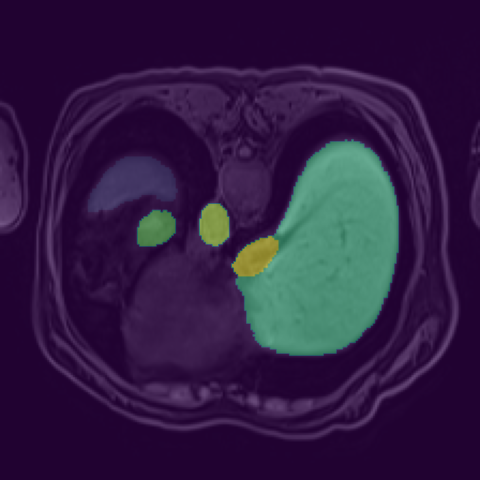} &
            \imgcell{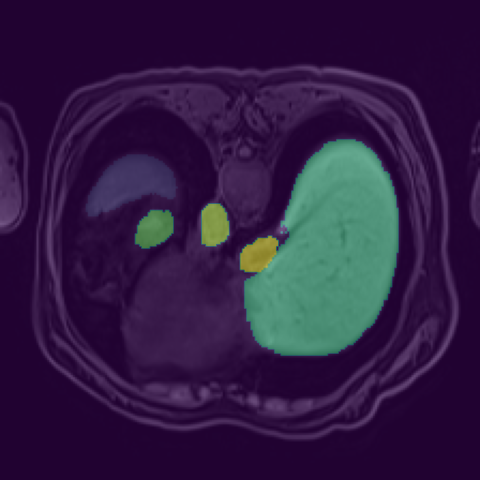} \\

            \imgcell{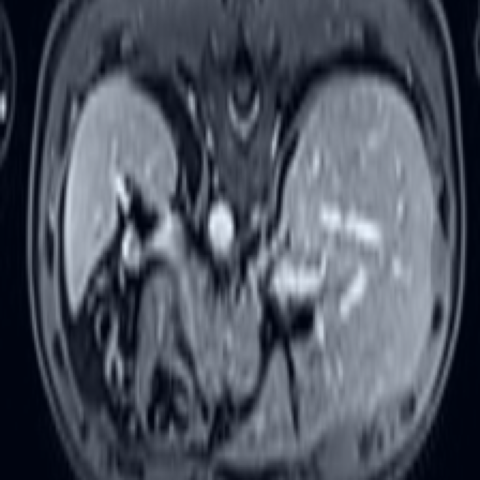} &
            \imgcell{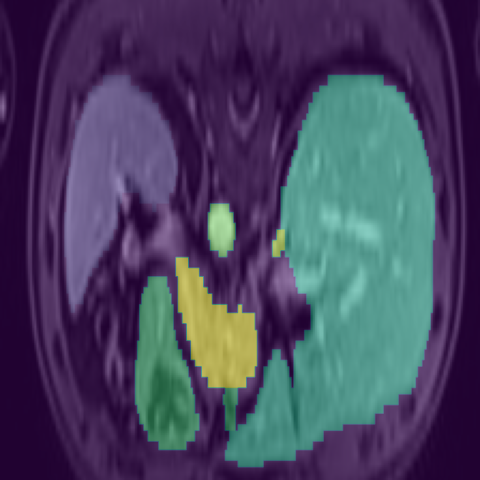} &
            \imgcell{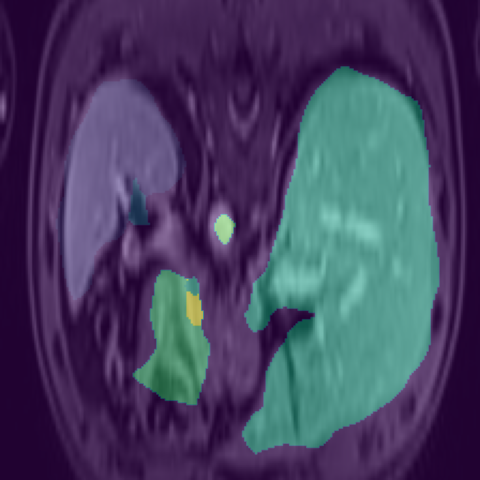} &
            \imgcell{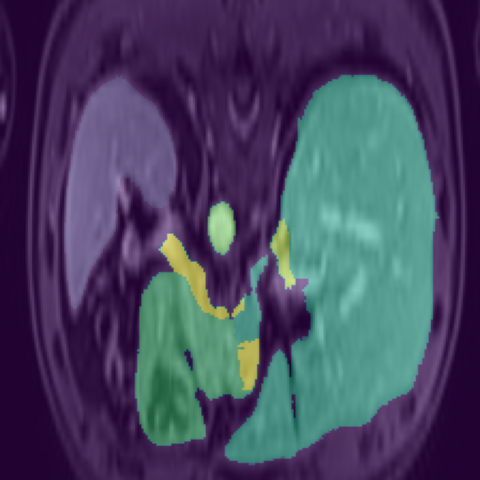} &
            \imgcell{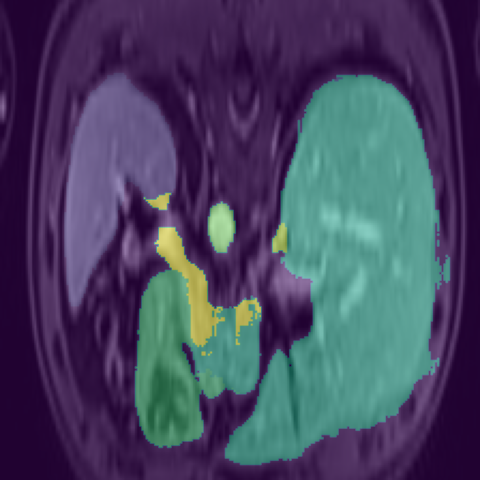} \\

            \imgcell{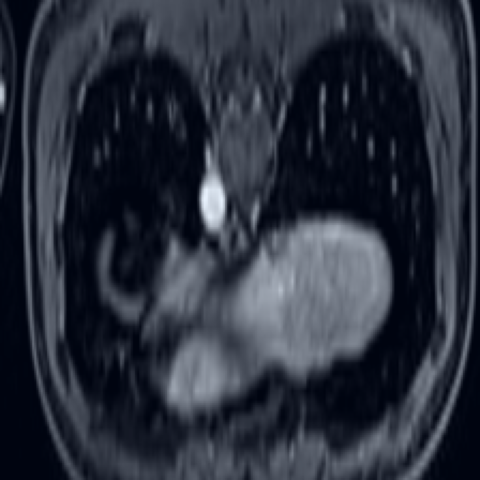} &
            \imgcell{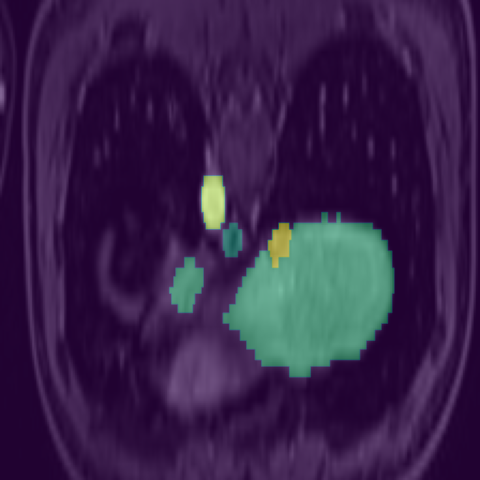} &
            \imgcell{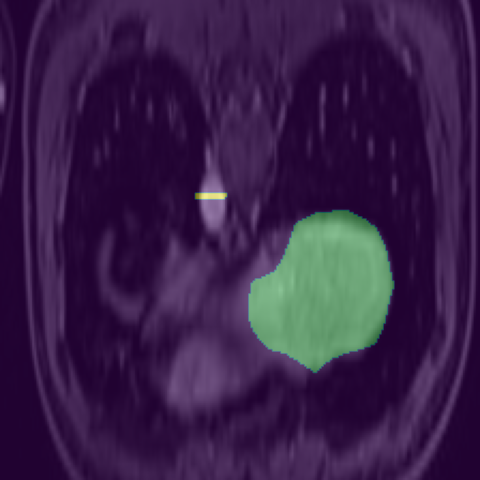} &
            \imgcell{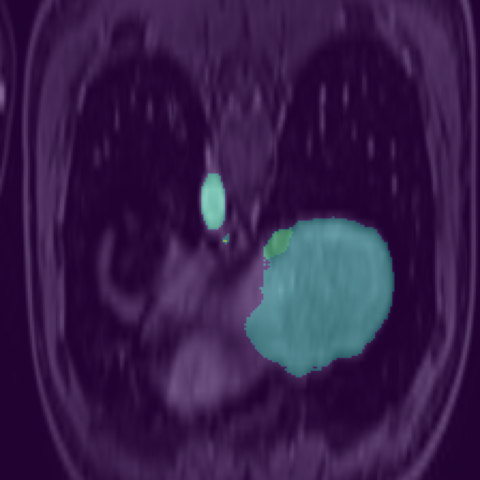} &
            \imgcell{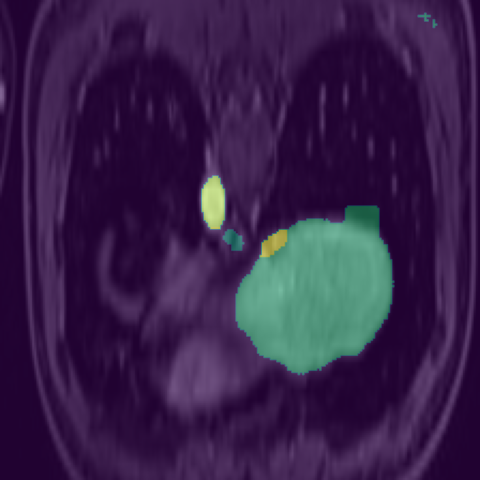} \\

            \imgcell{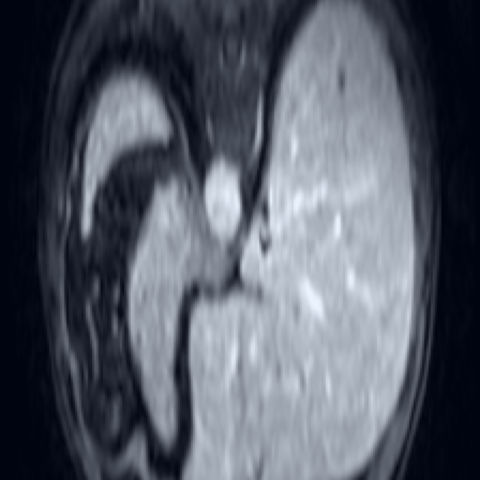} &
            \imgcell{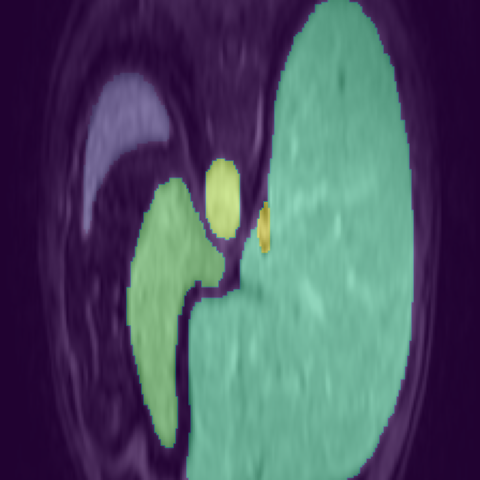} &
            \imgcell{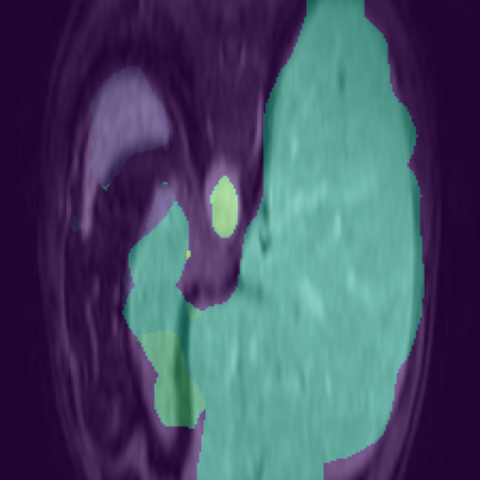} &
            \imgcell{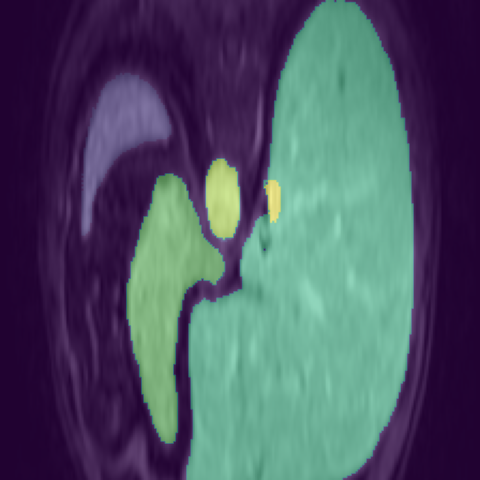} &
            \imgcell{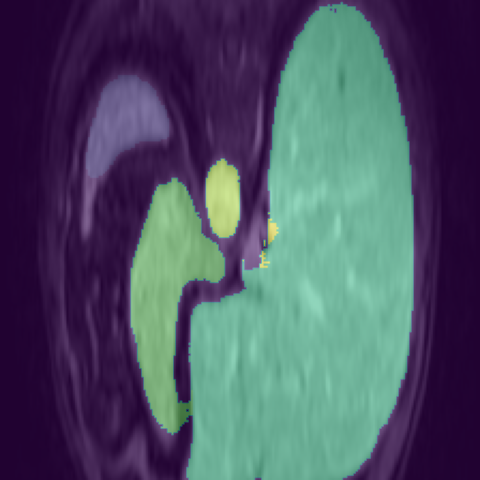} \\
            
        \end{tabular}
    \end{subfigure}

    \caption{Comparison of segmentation performance between a linear baseline, the ViT-UNet hybrid \cite{rmlp}, and our proposed ViTC-UNet utilizing DINOv2 \cite{dinov2} as encoder for semantic segmentation on the AMOS \cite{amos} dataset.}
\end{figure}

\begin{figure}[h!]
    \centering
    \newcommand{\imgcell}[1]{%
        \begin{minipage}[c][1.9cm][c]{\linewidth}
            \centering
            \includegraphics[width=\linewidth]{#1}
        \end{minipage}%
    }
    \newcommand{\tikzimg}[2]{%
        \begin{minipage}[c][1.9cm][c]{\linewidth}
            \centering
            \begin{tikzpicture}
                \node[inner sep=0pt] (image) {\includegraphics[width=\linewidth]{#1}};
                #2
            \end{tikzpicture}
        \end{minipage}%
    }
    \setlength{\tabcolsep}{7pt}
    \renewcommand{\arraystretch}{5.5} 

    \begin{subfigure}{\textwidth}
        \centering
        \begin{tabular}{>{\centering\arraybackslash}p{0.16\textwidth} >{\centering\arraybackslash}p{0.16\textwidth} >{\centering\arraybackslash}p{0.16\textwidth} >{\centering\arraybackslash}p{0.16\textwidth} >{\centering\arraybackslash}p{0.16\textwidth}}
            \textbf{CT Image} & \textbf{Ground Truth} & \textbf{Linear} & \textbf{ViT-UNet hybrid} & \textbf{ViTC-UNet} \\[-2.0em]
            
            \imgcell{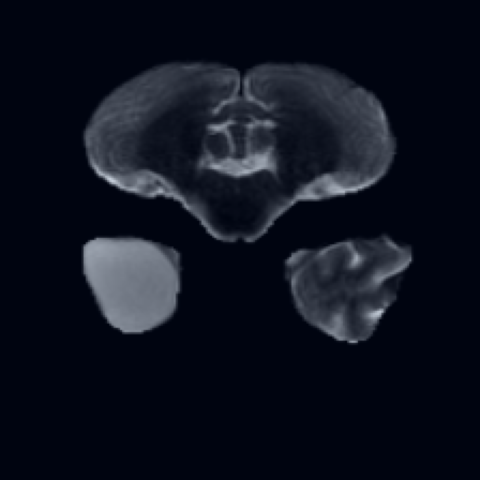} &
            \imgcell{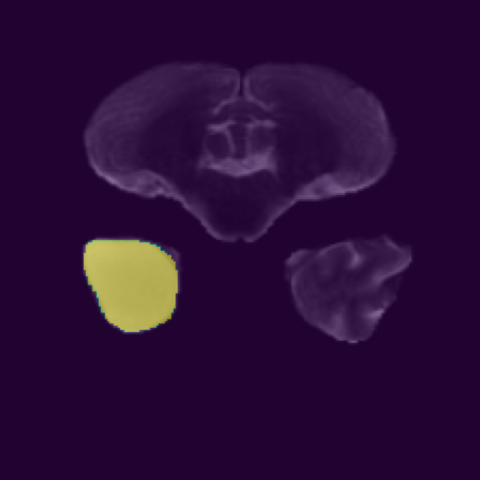} &
            \imgcell{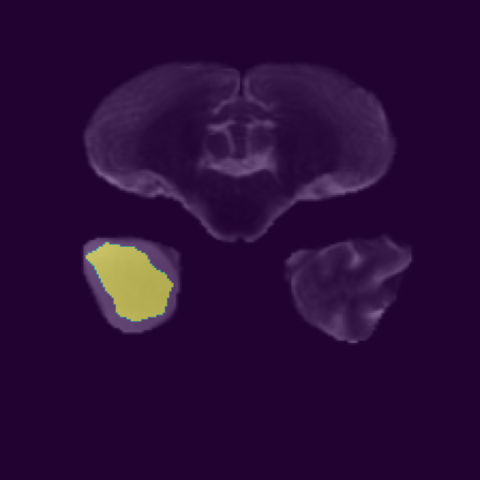} &
            \imgcell{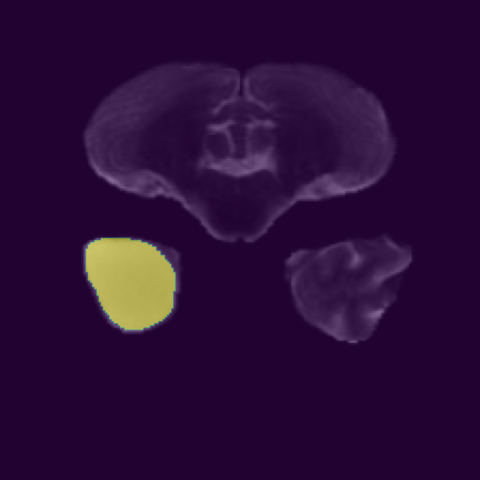} &
            \imgcell{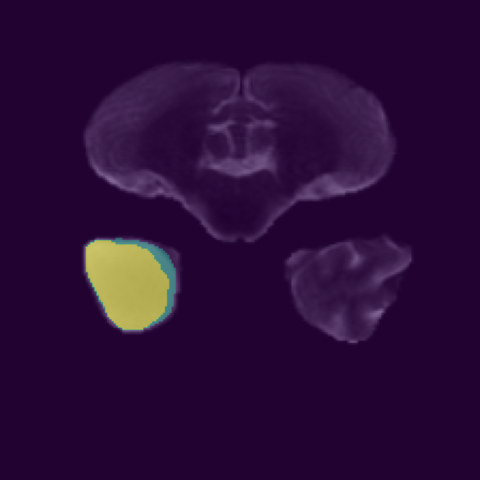} \\
            
            \imgcell{examples/overlays/mri/vitc_unet/brats/slice16_raw.png} &
            \imgcell{examples/overlays/mri/vitc_unet/brats/slice16_target.png} &
            \imgcell{examples/overlays/mri/linear/brats/slice16_pred.png} &
            \imgcell{examples/overlays/mri/hybrid/brats/slice16_pred.png} &
            \imgcell{examples/overlays/mri/vitc_unet/brats/slice16_pred.png} \\

            \imgcell{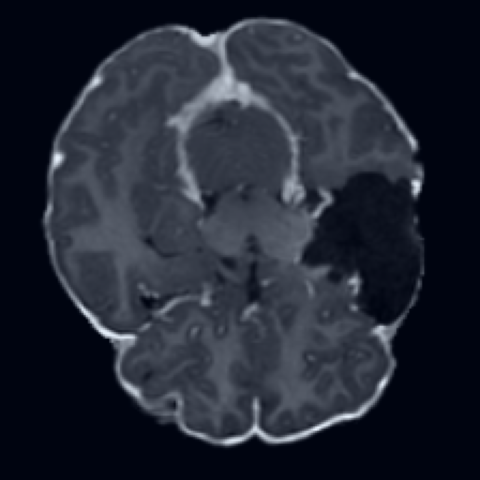} &
            \imgcell{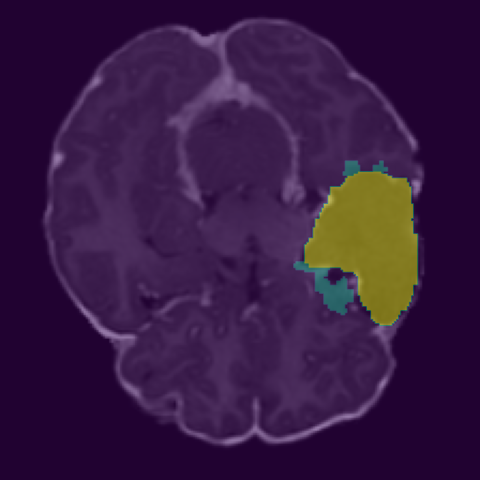} &
            \imgcell{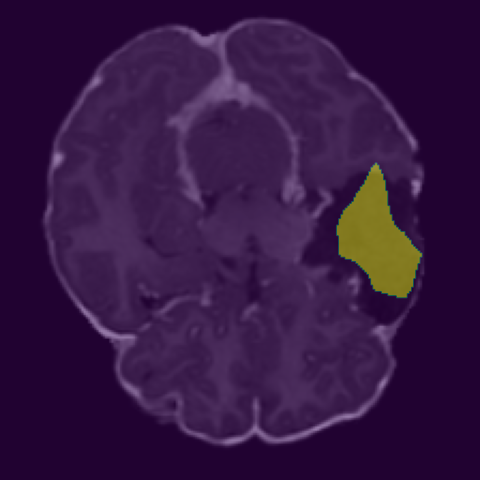} &
            \imgcell{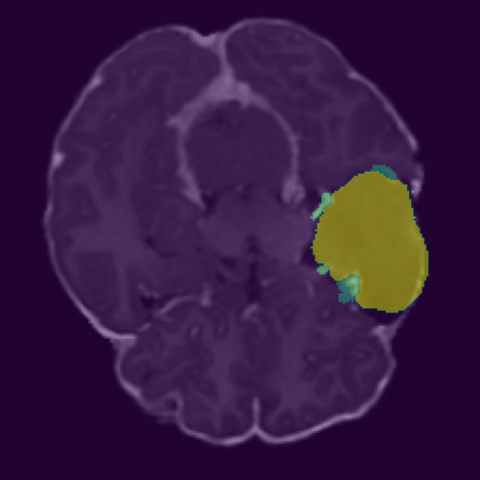} &
            \imgcell{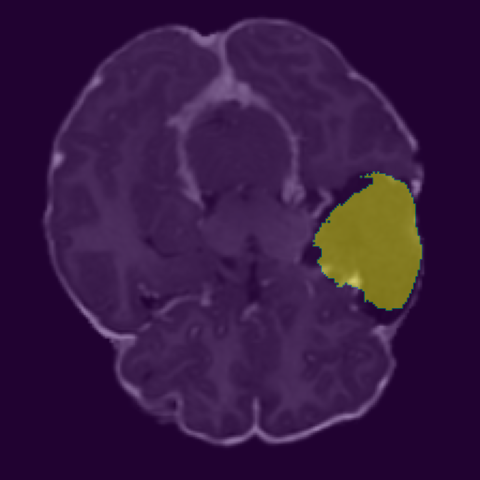} \\

            \imgcell{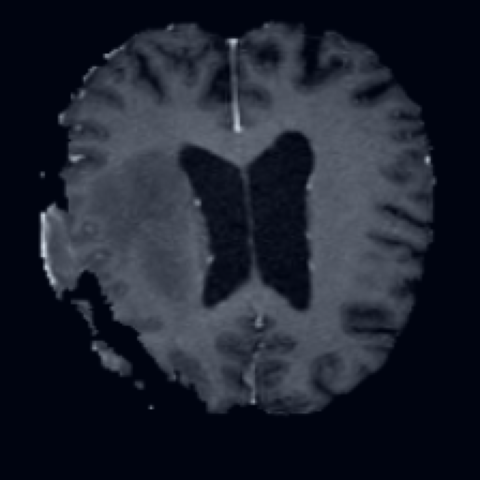} &
            \imgcell{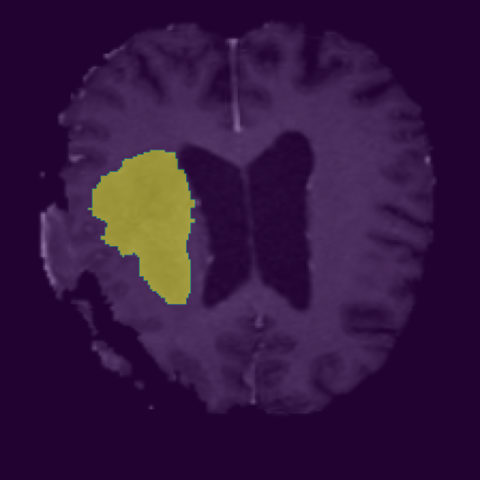} &
            \imgcell{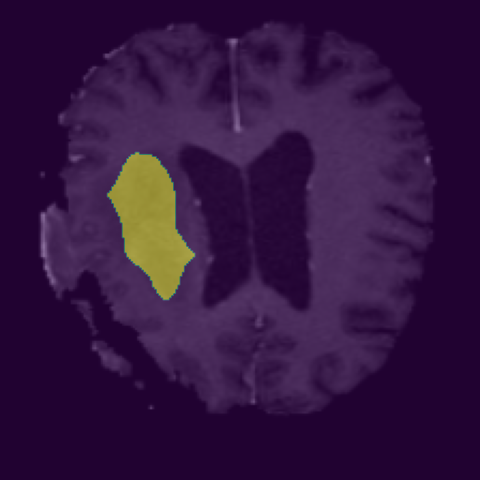} &
            \imgcell{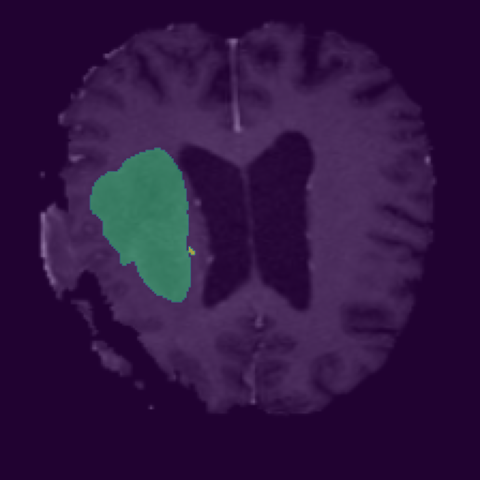} &
            \imgcell{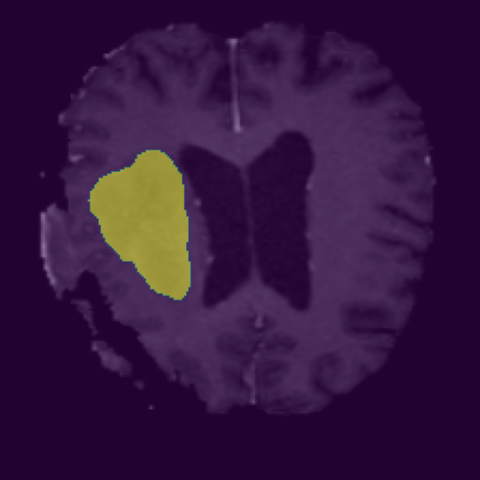} \\

            \imgcell{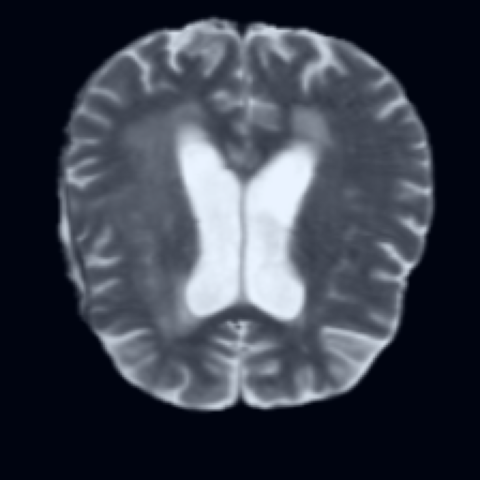} &
            \imgcell{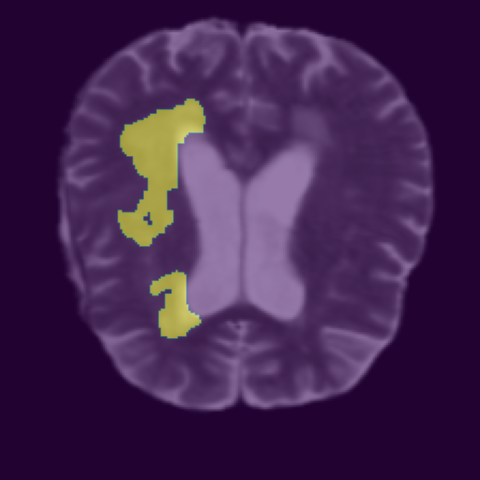} &
            \imgcell{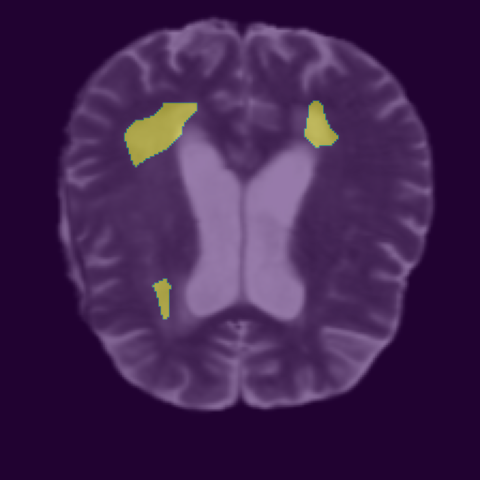} &
            \imgcell{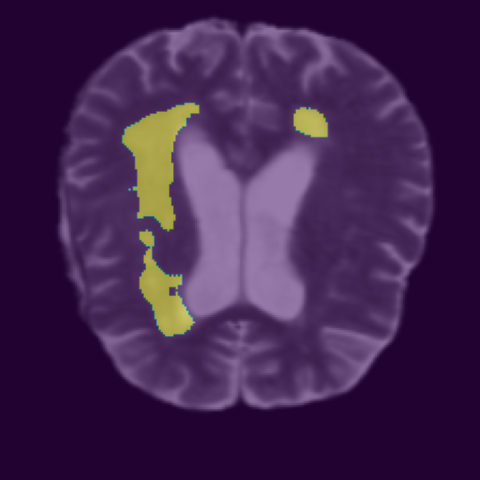} &
            \imgcell{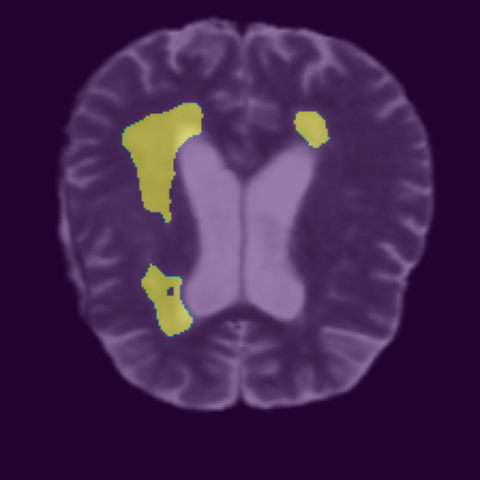} \\

            \imgcell{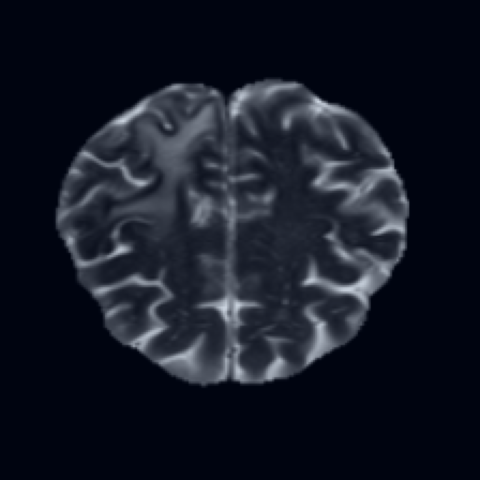} &
            \imgcell{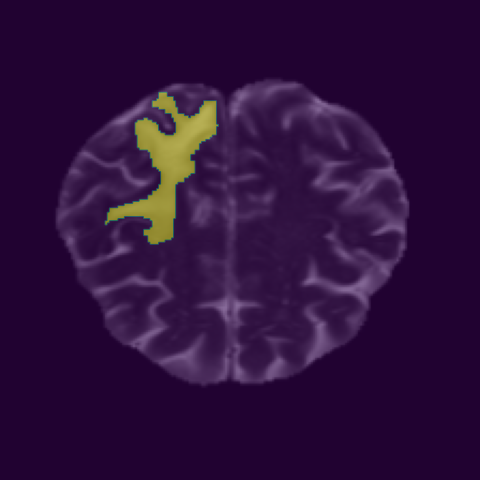} &
            \imgcell{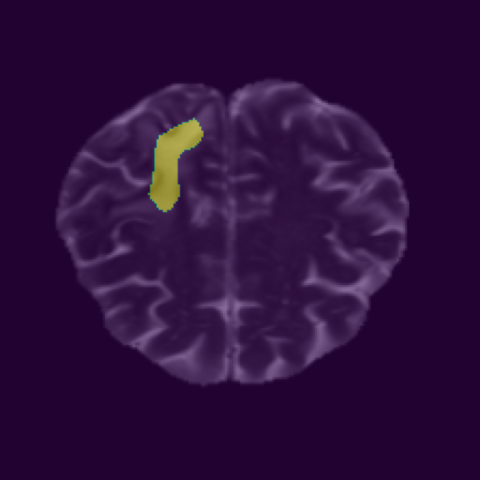} &
            \imgcell{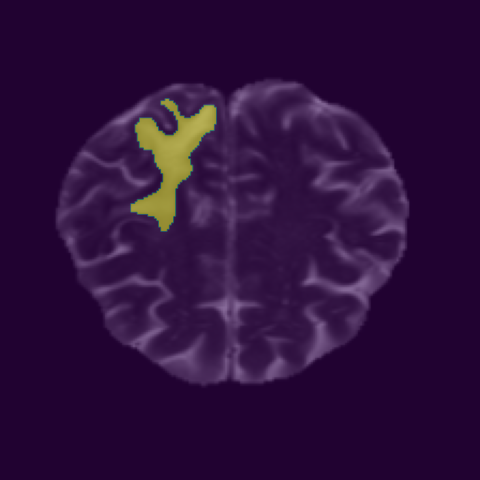} &
            \imgcell{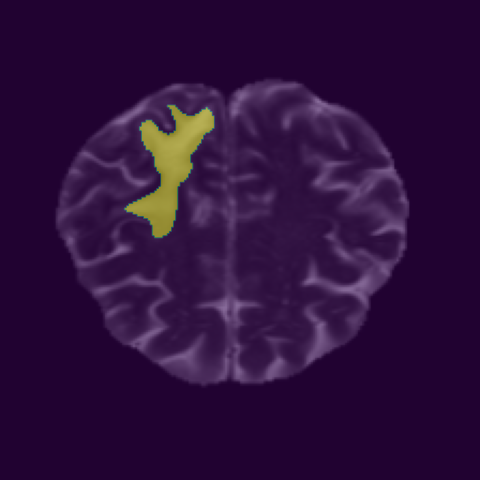} \\

            \imgcell{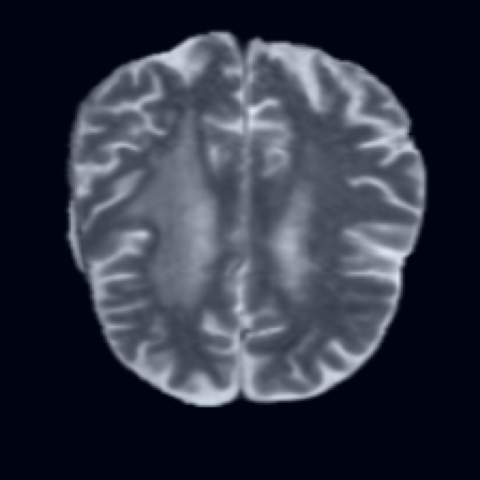} &
            \imgcell{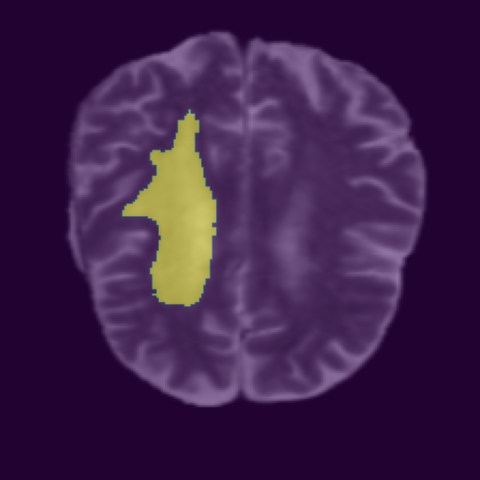} &
            \imgcell{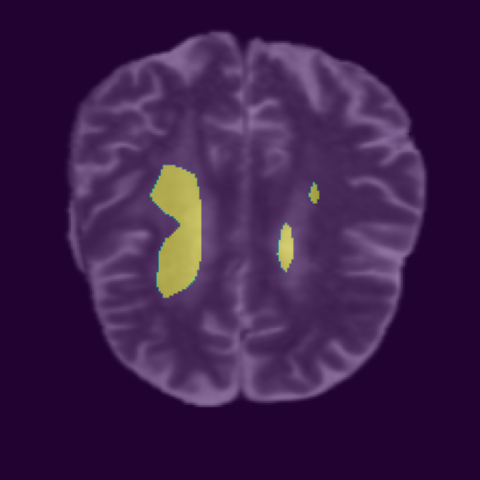} &
            \imgcell{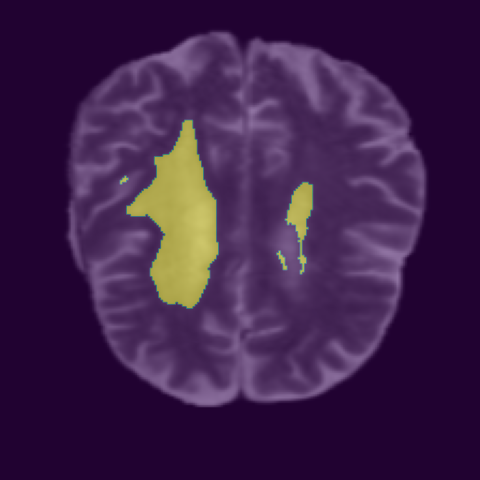} &
            \imgcell{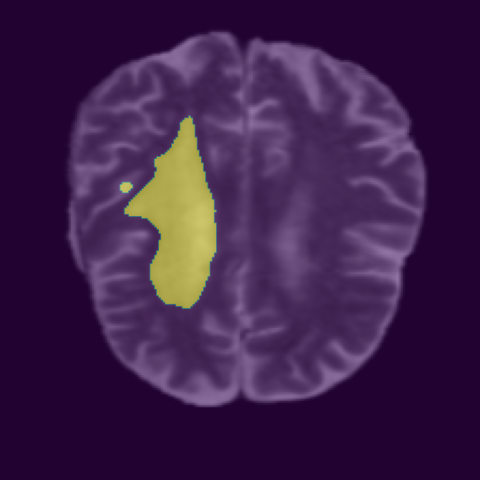} \\
            
        \end{tabular}
    \end{subfigure}

    \caption{Comparison of segmentation performance between a linear baseline, the ViT-UNet hybrid \cite{rmlp}, and our proposed ViTC-UNet utilizing DINOv2 \cite{dinov2} as encoder for semantic segmentation on the BraTS \cite{brats 1, brats 2,brats 3} dataset.}
\end{figure}

\begin{figure}[h!]
    \centering
    \newcommand{\imgcell}[1]{%
        \begin{minipage}[c][1.9cm][c]{\linewidth}
            \centering
            \includegraphics[width=\linewidth]{#1}
        \end{minipage}%
    }
    \newcommand{\tikzimg}[2]{%
        \begin{minipage}[c][1.9cm][c]{\linewidth}
            \centering
            \begin{tikzpicture}
                \node[inner sep=0pt] (image) {\includegraphics[width=\linewidth]{#1}};
                #2
            \end{tikzpicture}
        \end{minipage}%
    }
    \setlength{\tabcolsep}{7pt}
    \renewcommand{\arraystretch}{5.5} 

    \begin{subfigure}{\textwidth}
        \centering
        \begin{tabular}{>{\centering\arraybackslash}p{0.16\textwidth} >{\centering\arraybackslash}p{0.16\textwidth} >{\centering\arraybackslash}p{0.16\textwidth} >{\centering\arraybackslash}p{0.16\textwidth} >{\centering\arraybackslash}p{0.16\textwidth}}
            \textbf{CT Image} & \textbf{Ground Truth} & \textbf{Linear} & \textbf{ViT-UNet hybrid} & \textbf{ViTC-UNet} \\[-2.0em]
            
            \imgcell{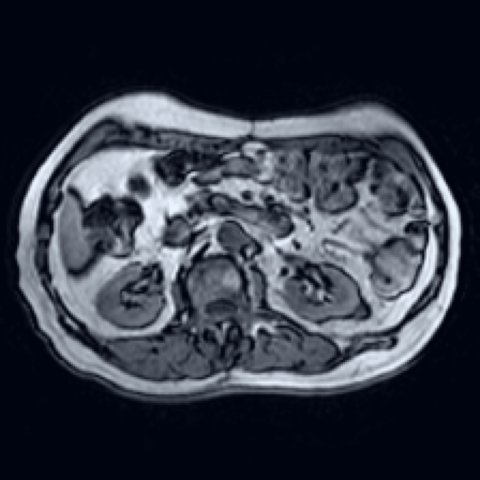} &
            \imgcell{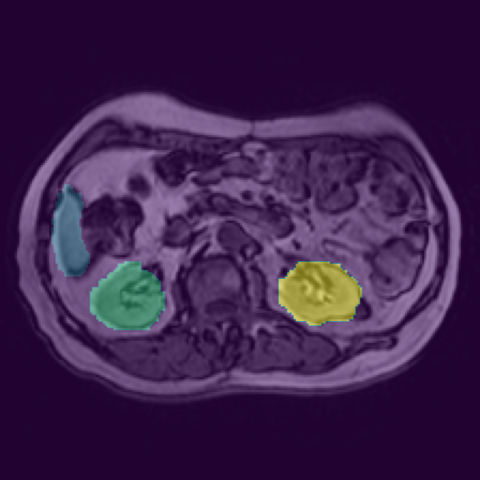} &
            \imgcell{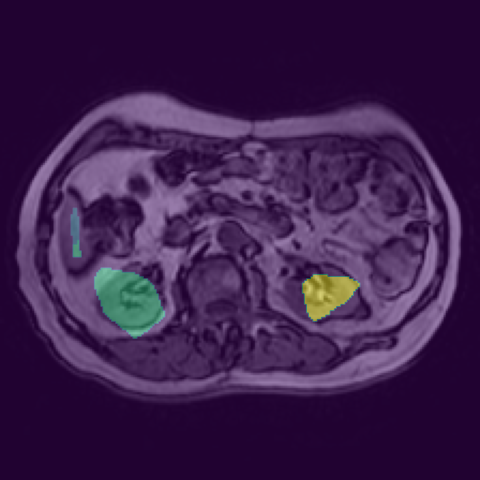} &
            \imgcell{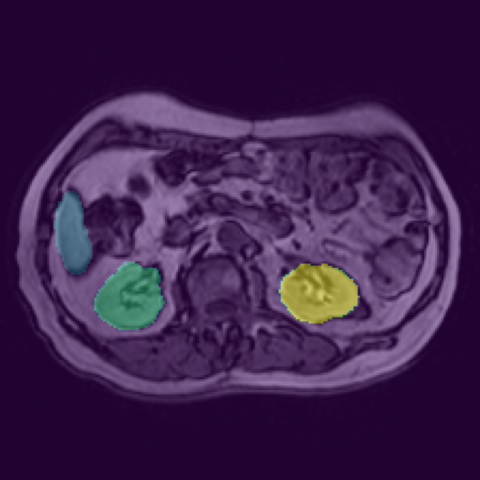} &
            \imgcell{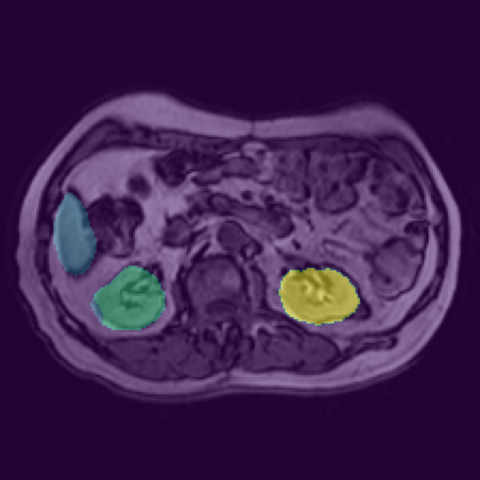} \\
            
            \imgcell{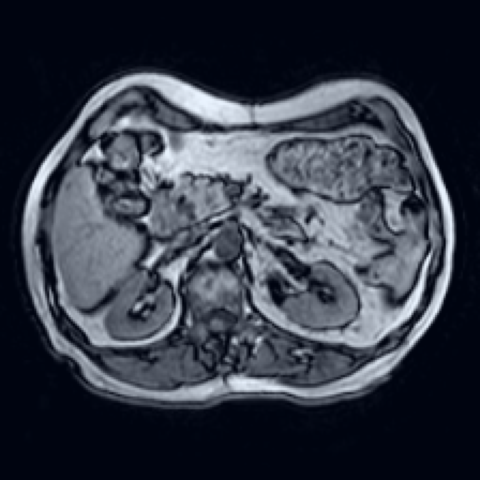} &
            \imgcell{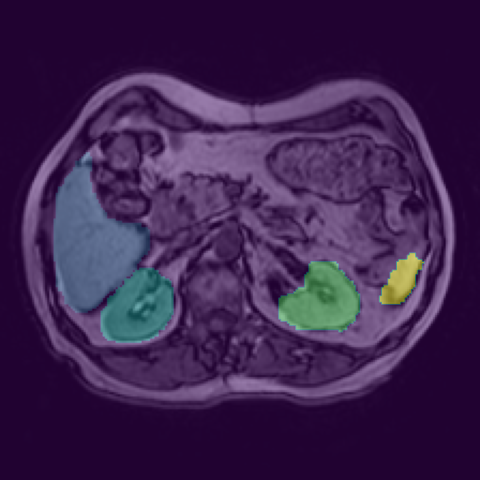} &
            \imgcell{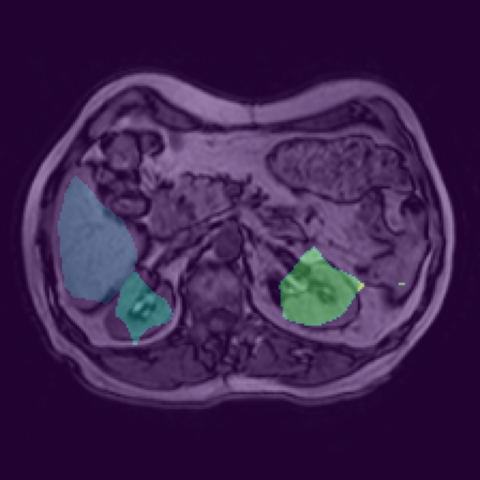} &
            \imgcell{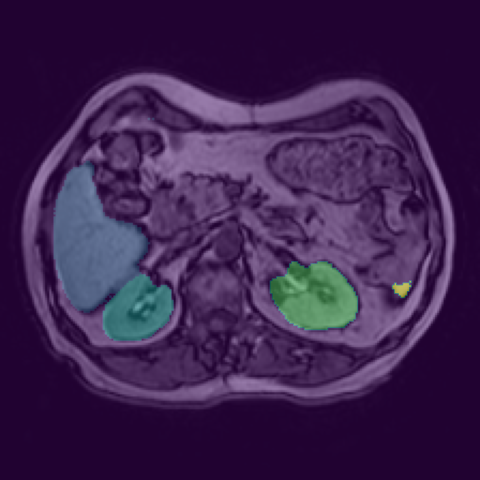} &
            \imgcell{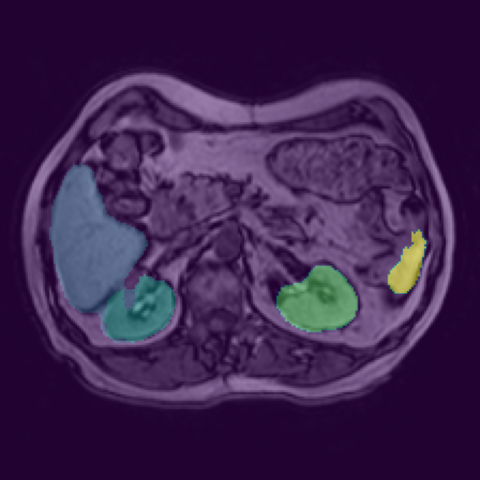} \\

            \imgcell{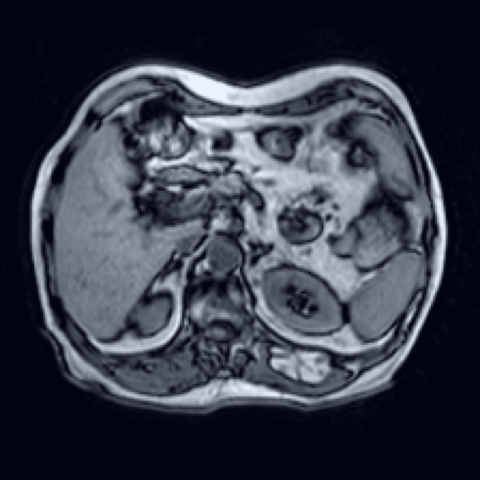} &
            \imgcell{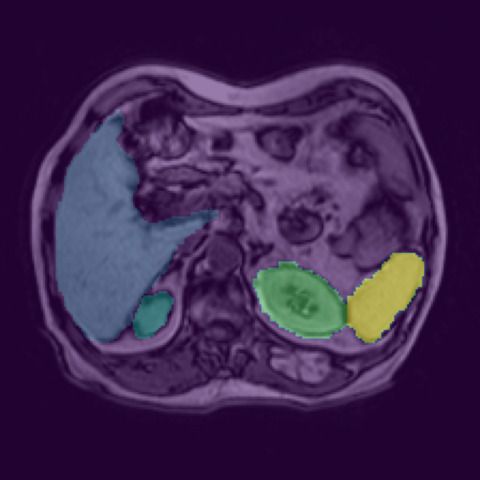} &
            \imgcell{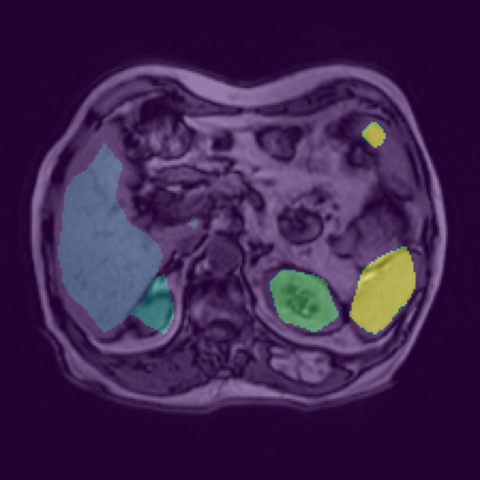} &
            \imgcell{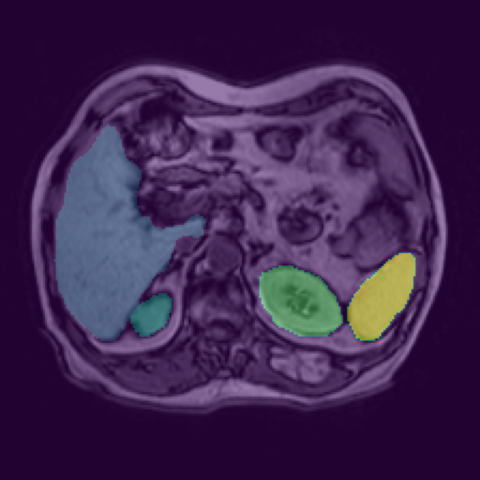} &
            \imgcell{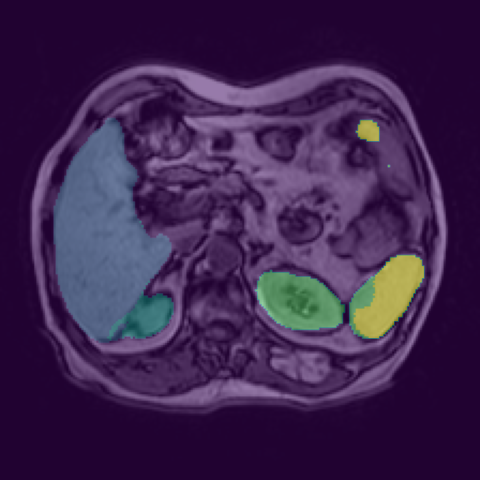} \\

            \imgcell{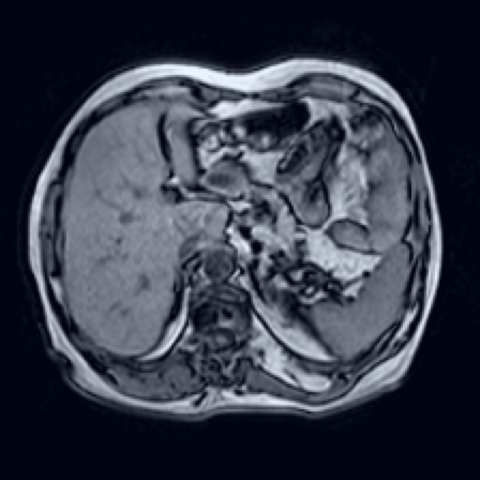} &
            \imgcell{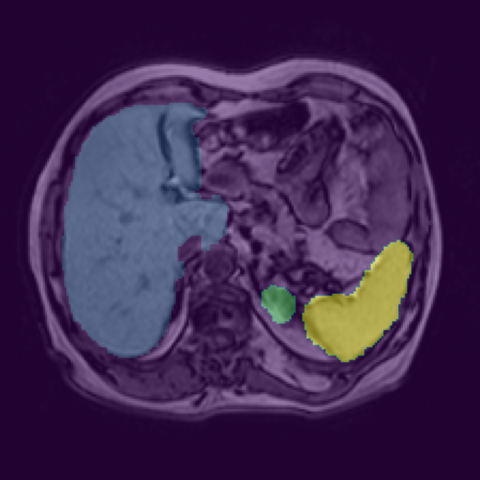} &
            \imgcell{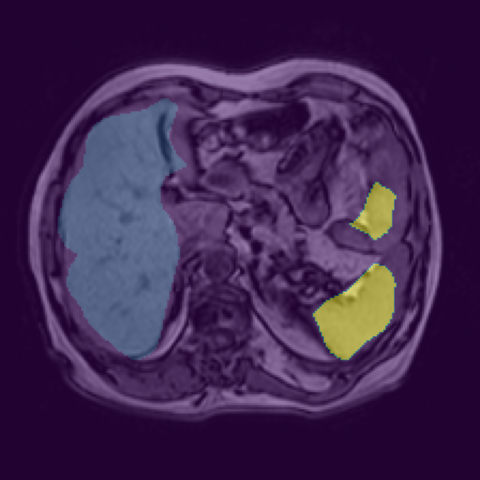} &
            \imgcell{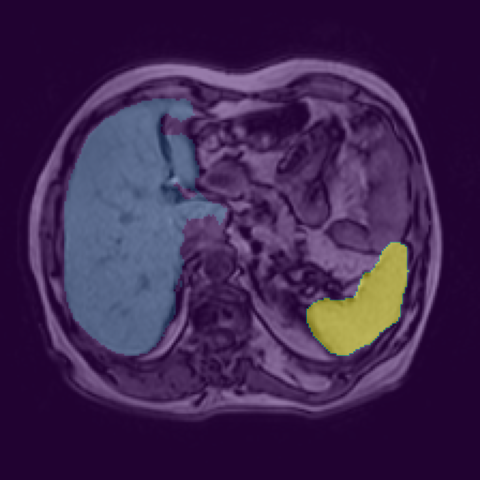} &
            \imgcell{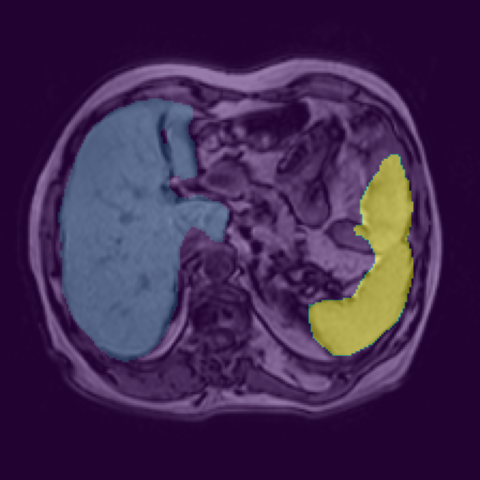} \\

            \imgcell{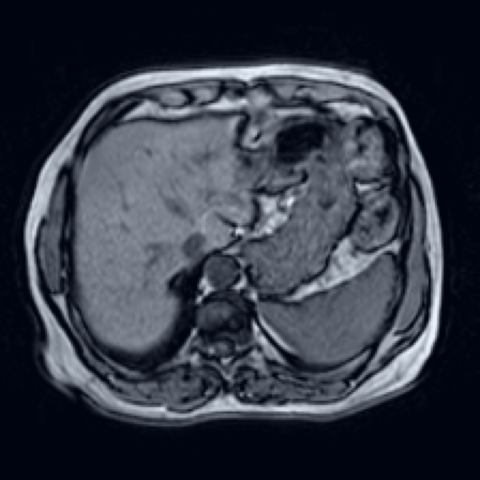} &
            \imgcell{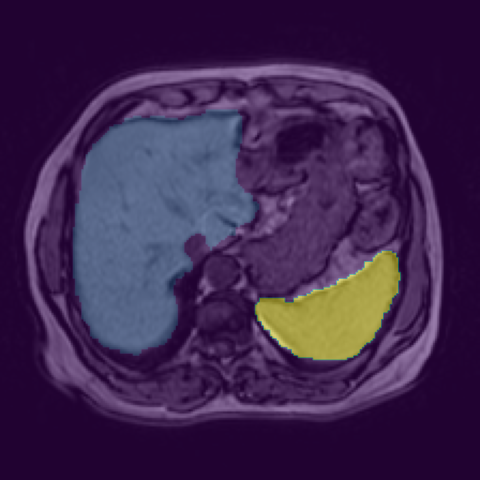} &
            \imgcell{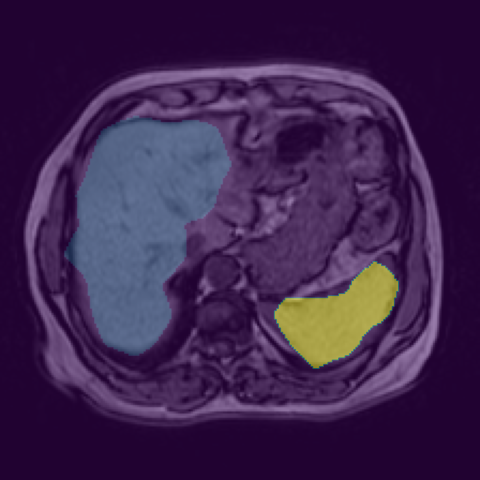} &
            \imgcell{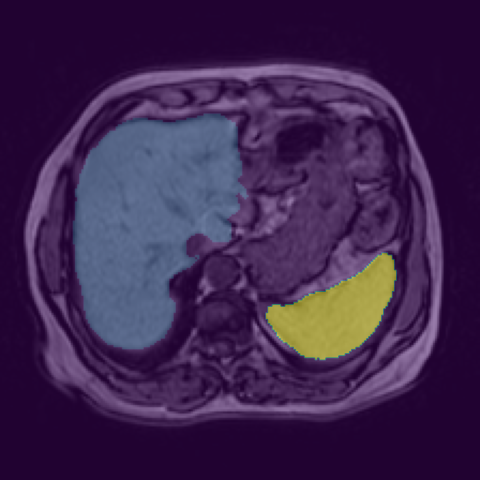} &
            \imgcell{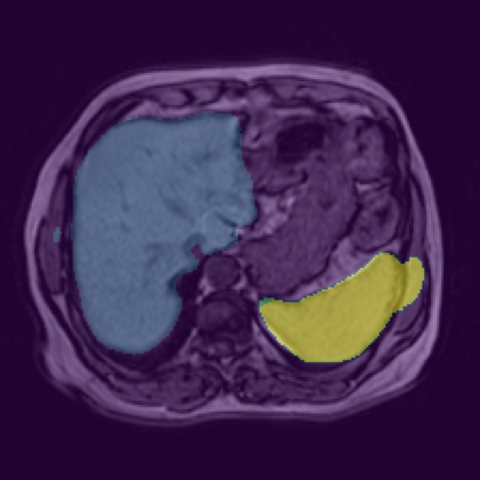} \\

        \end{tabular}
    \end{subfigure}

    \caption{Comparison of segmentation performance between a linear baseline, the ViT-UNet hybrid \cite{rmlp}, and our proposed ViTC-UNet utilizing DINOv2 \cite{dinov2} as encoder for semantic segmentation on the CHAOS \cite{chaos} dataset.}
\end{figure}

\begin{figure}[h!]
    \centering
    \newcommand{\imgcell}[1]{%
        \begin{minipage}[c][1.9cm][c]{\linewidth}
            \centering
            \includegraphics[width=\linewidth]{#1}
        \end{minipage}%
    }
    \newcommand{\tikzimg}[2]{%
        \begin{minipage}[c][1.9cm][c]{\linewidth}
            \centering
            \begin{tikzpicture}
                \node[inner sep=0pt] (image) {\includegraphics[width=\linewidth]{#1}};
                #2
            \end{tikzpicture}
        \end{minipage}%
    }
    \setlength{\tabcolsep}{7pt}
    \renewcommand{\arraystretch}{5.5} 

    \begin{subfigure}{\textwidth}
        \centering
        \begin{tabular}{>{\centering\arraybackslash}p{0.16\textwidth} >{\centering\arraybackslash}p{0.16\textwidth} >{\centering\arraybackslash}p{0.16\textwidth} >{\centering\arraybackslash}p{0.16\textwidth} >{\centering\arraybackslash}p{0.16\textwidth}}
            \textbf{CT Image} & \textbf{Ground Truth} & \textbf{Linear} & \textbf{ViT-UNet hybrid} & \textbf{ViTC-UNet} \\[-2.0em]
            
            \imgcell{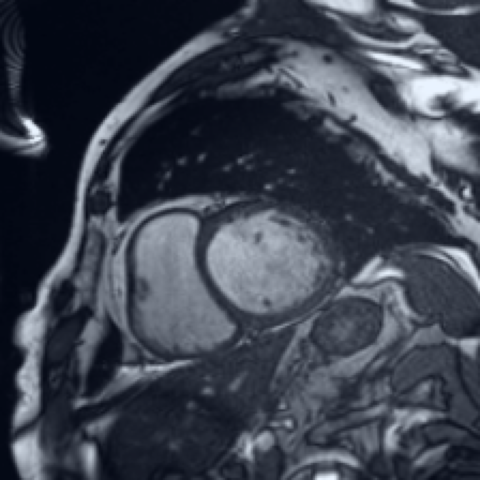} &
            \imgcell{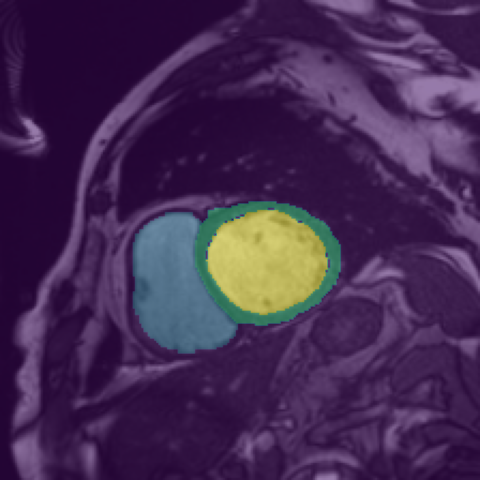} &
            \imgcell{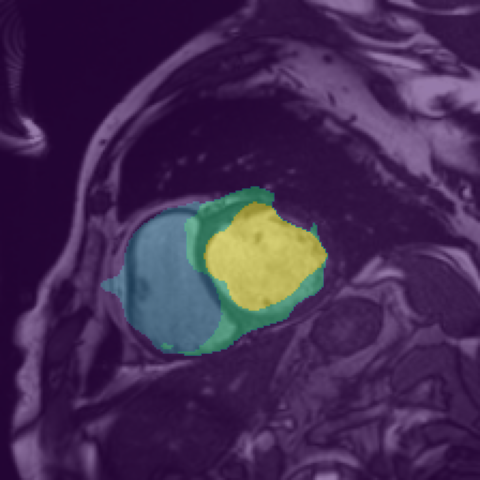} &
            \imgcell{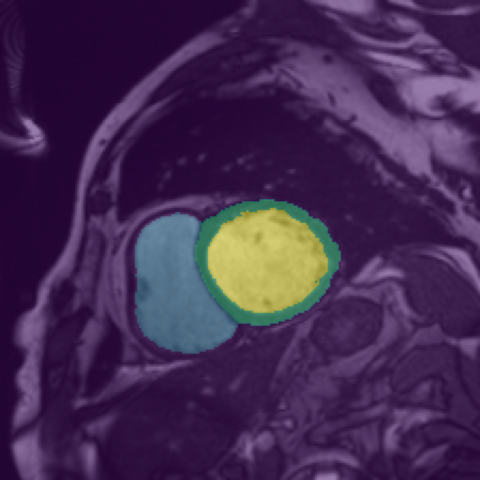} &
            \imgcell{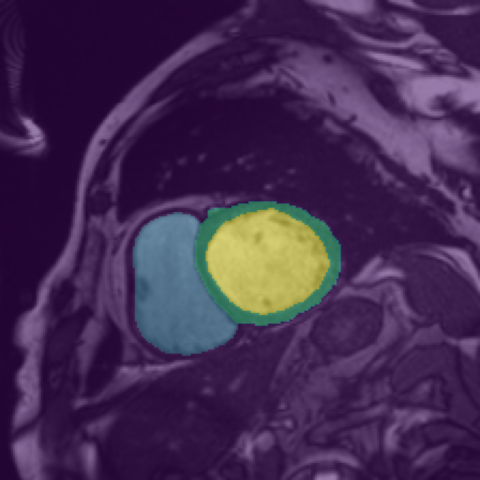} \\
            
            \imgcell{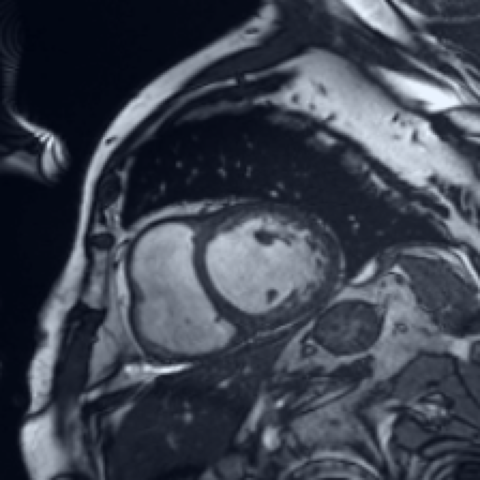} &
            \imgcell{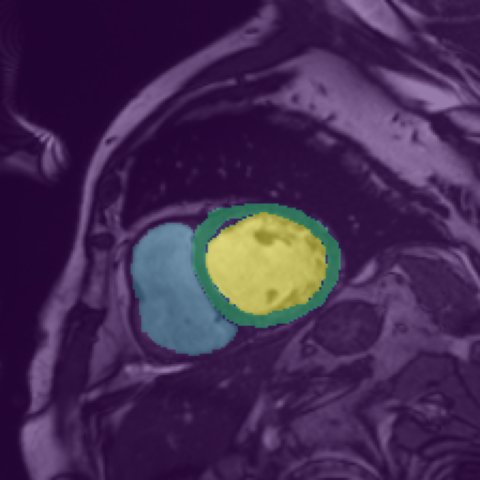} &
            \imgcell{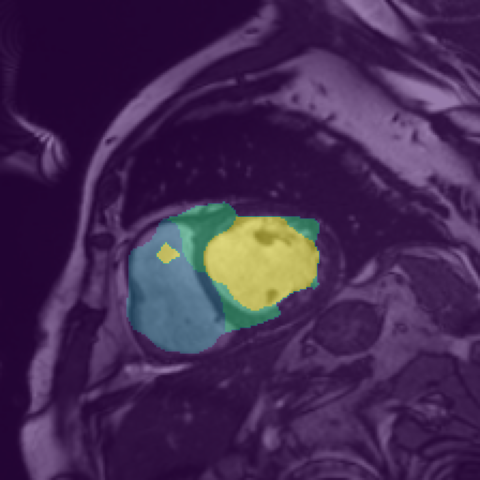} &
            \imgcell{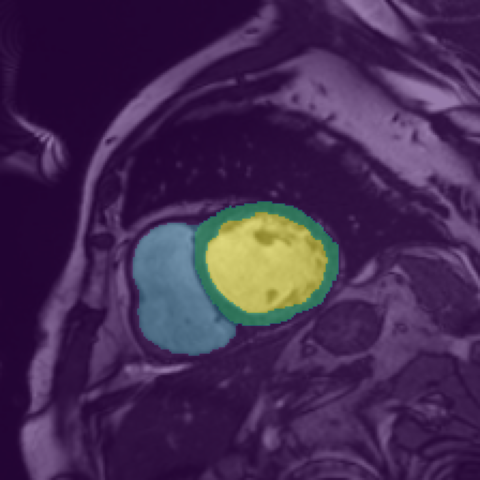} &
            \imgcell{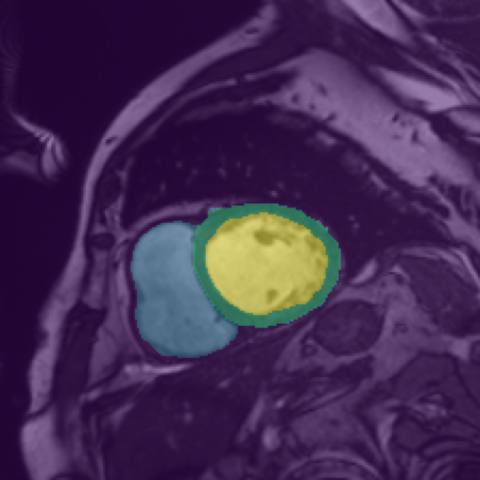} \\
            
            \imgcell{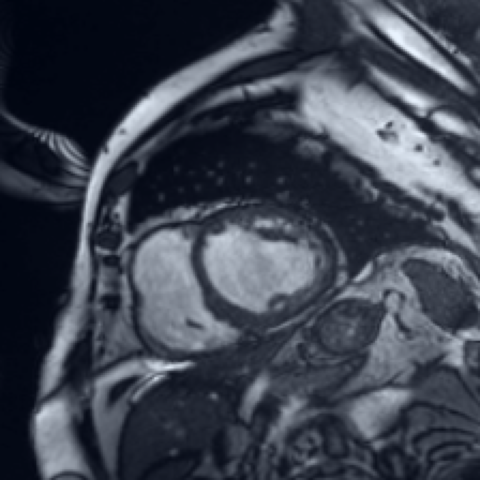} &
            \imgcell{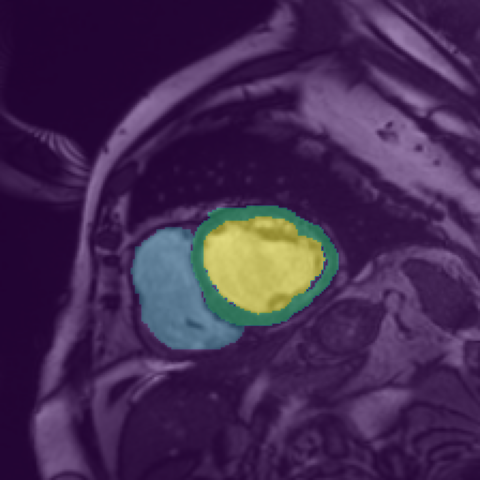} &
            \imgcell{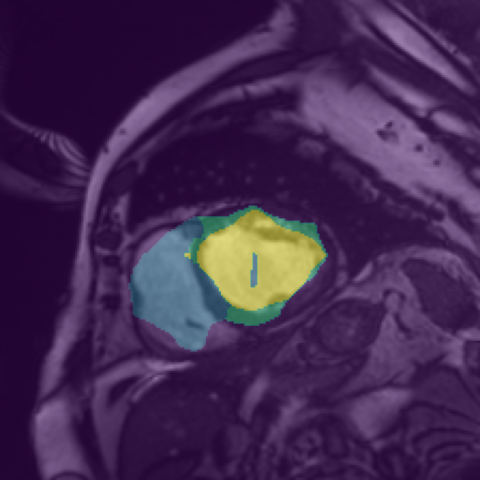} &
            \imgcell{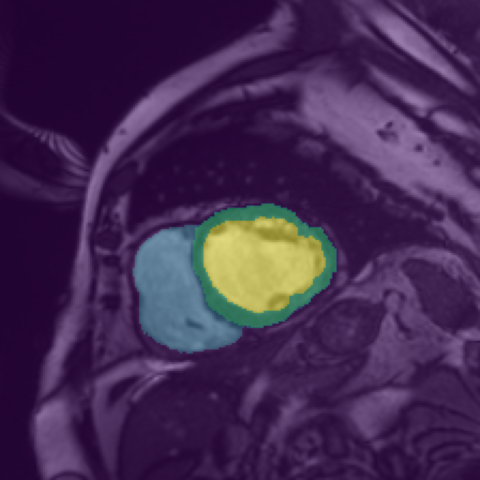} &
            \imgcell{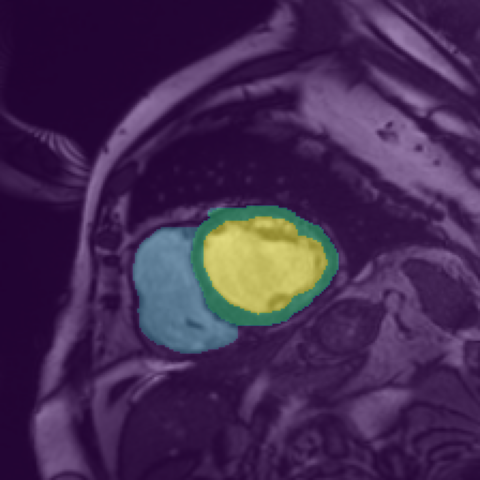} \\

            \imgcell{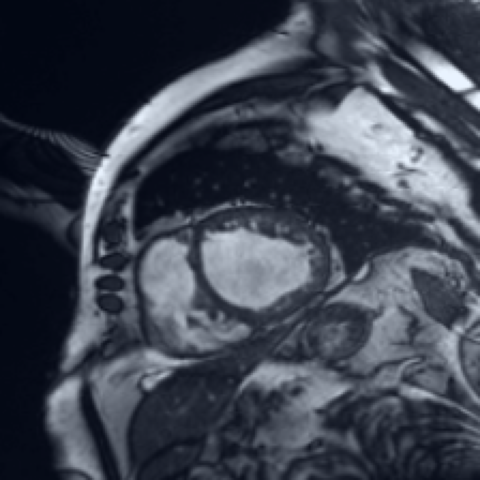} &
            \imgcell{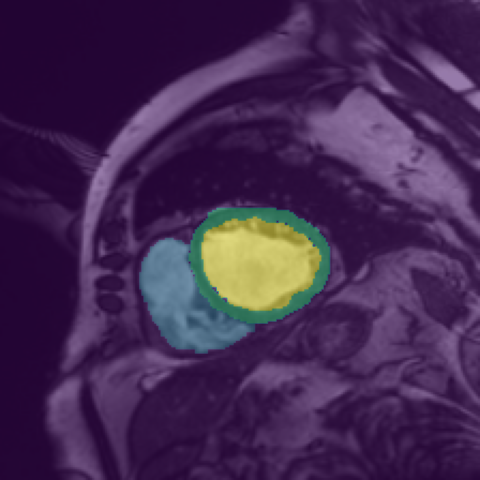} &
            \imgcell{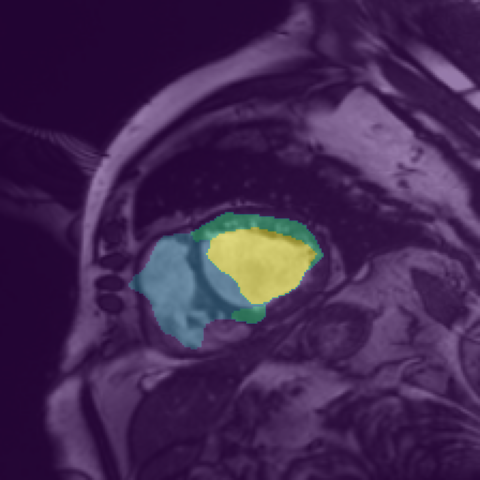} &
            \imgcell{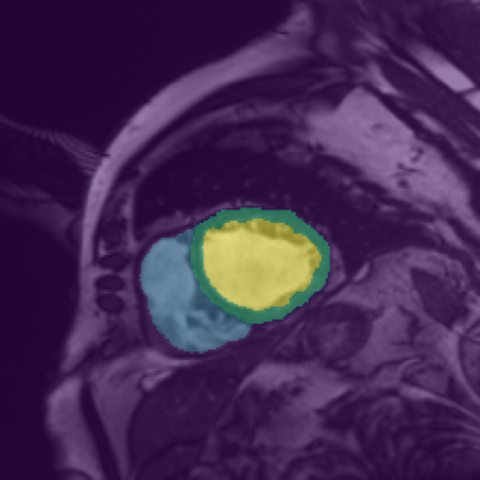} &
            \imgcell{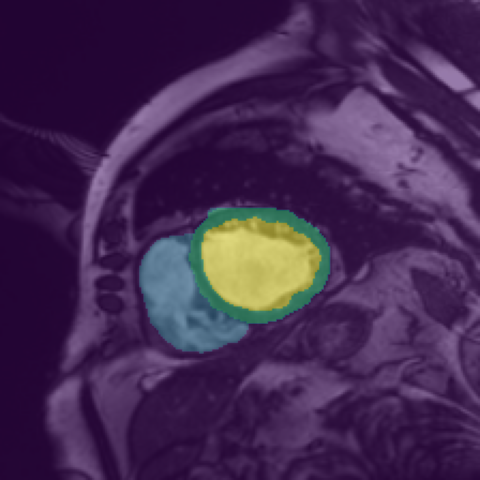} \\

            \imgcell{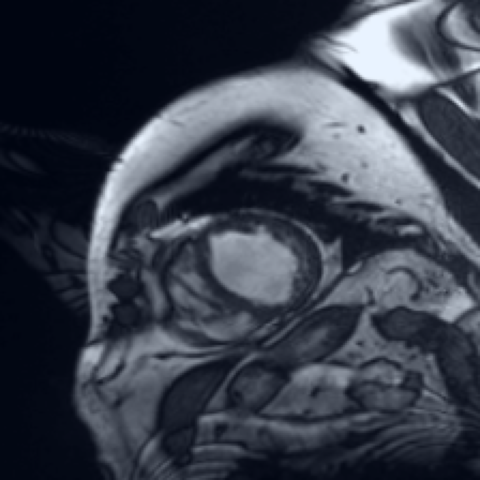} &
            \imgcell{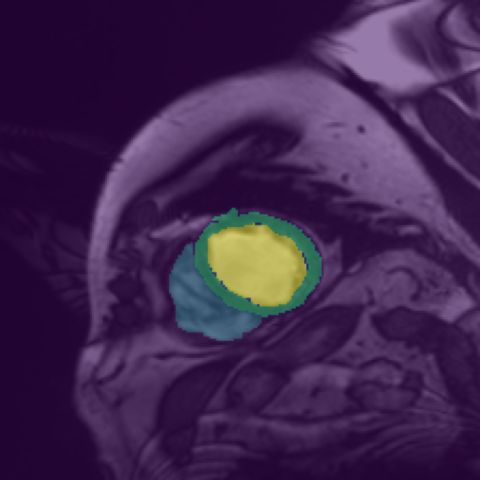} &
            \imgcell{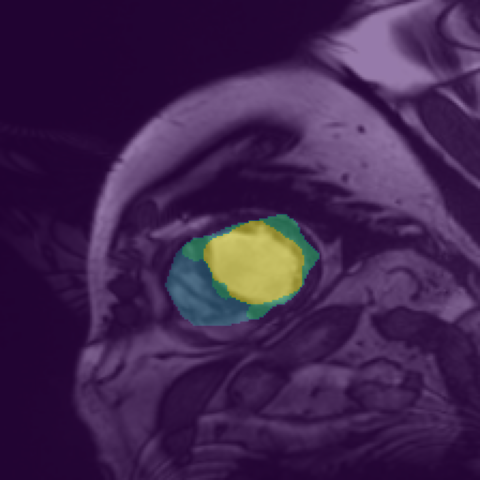} &
            \imgcell{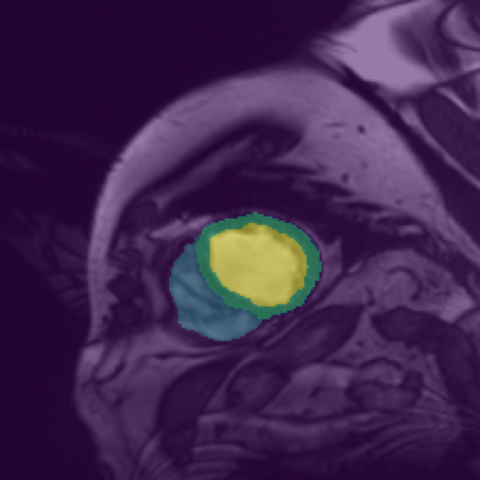} &
            \imgcell{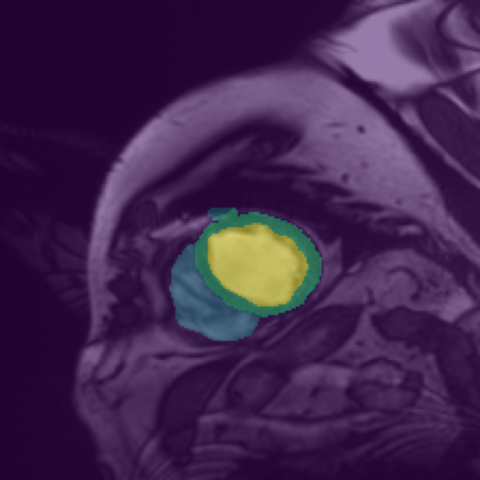} \\

        \end{tabular}
    \end{subfigure}

    \caption{Comparison of segmentation performance between a linear baseline, the ViT-UNet hybrid \cite{rmlp}, and our proposed ViTC-UNet utilizing DINOv2 \cite{dinov2} as encoder for semantic segmentation on the Heart\_ACDC \cite{nnunet heart acdc} dataset.}
\end{figure}

\begin{figure}[h!]
    \centering
    \newcommand{\imgcell}[1]{%
        \begin{minipage}[c][1.9cm][c]{\linewidth}
            \centering
            \includegraphics[width=\linewidth]{#1}
        \end{minipage}%
    }
    \newcommand{\tikzimg}[2]{%
        \begin{minipage}[c][1.9cm][c]{\linewidth}
            \centering
            \begin{tikzpicture}
                \node[inner sep=0pt] (image) {\includegraphics[width=\linewidth]{#1}};
                #2
            \end{tikzpicture}
        \end{minipage}%
    }
    \setlength{\tabcolsep}{7pt}
    \renewcommand{\arraystretch}{5.5} 

    \begin{subfigure}{\textwidth}
        \centering
        \begin{tabular}{>{\centering\arraybackslash}p{0.16\textwidth} >{\centering\arraybackslash}p{0.16\textwidth} >{\centering\arraybackslash}p{0.16\textwidth} >{\centering\arraybackslash}p{0.16\textwidth} >{\centering\arraybackslash}p{0.16\textwidth}}
            \textbf{CT Image} & \textbf{Ground Truth} & \textbf{Linear} & \textbf{ViT-UNet hybrid} & \textbf{ViTC-UNet} \\[-2.0em]
            
            \imgcell{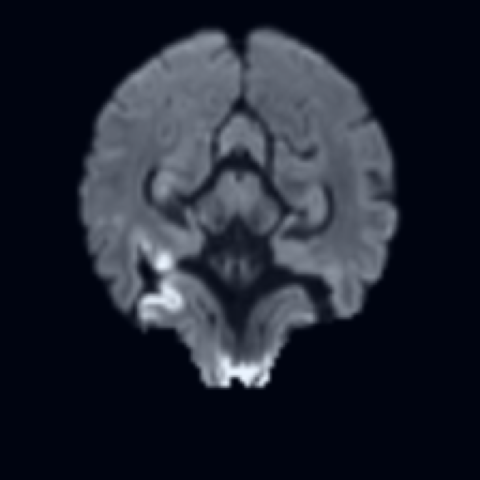} &
            \imgcell{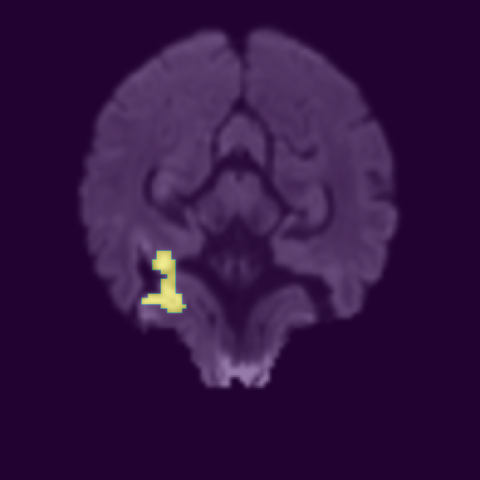} &
            \imgcell{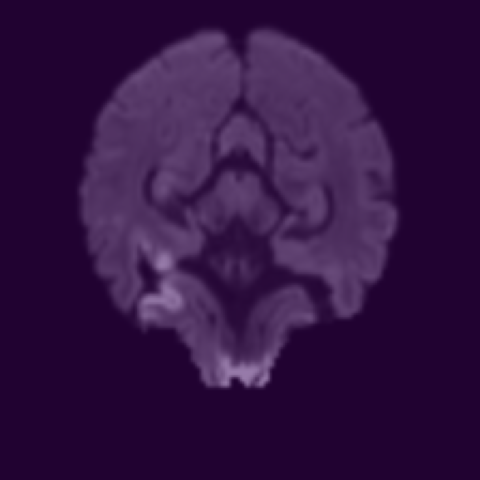} &
            \imgcell{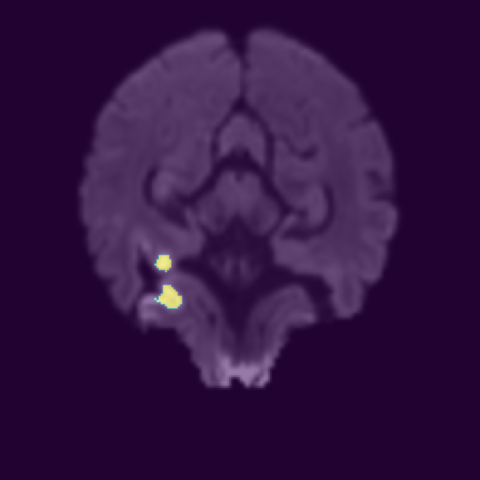} &
            \imgcell{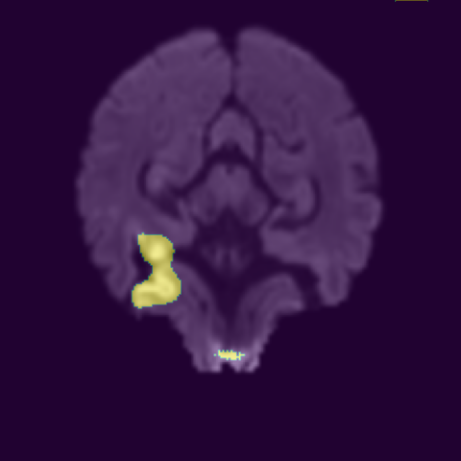} \\
            
            \imgcell{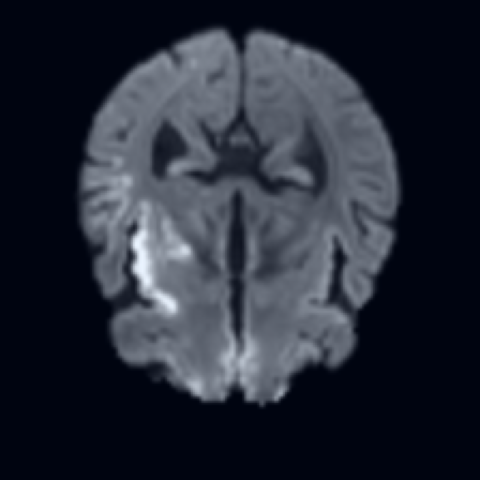} &
            \imgcell{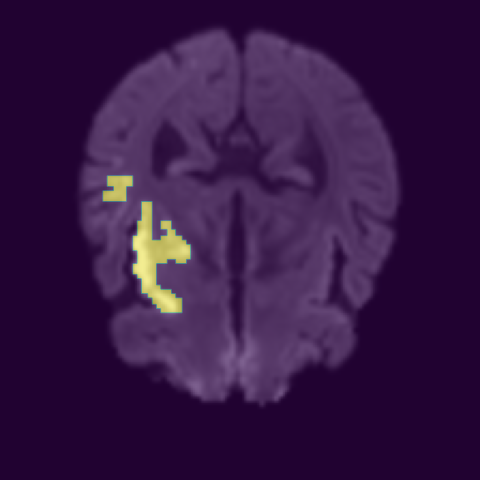} &
            \imgcell{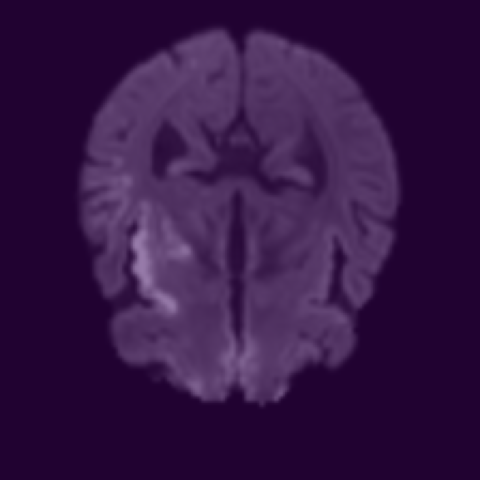} &
            \imgcell{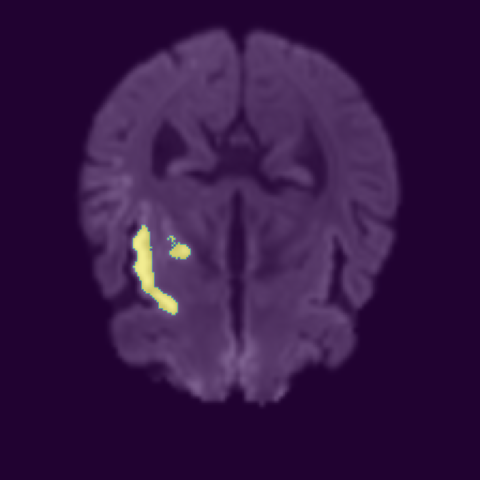} &
            \imgcell{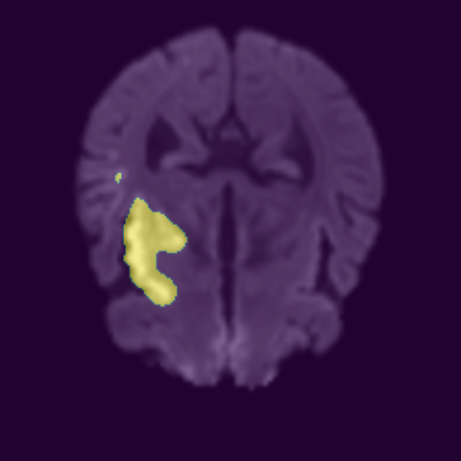} \\

            \imgcell{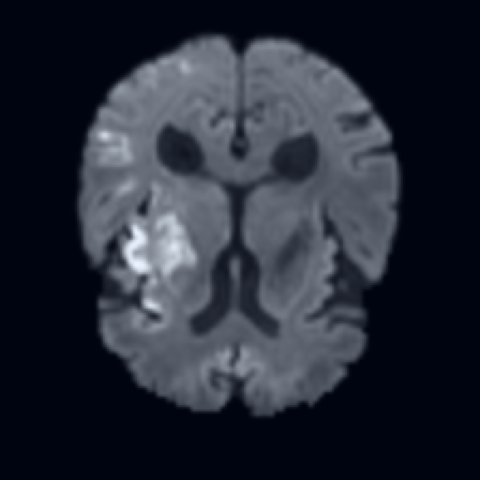} &
            \imgcell{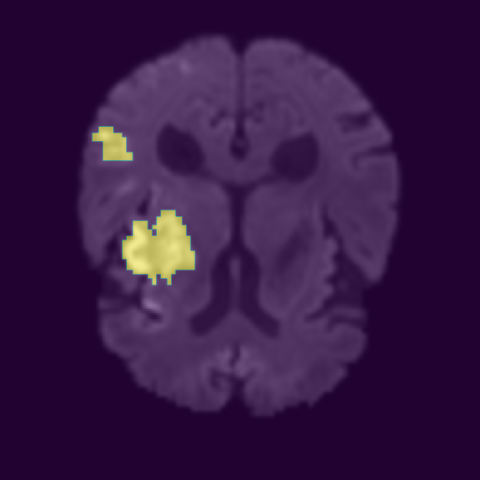} &
            \imgcell{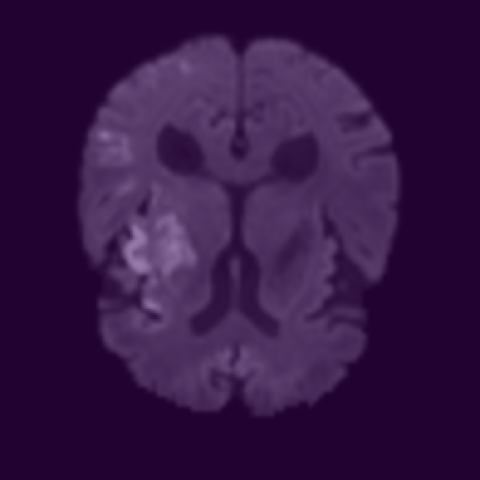} &
            \imgcell{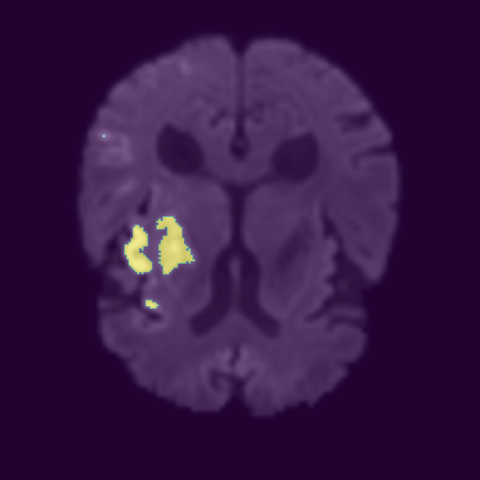} &
            \imgcell{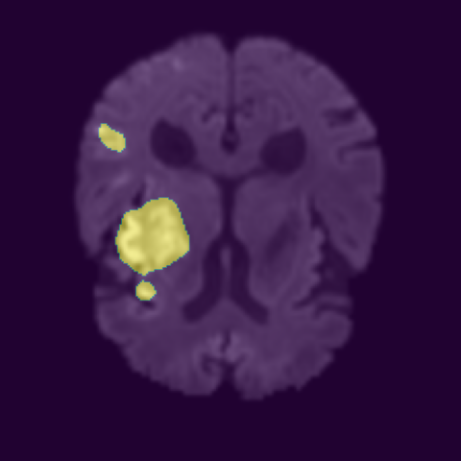} \\

            \imgcell{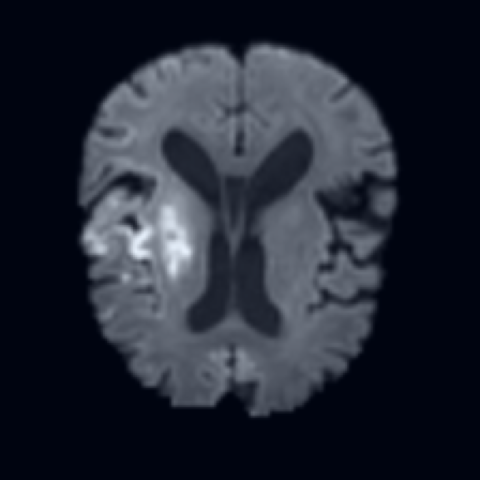} &
            \imgcell{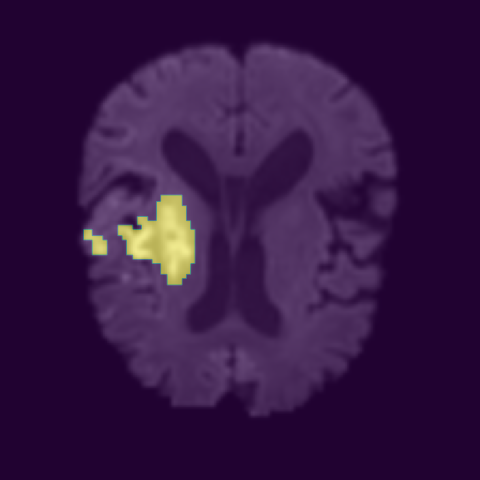} &
            \imgcell{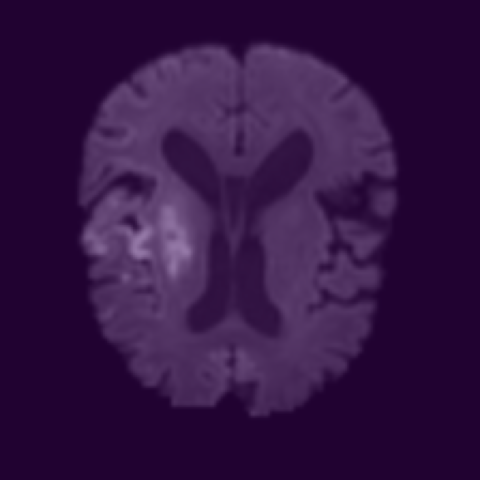} &
            \imgcell{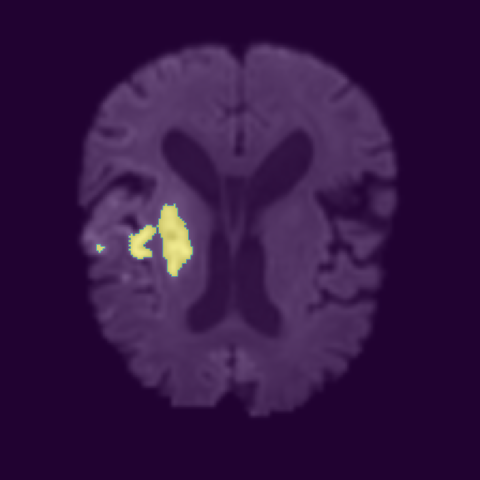} &
            \imgcell{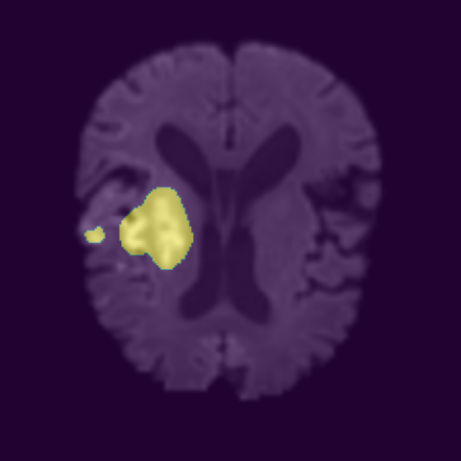} \\

            \imgcell{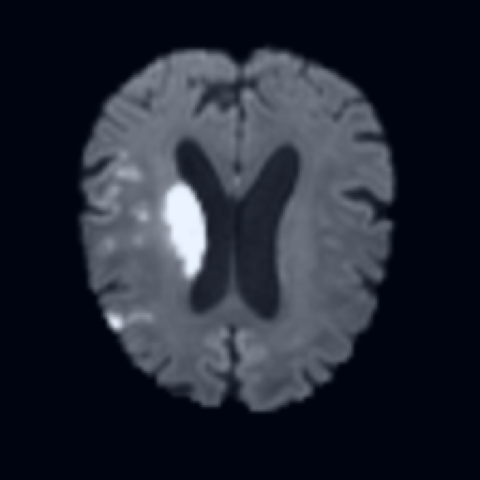} &
            \imgcell{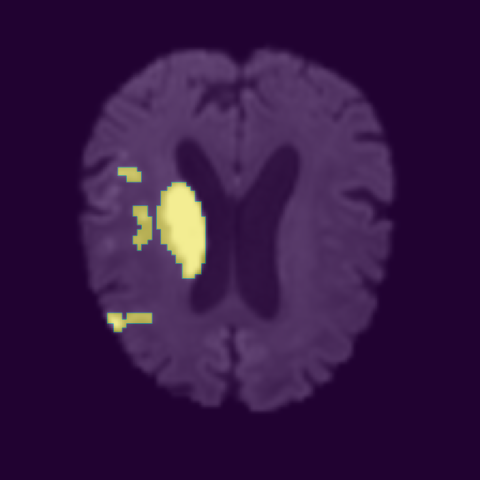} &
            \imgcell{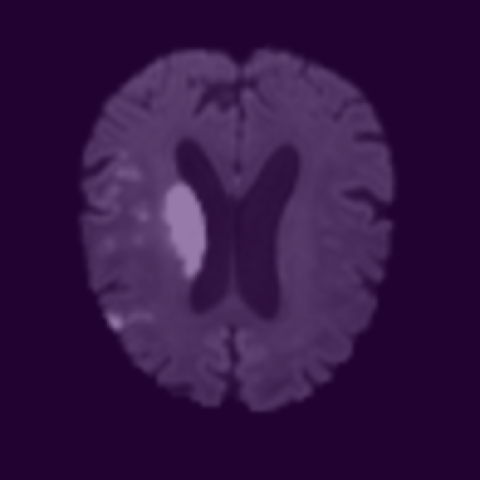} &
            \imgcell{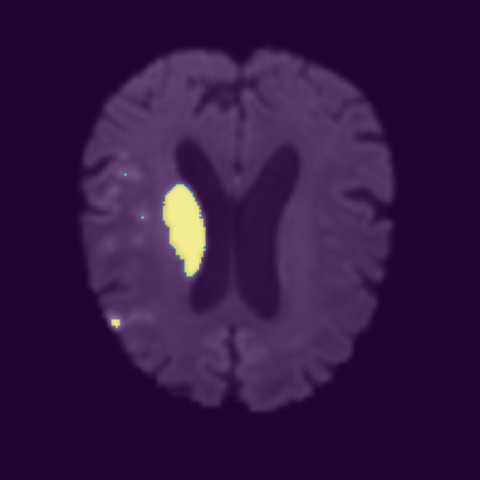} &
            \imgcell{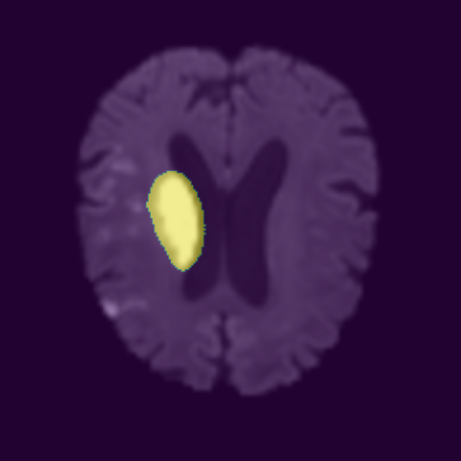} \\
            
        \end{tabular}
    \end{subfigure}

    \caption{Comparison of segmentation performance between a linear baseline, the ViT-UNet hybrid \cite{rmlp}, and our proposed ViTC-UNet utilizing DINOv2 \cite{dinov2} as encoder for semantic segmentation on the ISLES \cite{isles} dataset.}
\end{figure}

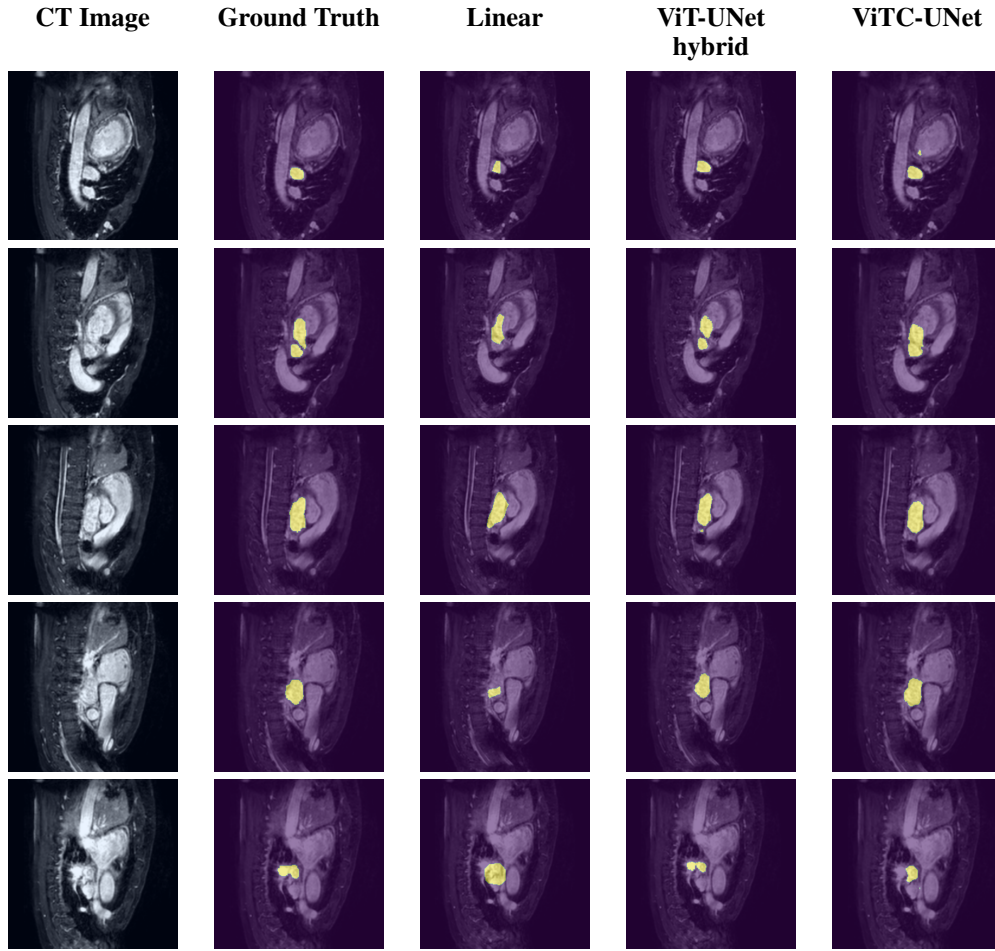
\begin{figure}[h!]
    \centering
    \newcommand{\imgcell}[1]{%
        \begin{minipage}[c][1.9cm][c]{\linewidth}
            \centering
            \includegraphics[width=\linewidth]{#1}
        \end{minipage}%
    }
    \newcommand{\tikzimg}[2]{%
        \begin{minipage}[c][1.9cm][c]{\linewidth}
            \centering
            \begin{tikzpicture}
                \node[inner sep=0pt] (image) {\includegraphics[width=\linewidth]{#1}};
                #2
            \end{tikzpicture}
        \end{minipage}%
    }
    \setlength{\tabcolsep}{7pt}
    \renewcommand{\arraystretch}{5.5} 

    \begin{subfigure}{\textwidth}
        \centering
        \begin{tabular}{>{\centering\arraybackslash}p{0.16\textwidth} >{\centering\arraybackslash}p{0.16\textwidth} >{\centering\arraybackslash}p{0.16\textwidth} >{\centering\arraybackslash}p{0.16\textwidth} >{\centering\arraybackslash}p{0.16\textwidth}}
            \textbf{CT Image} & \textbf{Ground Truth} & \textbf{Linear} & \textbf{ViT-UNet hybrid} & \textbf{ViTC-UNet} \\[-2.0em]
            
            \imgcell{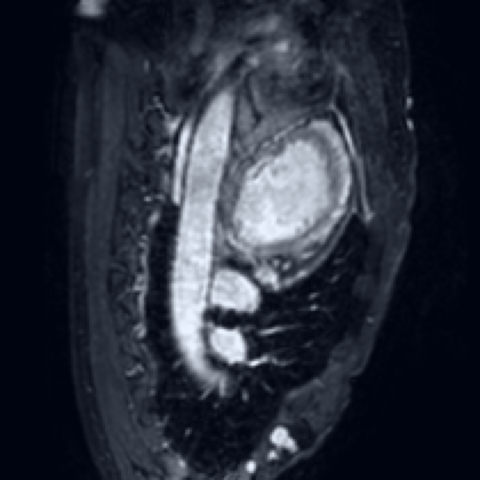} &
            \imgcell{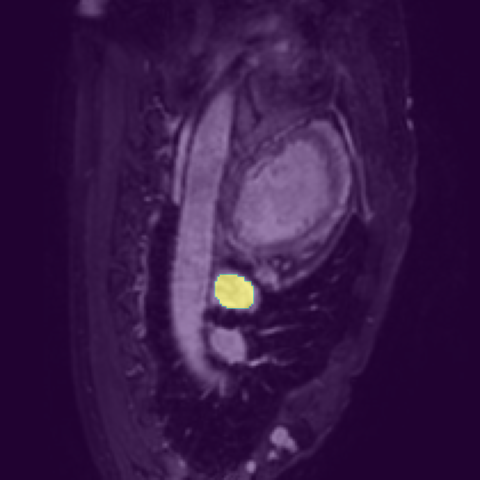} &
            \imgcell{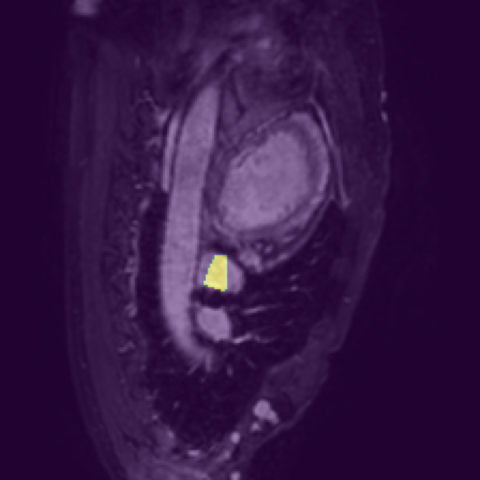} &
            \imgcell{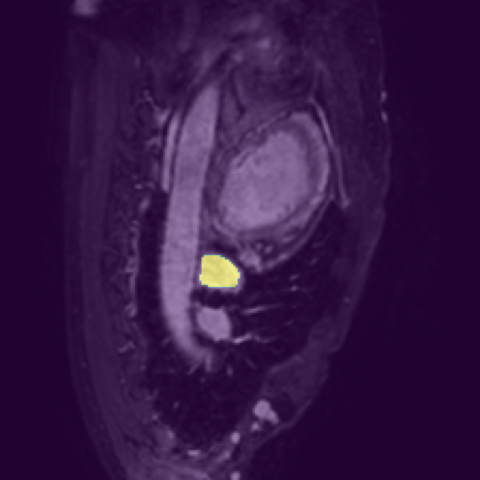} &
            \imgcell{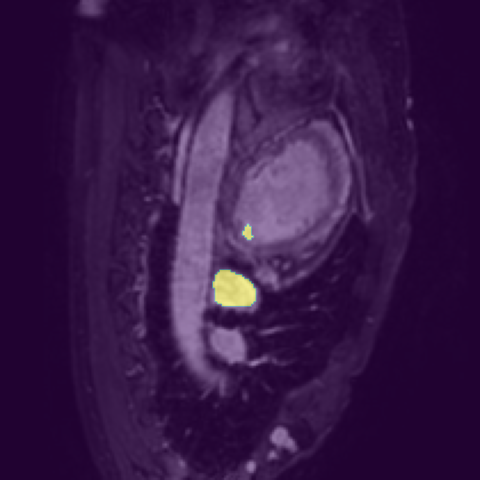} \\
            
            \imgcell{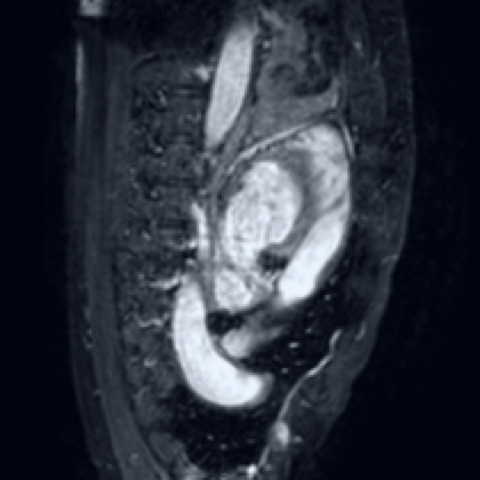} &
            \imgcell{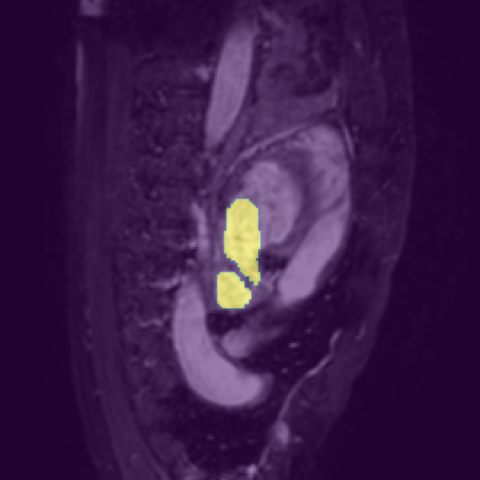} &
            \imgcell{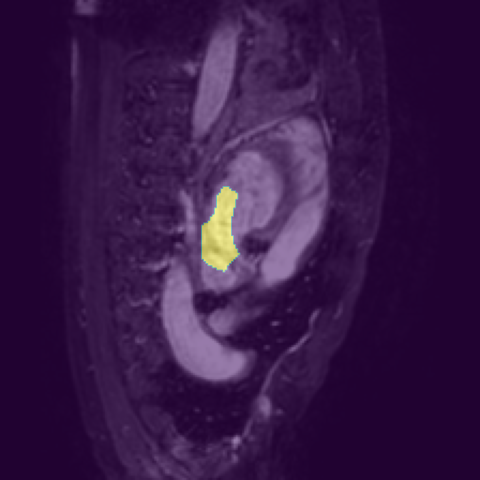} &
            \imgcell{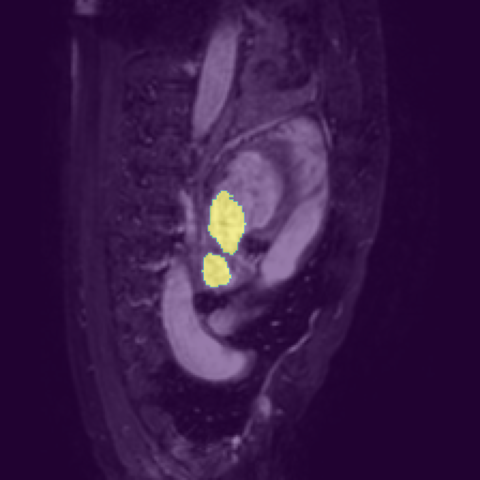} &
            \imgcell{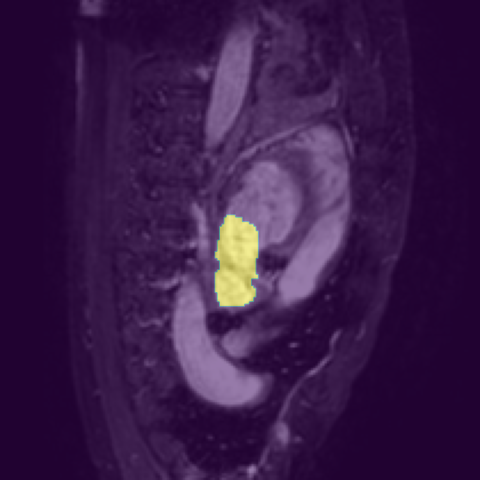} \\

            \imgcell{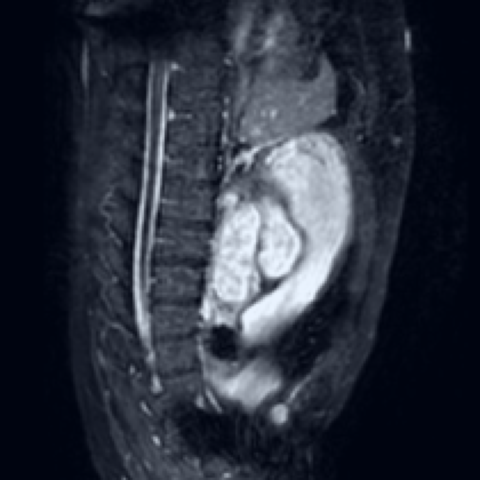} &
            \imgcell{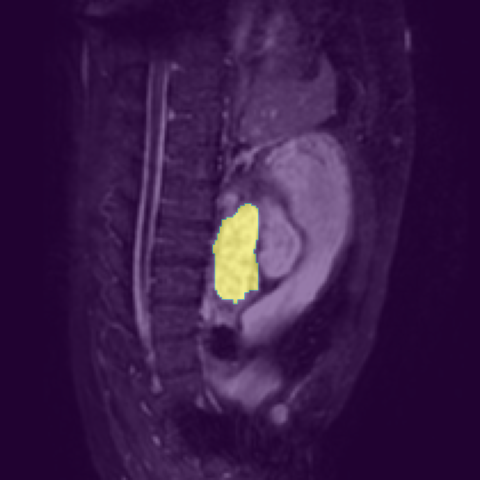} &
            \imgcell{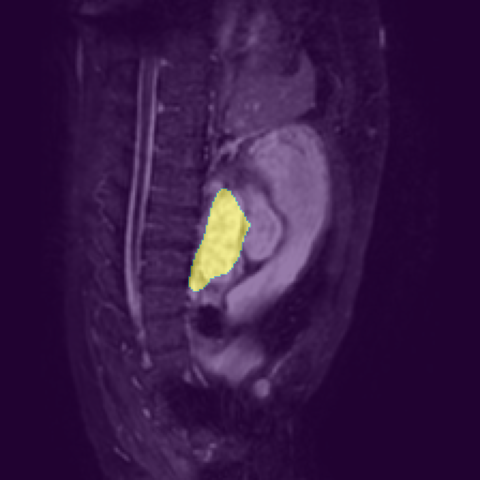} &
            \imgcell{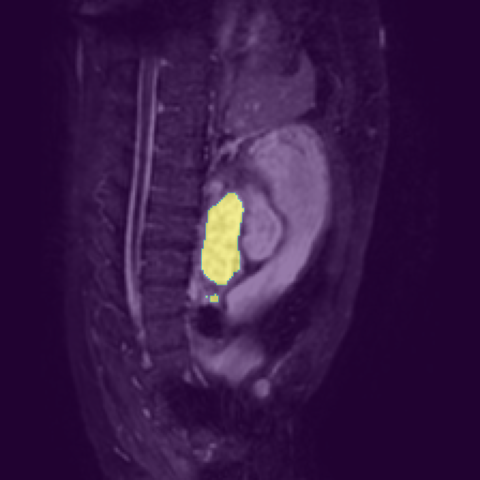} &
            \imgcell{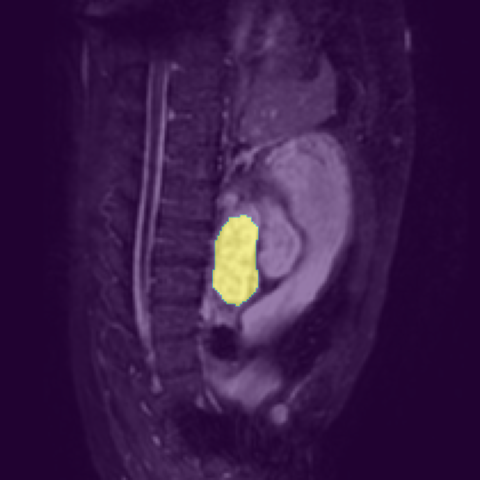} \\

            \imgcell{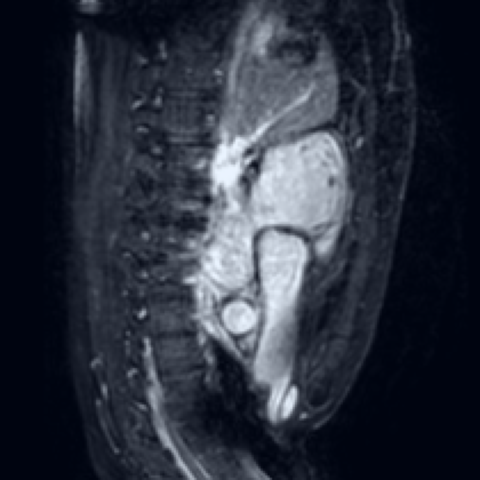} &
            \imgcell{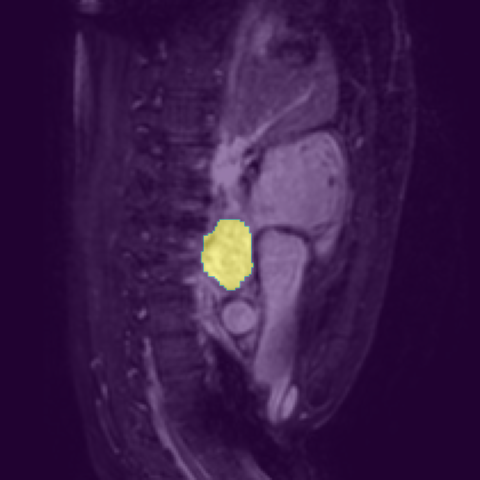} &
            \imgcell{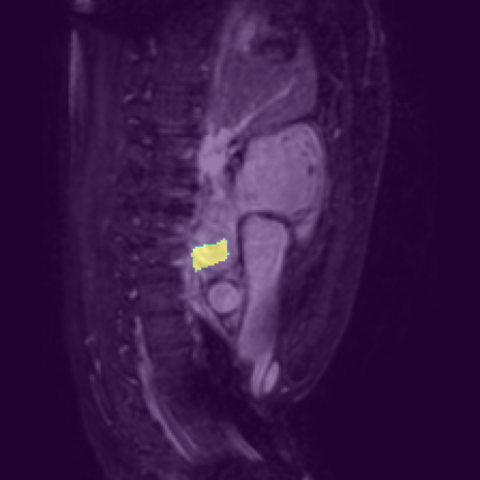} &
            \imgcell{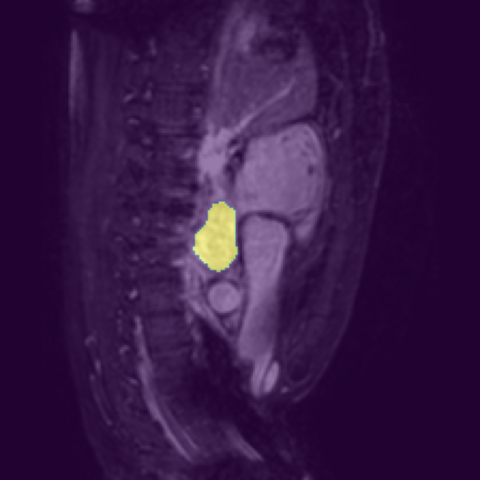} &
            \imgcell{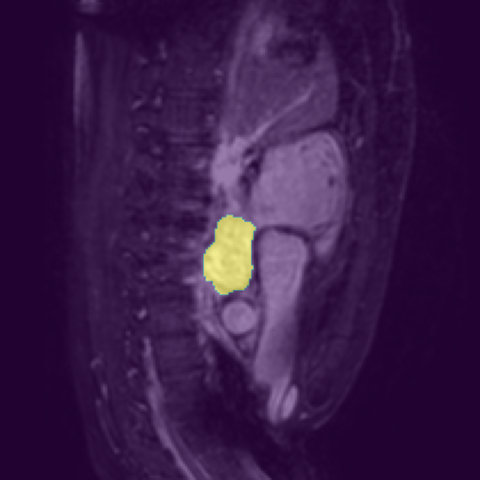} \\

            \imgcell{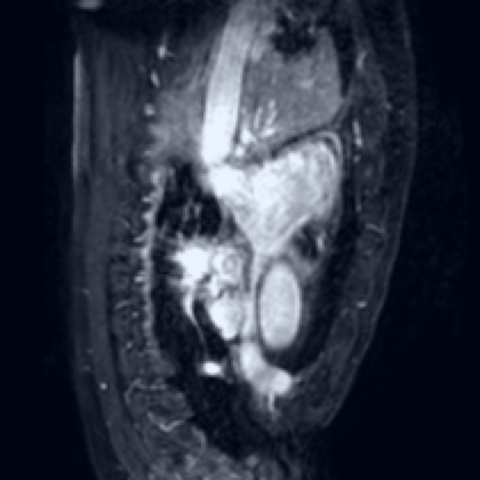} &
            \imgcell{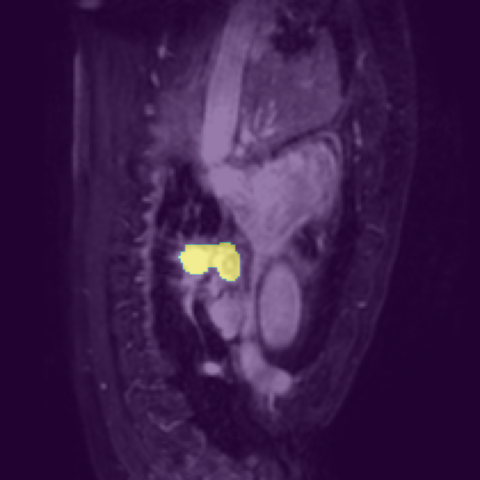} &
            \imgcell{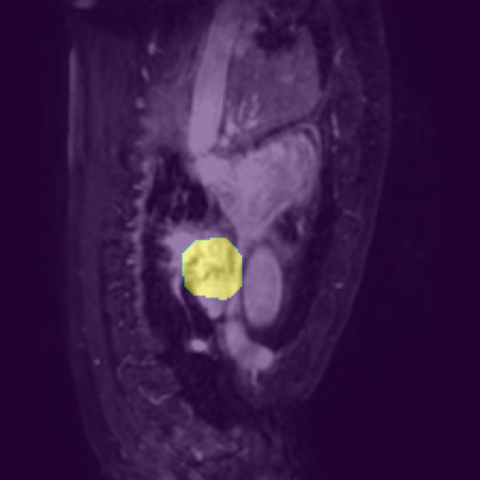} &
            \imgcell{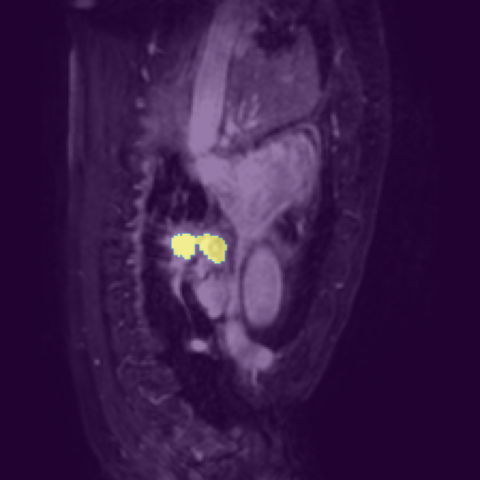} &
            \imgcell{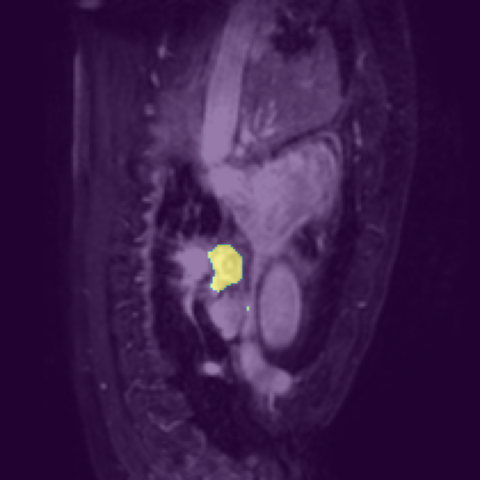} \\
            
        \end{tabular}
    \end{subfigure}

    \caption{Comparison of segmentation performance between a linear baseline, the ViT-UNet hybrid \cite{rmlp}, and our proposed ViTC-UNet utilizing DINOv2 \cite{dinov2} as encoder for semantic segmentation of the left atrium on the Medical Decathlon \cite{medical decathlon} dataset.}
\end{figure}

\begin{figure}[h!]
    \centering
    \newcommand{\imgcell}[1]{%
        \begin{minipage}[c][1.9cm][c]{\linewidth}
            \centering
            \includegraphics[width=\linewidth]{#1}
        \end{minipage}%
    }
    \newcommand{\tikzimg}[2]{%
        \begin{minipage}[c][1.9cm][c]{\linewidth}
            \centering
            \begin{tikzpicture}
                \node[inner sep=0pt] (image) {\includegraphics[width=\linewidth]{#1}};
                #2
            \end{tikzpicture}
        \end{minipage}%
    }
    \setlength{\tabcolsep}{7pt}
    \renewcommand{\arraystretch}{5.5} 

    \begin{subfigure}{\textwidth}
        \centering
        \begin{tabular}{>{\centering\arraybackslash}p{0.16\textwidth} >{\centering\arraybackslash}p{0.16\textwidth} >{\centering\arraybackslash}p{0.16\textwidth} >{\centering\arraybackslash}p{0.16\textwidth} >{\centering\arraybackslash}p{0.16\textwidth}}
            \textbf{CT Image} & \textbf{Ground Truth} & \textbf{Linear} & \textbf{ViT-UNet hybrid} & \textbf{ViTC-UNet} \\[-2.0em]
            
            \imgcell{examples/overlays/mri/vitc_unet/spider_ivd/slice2_raw.png} &
            \imgcell{examples/overlays/mri/vitc_unet/spider_ivd/slice2_target.png} &
            \imgcell{examples/overlays/mri/linear/spider_ivd/slice2_pred.png} &
            \imgcell{examples/overlays/mri/hybrid/spider_ivd/slice2_pred.png} &
            \imgcell{examples/overlays/mri/vitc_unet/spider_ivd/slice2_pred.png} \\
            
            \imgcell{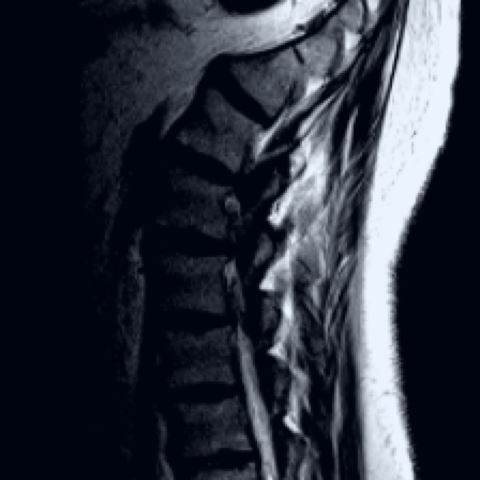} &
            \imgcell{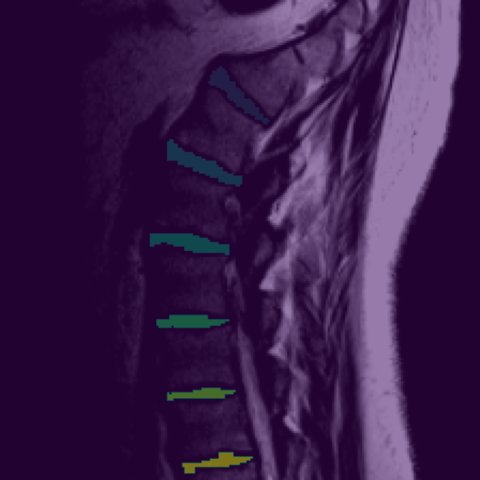} &
            \imgcell{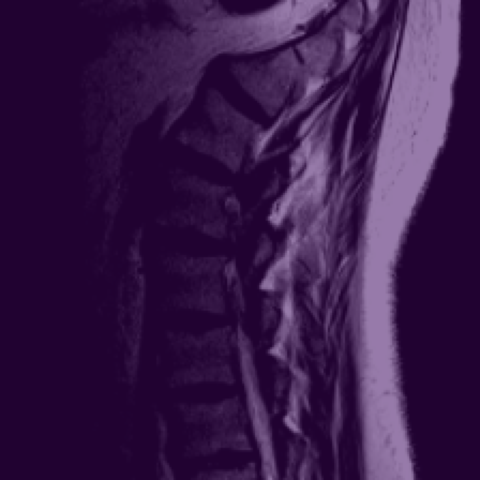} &
            \imgcell{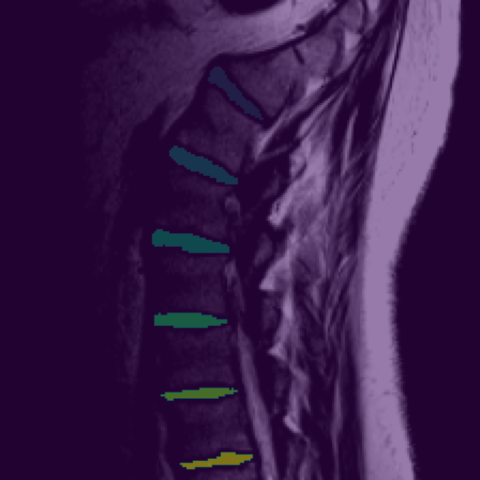} &
            \imgcell{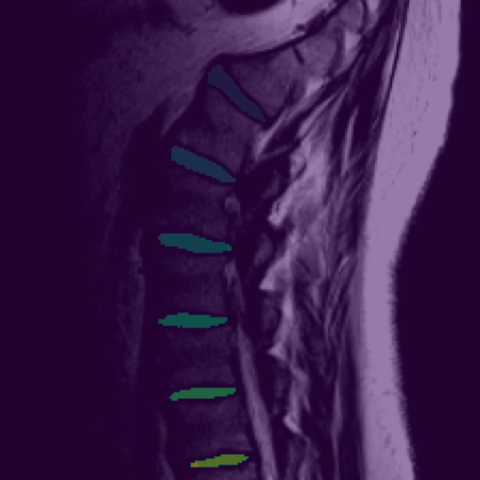} \\

            \imgcell{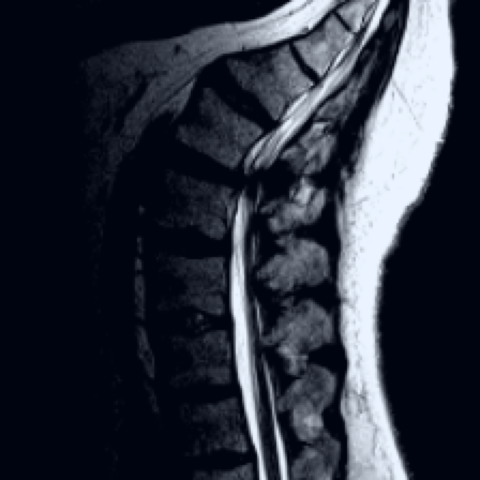} &
            \imgcell{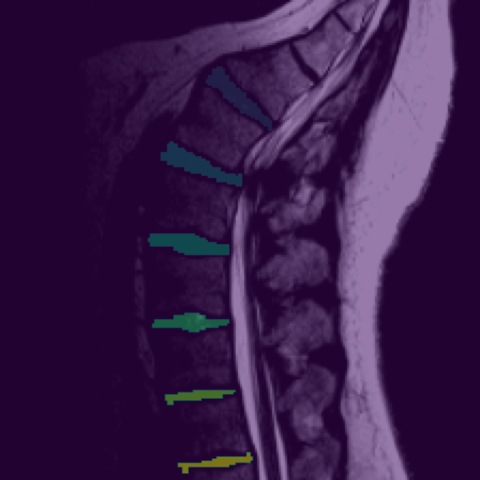} &
            \imgcell{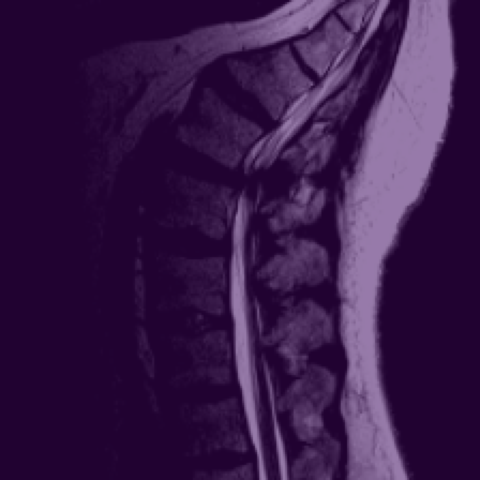} &
            \imgcell{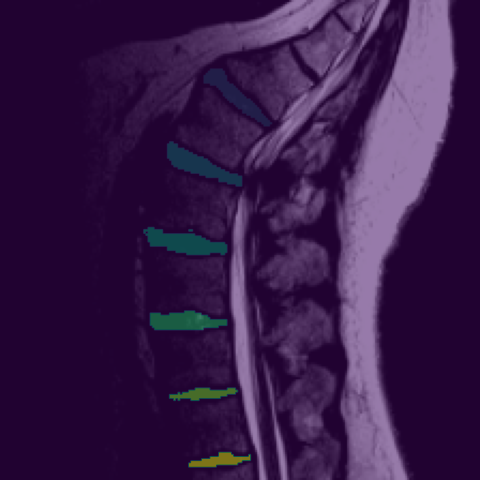} &
            \imgcell{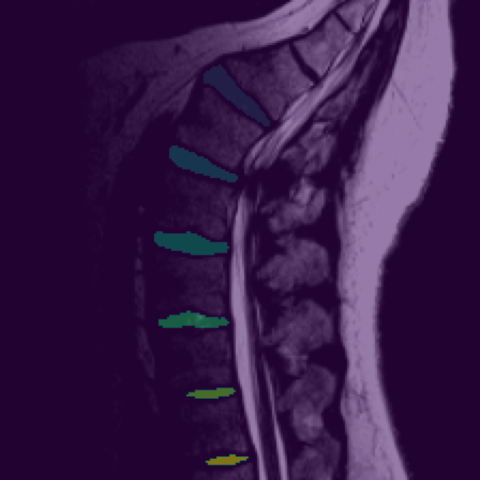} \\

            \imgcell{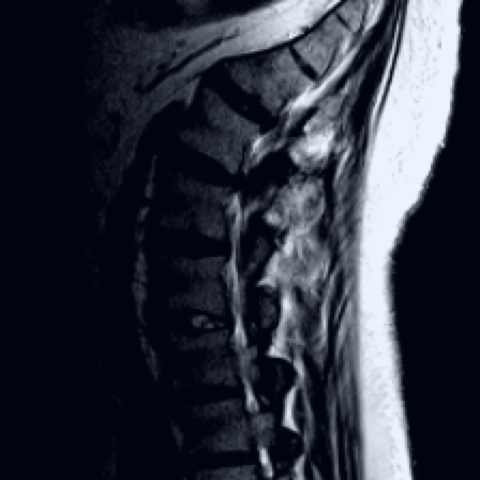} &
            \imgcell{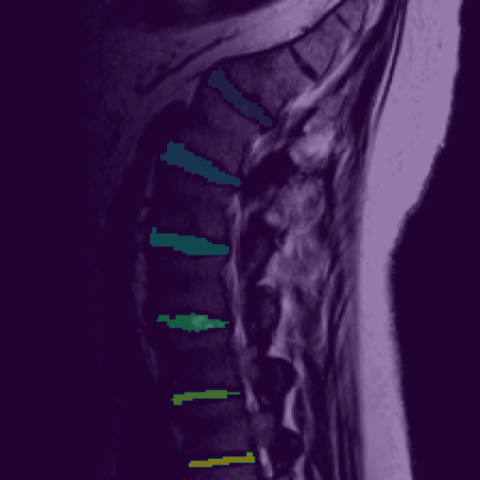} &
            \imgcell{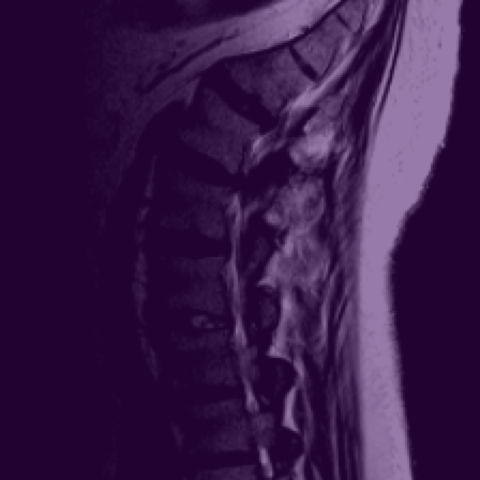} &
            \imgcell{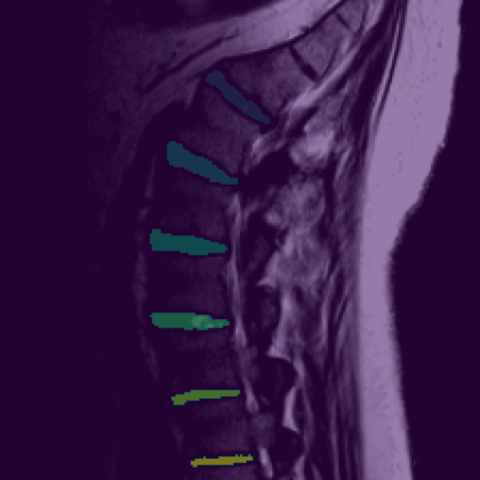} &
            \imgcell{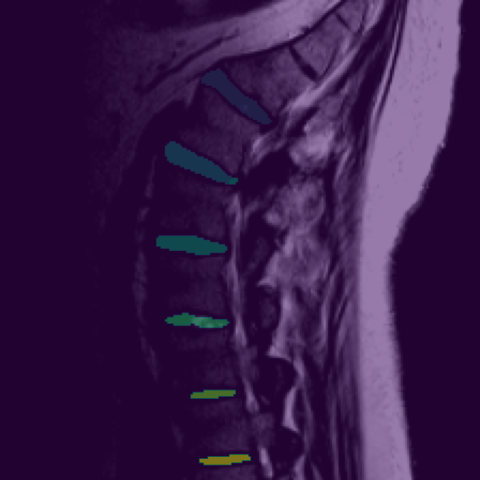} \\

            \imgcell{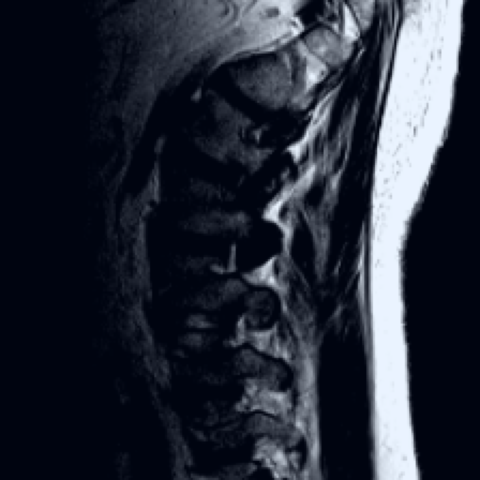} &
            \imgcell{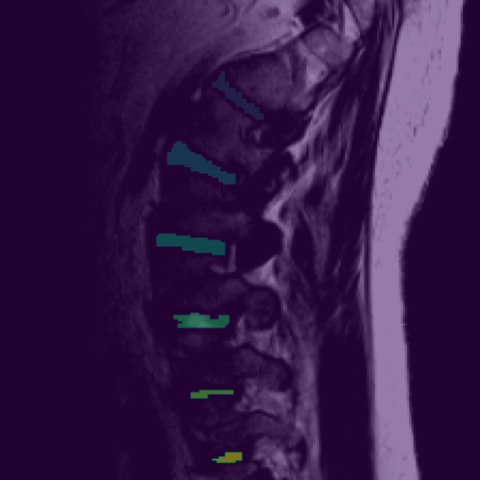} &
            \imgcell{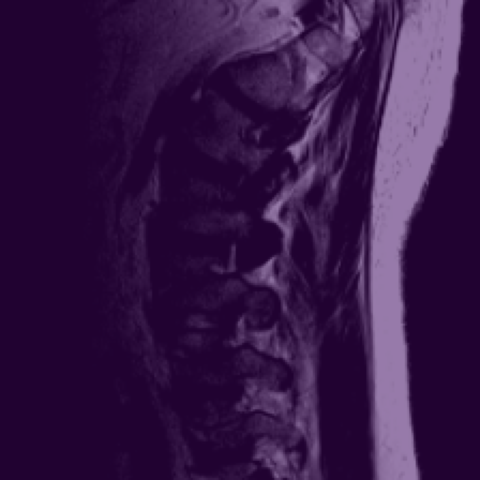} &
            \imgcell{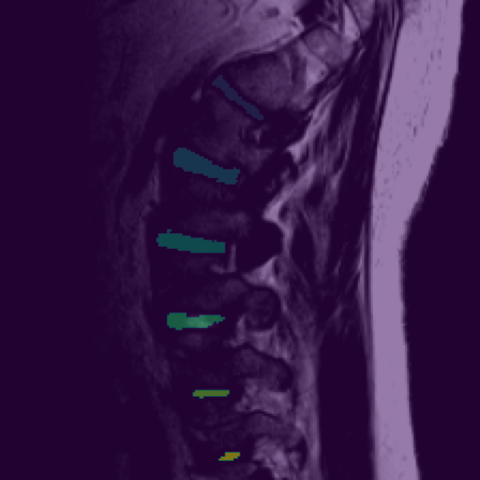} &
            \imgcell{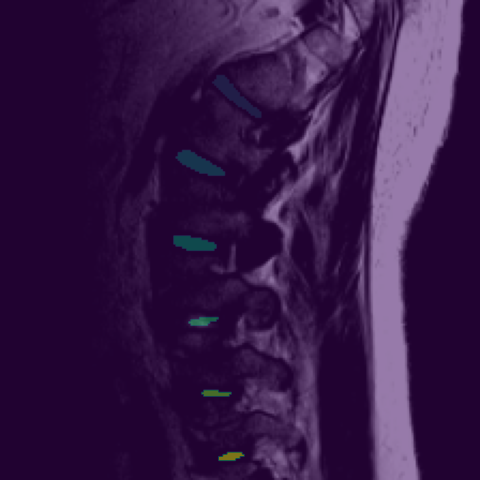} \\
            
        \end{tabular}
    \end{subfigure}

    \caption{Comparison of segmentation performance between a linear baseline, the ViT-UNet hybrid \cite{rmlp}, and our proposed ViTC-UNet utilizing DINOv2 \cite{dinov2} as encoder for semantic segmentation of inter vertebrae disks on the Spider \cite{spider} dataset.}
\end{figure}

\begin{figure}[h!]
    \centering
    \newcommand{\imgcell}[1]{%
        \begin{minipage}[c][1.9cm][c]{\linewidth}
            \centering
            \includegraphics[width=\linewidth]{#1}
        \end{minipage}%
    }
    \newcommand{\tikzimg}[2]{%
        \begin{minipage}[c][1.9cm][c]{\linewidth}
            \centering
            \begin{tikzpicture}
                \node[inner sep=0pt] (image) {\includegraphics[width=\linewidth]{#1}};
                #2
            \end{tikzpicture}
        \end{minipage}%
    }
    \setlength{\tabcolsep}{7pt}
    \renewcommand{\arraystretch}{5.5} 

    \begin{subfigure}{\textwidth}
        \centering
        \begin{tabular}{>{\centering\arraybackslash}p{0.16\textwidth} >{\centering\arraybackslash}p{0.16\textwidth} >{\centering\arraybackslash}p{0.16\textwidth} >{\centering\arraybackslash}p{0.16\textwidth} >{\centering\arraybackslash}p{0.16\textwidth}}
            \textbf{CT Image} & \textbf{Ground Truth} & \textbf{Linear} & \textbf{ViT-UNet hybrid} & \textbf{ViTC-UNet} \\[-2.0em]
            
            \imgcell{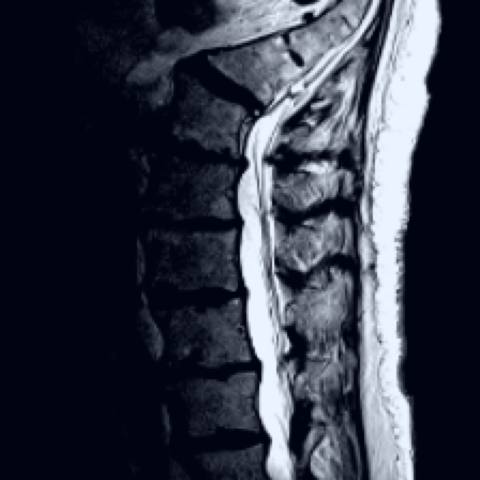} &
            \imgcell{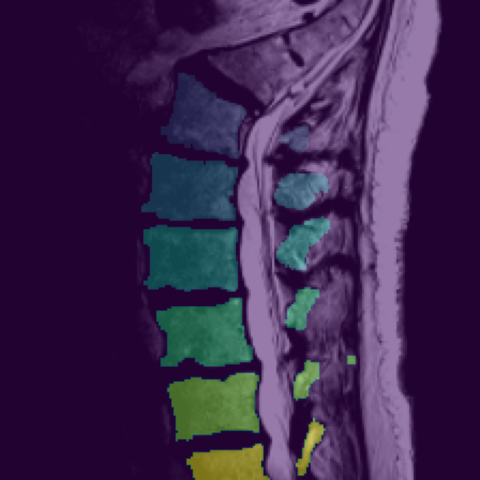} &
            \imgcell{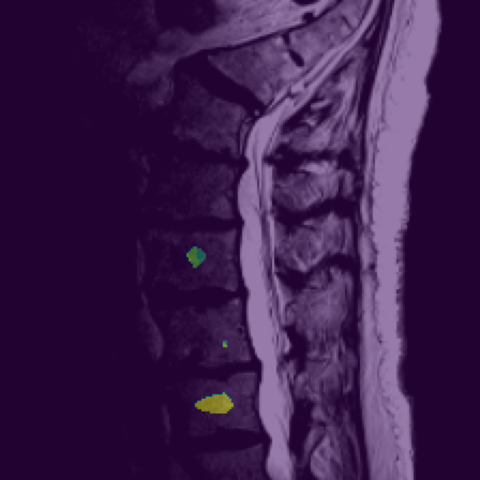} &
            \imgcell{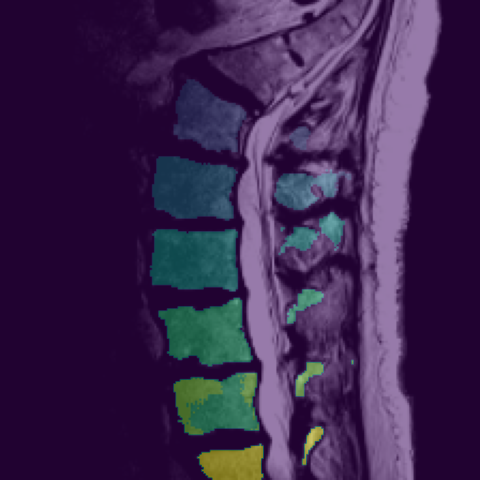} &
            \imgcell{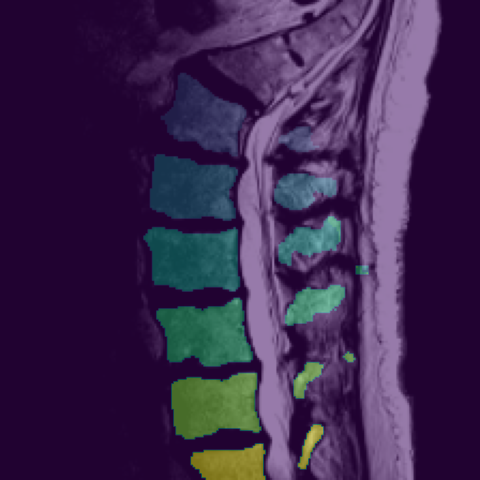} \\
            
            \imgcell{examples/overlays/mri/vitc_unet/spider_vertebrae/slice9_raw.png} &
            \imgcell{examples/overlays/mri/vitc_unet/spider_vertebrae/slice9_target.png} &
            \imgcell{examples/overlays/mri/linear/spider_vertebrae/slice9_pred.png} &
            \imgcell{examples/overlays/mri/hybrid/spider_vertebrae/slice9_pred.png} &
            \imgcell{examples/overlays/mri/vitc_unet/spider_vertebrae/slice9_pred.png} \\

            \imgcell{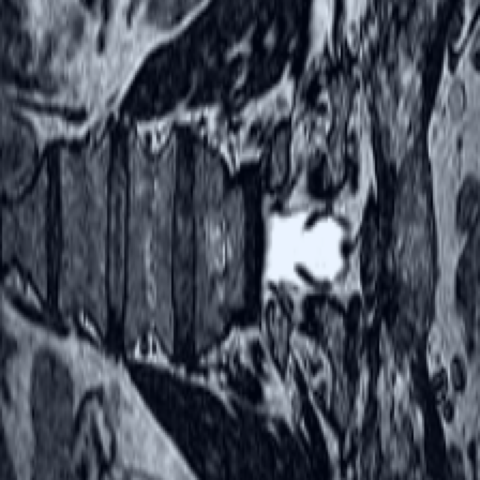} &
            \imgcell{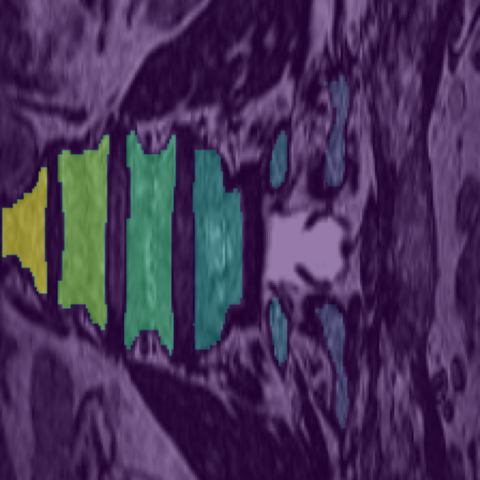} &
            \imgcell{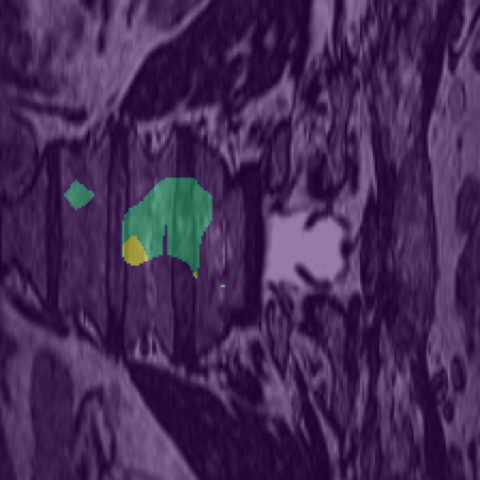} &
            \imgcell{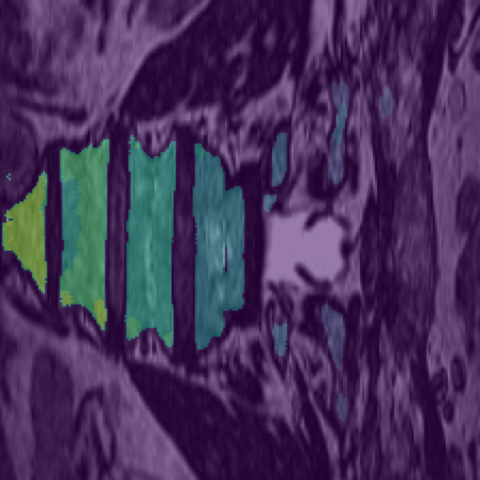} &
            \imgcell{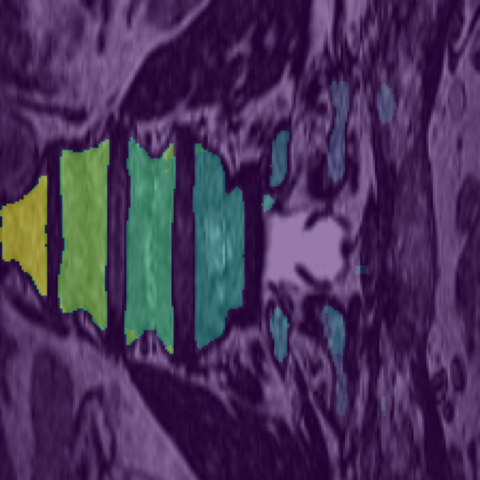} \\
        \end{tabular}
    \end{subfigure}

    \caption{Comparison of segmentation performance between a linear baseline, the ViT-UNet hybrid \cite{rmlp}, and our proposed ViTC-UNet utilizing DINOv2 \cite{dinov2} as encoder for semantic segmentation of vertebrae on the Spider \cite{spider} dataset.}
\end{figure}

\clearpage
\newpage

\end{document}